\documentclass{article}

\usepackage{microtype}
\usepackage{graphicx}
\usepackage{subfigure}
\usepackage{booktabs} 
\usepackage{thmtools, thm-restate}
\usepackage{booktabs,tabularx}
\usepackage{caption}

\usepackage{hyperref}

\usepackage[accepted]{icml2018}

\usepackage{wrapfig}

\usepackage{algorithm}
\usepackage{algorithmicx}
\usepackage{algpseudocode}

\usepackage{array}
\usepackage{calc}
\usepackage{xspace}

\usepackage{enumitem}

\usepackage{amsmath}
\usepackage{amssymb}
\usepackage{amsthm} 

\usepackage[svgnames]{xcolor}

\usepackage[disable]{todonotes}
\definecolor{blued}{RGB}{70,197,221}
\newcommand{\todod}[1]{\todo[color=blued,inline]{#1}}
\definecolor{citrine}{rgb}{0.89, 0.82, 0.04}
\definecolor{blued}{RGB}{70,197,221}
\definecolor{pearOne}{HTML}{2C3E50}
\definecolor{pearTwo}{HTML}{A9CF54}
\definecolor{pearTwoT}{HTML}{C2895B}
\definecolor{pearThree}{HTML}{E74C3C}
\colorlet{titleTh}{pearOne}
\colorlet{bull}{pearTwo}
\definecolor{pearcomp}{HTML}{B97E29}
\definecolor{pearFour}{HTML}{588F27}
\definecolor{pearFith}{HTML}{ECF0F1}
\definecolor{pearDark}{HTML}{2980B9}
\definecolor{pearDarker}{HTML}{1D2DEC}

\newcommand{\todom}[1]{\todo[color=citrine]{\tiny#1}}

\newcommand{\todoa}[1]{\todo[color=yellow, inline]{#1}}
\newcommand{\todoaout}[1]{\todo[color=yellow]{\tiny#1}}

\usepackage{xcolor} 

\newcommand{\rcolb}[1]{\textcolor{red!90}{\textbf{#1}}}

\newcommand{\bcolb}[1]{\textcolor{blue!90}{\textbf{#1}}}

\usepackage{tikz}
\usetikzlibrary{decorations.pathreplacing,calc}

\def\:#1{\protect \ifmmode {\mathbf{#1}} \else {\textbf{#1}} \fi}
\newcommand{\hfssparse}{\texttt{Sparse-HFS}\xspace}
\newcommand{\hfsstable}{\texttt{Stable-HFS}\xspace}
\newcommand{\hfs}{\texttt{HFS}\xspace}

\newcommand{\rls}{\textsc{LapSmo}\xspace}
\newcommand{\ssl}{\textsc{SSL}\xspace}
\newcommand{\scluster}{\textsc{SC}\xspace}
\newcommand{\squeak}{\texttt{SQUEAK}\xspace}
\newcommand{\gd}{\texttt{GD}\xspace}
\newcommand{\lembed}{\textsc{LE}\xspace}
\newcommand{\disre}{\texttt{\textcolor[rgb]{0.5,0.2,0}{DiSRe}}\xspace}
\newcommand{\kn}{\texttt{kN}\xspace}

\newcommand{\exact}{\texttt{EXACT}\xspace}
\newcommand{\ltr}{\texttt{LTR}\xspace}
\newcommand{\deff}{d_{\text{eff}}}
\newcommand{\B}{\mathcal{B}}

\newcommand*{\eqdef}{\triangleq}

\newcommand{\CommaBin}{\mathbin{\raisebox{0.5ex}{,}}}

\newcommand{\laplacian}{\:L}

\newcommand{\qbar}{\overline{q}}

\newcommand{\mergeresparsH}{\texttt{\textcolor[rgb]{0.5,0.2,0}{Merge-Resparsify}}\xspace}

\newcommand{\poolH}{\mathcal{S}}

\newcommand{\dictpool}{\poolH}
\newcommand{\coldict}{\Hg}

\newcommand{\nhl}{m}
\newcommand{\concmat}{\:Q}
\newcommand{\concvec}{\:q}
\newcommand{\atau}{\wt{r}}
\newcommand{\featkermatrix}{\:B}

\newcommand{\lidx}{t}

\newcommand{\calS}{\mathcal{S}}
\newcommand{\X}{\mathcal{V}}

\newcommand{\indfunc}{\mathbb{I}}

\renewcommand{\Re}{\mathbb{R}}

\newcommand{\wt}[1]{\widetilde{#1}}
\newcommand{\wh}[1]{\widehat{#1}}
\newcommand{\ol}[1]{\overline{#1}}

\newcommand{\wb}[1]{\overline{#1}}
\newcommand{\transp}{\mathsf{T}}
\DeclareMathOperator*{\argmin}{arg\,min}

\DeclareMathOperator*{\Tr}{Tr}

\DeclareMathOperator*{\Ker}{Ker}

\newcommand{\norm}[2]{\left\Vert #1 \right\Vert_{#2}}
\newcommand{\normsmall}[1]{\Vert #1 \Vert}

\newcommand{\probability}{\mathbb{P}}

\DeclareMathOperator*{\expectedvalue}{\mathbb{E}}

\newcommand{\condbar}{\;\middle|\;}

\renewcommand{\Re}{\mathbb{R}}

\newcommand{\Real}{\mathbb{R}}

\newcommand{\Gg}{\mathcal{G}}
\newcommand{\Hg}{\mathcal{H}}

\newcommand{\edgeset}{\mathcal{E}}

\newcommand{\funcspace}{\mathcal{F}}
\newcommand{\F}{\mathcal{F}}
\newcommand{\Fp}{\mathcal{F}}

\newcommand{\trainingset}{\mathcal{S}}
\newcommand{\testset}{\mathcal{T}}
\newcommand{\clusterspace}{\mathcal{C}}

\newcommand{\vareps}{\varepsilon}
\renewcommand{\epsilon}{\varepsilon}
\newcommand{\bigotime}{\mathcal{O}}
\DeclareMathOperator*{\polylog}{polylog}

\newtheorem{lemma}{Lemma}
\newtheorem{definition}{Definition}
\newtheorem{proposition}{Proposition}

  \hypersetup{ %
    pdftitle={Improved Large-Scale Graph Learning through Ridge Spectral Sparsification},
    pdfauthor={Calandriello, Daniele and Koutis, Ioannis and Lazaric, Alessandro and  Valko, Michal},
    pdfsubject={Proceedings of the International Conference on Machine Learning 2018},
    pdfkeywords={},
    pdfborder=0 0 0,
    pdfpagemode=UseNone,
    colorlinks=true,
    linkcolor=pearThree,
    citecolor=pearDark,
    filecolor=mydarkblue,
    urlcolor=pearDarker,
    pdfview=FitH}

\icmltitlerunning{Improved Large-Scale Graph Learning through Ridge Spectral Sparsification}

\begin{document}

\twocolumn[
\icmltitle{Improved Large-Scale Graph Learning through Ridge Spectral Sparsification}



\icmlsetsymbol{equal}{*}

\begin{icmlauthorlist}
\icmlauthor{Daniele Calandriello}{sequel,iit}
\icmlauthor{Ioannis Koutis}{unj}
\icmlauthor{Alessandro Lazaric}{facebook}
\icmlauthor{Michal Valko}{sequel}
\end{icmlauthorlist}

\icmlaffiliation{sequel}{SequeL team, INRIA Lille - Nord Europe, France}
\icmlaffiliation{facebook}{Facebook AI Research, Paris, France}
\icmlaffiliation{iit}{LCSL, IIT, Italy, and MIT, USA.}
\icmlaffiliation{unj}{New Jersey Institute of Technology, USA}

\icmlcorrespondingauthor{Daniele Calandriello}{daniele.calandriello@iit.it}

\icmlkeywords{Machine Learning, ICML}
\vskip 0.3in
]



\printAffiliationsAndNotice{}  

\begin{abstract}

The representation and learning benefits of methods based on graph Laplacians, such as \emph{Laplacian smoothing} or \textit{harmonic function solution for semi-supervised learning} (\ssl),  are empirically and theoretically well supported. Nonetheless, the exact versions of these methods scale poorly with the number of nodes $n$ of the graph. In this paper, we combine a spectral sparsification routine with Laplacian learning. Given a graph $\Gg$ as input, our algorithm computes a sparsifier in a \textit{distributed} way in $\bigotime(n\log^3(n))$ time, $\bigotime(m\log^3(n))$ work and $\bigotime(n\log(n))$ memory, using only $\log(n)$ rounds of communication. Furthermore, motivated by the regularization often employed in learning algorithms, we show that constructing sparsifiers that preserve the spectrum of the Laplacian \textit{only up to} the regularization level may drastically reduce the size of the final graph. By constructing a spectrally-similar graph, we are able to bound the error induced by the sparsification for a variety of downstream tasks (e.g., \ssl). We empirically validate the theoretical guarantees on Amazon co-purchase graph and compare to the state-of-the-art heuristics.

\end{abstract}

\vspace{-0.05in}
\section{Introduction}
\vspace{-0.05in}

Graphs are a very effective data structure to represent relationships between entities (e.g.,\@ social and collaboration networks, influence graphs
). Over the years, many machine learning problems have been defined and solved exploiting the graph representation, such as  \emph{graph-regularized least squares} (\textsc{LapRLS}\xspace, \citealt{belkin2005manifold}), \emph{Laplacian smoothing} (\rls, \citealt{sadhanala_graph_2016})
graph \emph{semi-supervised learning} (\ssl, \citealt{chapelle2010semi-supervised, zhu2003semi-supervised}),
\emph{laplacian embedding} (\lembed, \citealt{belkin2001laplacian}, and \emph{spectral clustering} (\scluster,
\citealt{von2007tutorial}). The intuition behind graph-based learning is that
the information expressed by the graph helps to capture
the underlying structure of the problem (e.g., a
manifold), thus improving the learning. For instance, \rls and \ssl rely on the assumption that nodes that are \textit{close} in the graph are
more likely to have similar labels. Similarly, \lembed and \scluster try to find
a low-dimensional representation of the nodes using the eigenvectors of the
Laplacian of the graph. In general, given a graph $\Gg$ of $n$ nodes and $m$ edges, most of graph-based learning tasks require computing the minimum of a cost function based on the associated $n \times n$ Laplacian matrix~$\:L_{\Gg}$, which contains $m$ non-zero entries. Solving \emph{exactly} such optimization problems amounts to $\bigotime(n^3)$ time and $\bigotime(n^2)$ space complexity in the worst case and they become infeasible even for mildly large/dense graphs.\todoaout{Here I put dense, even if the complexities do not explicitly scale with $m$...}

A complete review of the literature on large-scale graph learning is beyond the scope of this paper and we only consider methods that reduce learning space and time complexity starting from a given graph received as input.\footnote{Many algorithms reduce the complexity of graph learning \textit{at construction} time but they cannot be applied to \textit{natural} graphs (e.g., social graphs) and therefore we do not review them.} We identify mainly three possible approaches. We can (1) reduce runtime replacing the pseudo-inverse operator $\laplacian_{\Gg}^+$ with an \emph{iterative solver}, (2) reduce time and space complexity replacing the large graph $\Gg$ with a sparser approximation~$\Hg$, or  (3) reduce runtime and increase memory capacity by \emph{distributing} the computation across multiple machines. \\[0.05in]
\textit{Iterative solvers.} Iterative methods can solve a number of learning problems without explicitly constructing $\laplacian_{\Gg}^+$ (e.g., gradient descent, \gd, for \rls, iterative averaging for \ssl, and the power method for \scluster). In this case we only need $\bigotime(m)$ time per iteration.
Unfortunately, all simple iterative methods (e.g., \gd) converge in a number of iterations proportional to the condition number of the Laplacian, $\kappa = \lambda_{\max}(\:L_\Gg)/\lambda_{\min}(\:L_\Gg)$, which may grow linearly with the number of nodes $n$, thus removing the advantage of the iterative method, whose complexity tends to $\bigotime(n^3)$ in the worst case. Advanced iterative methods, such as the \emph{preconditioned conjugate gradient,} use preconditioning to find an accurate solution in a number of iterations independent of~$\kappa$.
\citet{koutis2011a-nearly-m} gives a nearly-linear
solver for Laplacians or \emph{strongly diagonally dominant} (SDD) matrices,
that using a chain of preconditioners,
converges in only $\bigotime(m\log(n))$ time. As space and time costs scale with the number of edges, a natural desire is to reduce $m$ by sparsifying and distributing the graph.\\[0.05in] 
\textit{Graph sparsification.} The objective of sparsification methods is to remove \emph{redundant} edges, so that the resulting sparse sub-graph can be easily stored in memory and efficiently manipulated to compute final solutions. A simple \textit{graph-sparsification} technique is to sample $n\qbar$ (with $\qbar > 1$) edges from~$\Gg$ with probabilities proportional to the edge weights with replacement. While computationally very efficient, uniform sampling requires sampling a number of edges proportional to $\bigotime(n\mu(\Gg))$ (i.e., $\qbar \propto \mu(\Gg)$), where $\mu(\Gg)$ is the \textit{coherence} of the Laplacian matrix, and it can grow as large as $n$ when the
graph is highly structured (e.g., if there is a single edge $e$
connecting two components of the graph we need to sample all of the edges of
the graph---potentially $\bigotime(n^2)$---to guarantee that we do not exclude
$e$ and generate an inappropriate~$\Hg$). 
A more refined approach is the $k$-neighbors (\kn) sparsifier~\citep{sadhanala_graph_2016}, which performs local sparsifications node-by-node by keeping all edges at nodes with degree smaller than $\qbar$, and samples them proportionally to their weights whenever the degree is bigger than~$\qbar$. While in certain structured graphs, this method 
may perform much better than uniform~\citep{von_luxburg_hitting_2014}, in the general case $\qbar$, still needs to scale with the coherence~$\mu(\Gg)$. A more effective method is to sample edges proportionally to their \textit{effective resistance}, which intuitively measures the importance of an edge in preserving the minimum distance between two nodes. As a result, only relevant edges are kept and the sparsified graph could be reduced to $\bigotime(n\polylog(n))$ edges. Nonetheless, computing effective resistances also requires the pseudo-inverse
$\laplacian_{\Gg}^{+}$, thus being as expensive as solving any graph-Laplacian learning problem.\\[0.05in]
\textit{Distributed computing.} When the number of edges $m$ is too large to fit the
whole graph in a single machine, we are forced to distribute the edges across
multiple machines. At the same time, if the sparsifier construction or the
downstream inference can be parallelized, we can also reduce their runtime.
Unfortunately, distributing data and computation across multiple machines
can cause large communication costs. For example,
simple \gd or label propagation methods require
$\bigotime(\kappa)$ iterations (and communication rounds) to converge
and access to non-local (e.g., neighbors in a graph) data.
While preconditioned solvers reduce the number of iterations,
almost none of their memory access is local, thus making difficult to have efficient distributed implementations.



\textbf{Contribution.}
In this paper, we propose a new approach that aims at integrating the benefits of the three different methods above. Using the large memory and computational capacity of distributed computing and leveraging the sequential sparsification methods of~\citet{kelner_spectral_2013} and~\citet{calandriello_disqueak_2017}, we show how to compute an accurate sparsifier $\Hg$ of graph $\Gg$ in
$\bigotime(n\log^3(n))$ time, $\bigotime(n\log^2(n))$ work and $\bigotime(n\log(n))$ memory,
using only $\log(n)$ rounds of communication.
Afterwards, learning tasks can be solved directly on $\:L_{\Hg}$ on a single machine
using near-linear time solvers, resulting in an overall $\bigotime(n\log^3(n))$ runtime.
Moreover, we show that the regularization used in some graph-based learning algorithms allows using even sparser graphs. In particular, we introduce the notion of \textit{ridge} effective resistance to obtain sparsifiers that are better adapted to solve Laplacian-regularized learning tasks (e.g., \rls, \ssl) and are smaller than standard spectral sparsifiers without compromising the performance of downstream tasks.

%
\vspace{-0.05in}
\section{Background}
\vspace{-0.05in}

We use lowercase letters $a$ for scalars, bold lowercase letters~$\:a$
for vectors and uppercase bold letters $\:A$ for matrices.
We use  $\:A \preceq \:B$ to denote that $\:B - \:A$ is positive semi-definite (PSD),
$[\:A]_{i,j}$ to indicate the $(i,j)$-th
entry of $\:A$, and
ordered the eigenvalues as
$\lambda_{1}(\:A) \leq \ldots \leq \lambda_n(\:A)$.

\subsection{Graphs and graph Laplacian}

We denote with $\Gg = (\X,\edgeset)$, an undirected weighted graph with $n$
nodes $\X$ and $m$ edges $\edgeset$. 
Each edge $e_{i,j}\in\edgeset$ has a weight $a_{e_{i,j}}$ measuring the ``similarity'' between
nodes~$i$ and $j$.
Given graphs $\Gg$ and $\Gg'$ over the same set of nodes~$\X$,
$\Gg+\Gg'$ denotes the graph obtained by summing the weights of their edges. 
%
For graph $\Gg$, we introduce the weighted adjacency matrix  $\:A_{\Gg}$ with
entries $[\:A_{\Gg}]_{i,j} = a_{e_{i,j}}$,  the total weights $A = \sum_{e} a_e$ , and the diagonal
degree matrix $\:D_{\Gg}$ with entries $[\:D_{\Gg}]_{i,i} \triangleq \sum_j a_{e_{i,j}}$. The Laplacian of $\Gg$ is the PSD matrix $\:L_\Gg \triangleq \:D_\Gg - \:A_\Gg$.
Furthermore, we assume that~$\Gg$ is connected and thus $\:L_\Gg$ has only one eigenvalue
equal to $0$ and $\Ker(\:L_\Gg)=\:1$. Let $\:L_\Gg^+$ be the pseudoinverse of $\:L_\Gg$ and
$\:L_{\Gg}^{-1/2} = (\:L_\Gg^+)^{1/2}$.
For any node $i=1,\ldots,n$, we denote with $\mathbf{\chi}_i\in\Re^n$, the
indicator vector, so that $\:b_e  \triangleq \sqrt{a_e}(\mathbf{\chi}_i - \mathbf{\chi}_j)$ is the ``edge'' vector. If we
denote with $\:B_{\Gg}$ the $m \times n$ signed edge-vertex incidence matrix,
then the Laplacian can be written as $\:L_{\Gg} = \sum_e \:b_e \:b_e^\transp = \:B_\Gg^\transp
\:B_\Gg$.

\subsection{Learning on graphs}
Given  graph $\Gg$ and its Laplacian $\:L_{\Gg}$, we denote with
$\:f \in \Real^n$, a \emph{labeling} of its nodes, where $[\:f]_i$ is
the value associated with the $i$-th node.
Many graph learning algorithms assume that the optimal labeling $\:f^\star$ is \emph{smooth}
w.r.t.\ the graph, i.e., the quantity
$\sum_{e} a_e([\:f^\star]_{e_i} - [\:f^\star]_{e_j})^2 = {\:f^\star}^\transp\:L_{\Gg}\:f^\star$ is small.
In the following, we review examples
from the supervised, semi-supervised and unsupervised learning with graphs.


\textbf{Laplacian smoothing (\rls) with Gaussian noise.} Given a graph $\Gg$ on $n$ nodes,
let $\:y \triangleq \:f^\star + \:\xi$ be
a noisy measurement of $\:f^\star$ with $[\:\xi]_i \sim \mathcal{N}(0, \sigma^2)$.
The goal of \rls is to find a vector $\wh{\:f}$ that
accurately reconstructs~$\:f^\star$ under the graph smoothness assumption by solving
\begin{align}\label{eq:laprls}
\wh{\:f} &\triangleq 
\argmin_{\:f \in \Re^n}  
(\:f - \:y)^\transp (\:f-\:y) + \lambda \:f^\transp \:L_\Gg \:f\nonumber\\
&= (\lambda \:L_\Gg + \:I)^{-1} \:y,
\end{align}\todom{check the formulation}
where $\lambda$ is a regularization parameter.

\textbf{Graph semi-supervised learning (\ssl).} In \ssl,
the input $\:f_\ell$ is a partial observation of the labels $\:f^\star$ for a subset 
$\trainingset \subset [n]$ of nodes.
The goal is to predict the labels $\:f_u$  of the unrevealed nodes.
The \emph{harmonic function solution} (\hfs)
by~\citet{zhu2003semi-supervised} 
solves the optimization problem
\begin{align}\label{eq:hfs.original}
\wh{\:f}_{\text{HFS}} &\triangleq \argmin_{\:f \in \Re^n}  \tfrac{1}{\ell}
 (\:f - \:y)^\transp \:\ell_{\trainingset} (\:f-\:y) + \lambda \:f^\transp \:L_\Gg \:f\nonumber\\
&= (\lambda \ell \:L_\Gg + \:I_\trainingset)^+ \:y_\calS,
\end{align}  \todom{check the formulation}
where $\ell \triangleq |\trainingset|$ is the number of labeled nodes received as input,
$\:I_{\trainingset}\in\Re^{n\times n}$ is the identity matrix with zeros
at nodes not in $\trainingset$, and  $\:y_\calS \triangleq \:I_\calS \:y \in \Re^n$.
Similarly, in \emph{local transductive regression} (\ltr) \citep{cortes2008stability}, the optimization problem is
\begin{align}\label{eq:hfs.cortes}
\wh{\:f}_{\ltr} &\triangleq \argmin_{\:f \in \Re^n}  
 (\:f - \:y)^\transp \:C (\:f-\:y) +  \:f^\transp (\:L_{\Gg} + \lambda\:I) \:f\nonumber\\
&= (\:C^{-1}(\:L_{\Gg} + \lambda\:I)  + \:I)^{-1} \:y_\calS,
\end{align}
where 
$\:C$ is
a diagonal matrix with entries $c_\ell$ for nodes in~$\trainingset$,~$c_u$ for
entries not in $\trainingset$, and $c_\ell \geq c_u >0$.

\textbf{Spectral clustering (\scluster).}
Applying the Laplacian smoothness assumption, the goal of \scluster is to find $k$ disjoint subset assignments
such that the clusters are smooth w.r.t.\,the Laplacian. Let $\{\:f_c\}_{c=1}^k$ be the cluster indicator vectors such that $[\:f_c]_i \triangleq 1$ if
node $i$ is in the $c$-th cluster and $[\:f_c]_i \triangleq 0$ otherwise. Denote with $\:F \in \Real^{n \times k},$
the matrix containing the assignments, and let $\clusterspace$ be the space
of feasible clustering, such that all $\:f_c$ are binary and
each row of $\:F$ contains only one non-zero entry. Since computing the minimum ratio-cut is NP-hard~\cite{von2007tutorial,lee2014cluster}, even under constraints~\cite{pmlr-v51-cucuringu16}, \scluster defines instead the relaxed problem
\[
\wh{\:F} \triangleq \argmin_{\:F : \:F^\transp\:F = \:I_k, \:f_c \bot \:1 } \Tr(\:F^\transp\laplacian\:F).
\]
Once the relaxed solution is computed, we can use different heuristics
to recover the clustering, such as thresholding or performing a $k$-means clustering on the $\wh{\:F}$ matrix.

\textbf{Computational complexity.}
The problems above require either to compute an eigendecomposition
of the Laplacian $\:L_{\Gg}$ or to solve a linear system involving $\:L_{\Gg}$. 
Computing these exactly is not feasible when
the number of nodes $n$ and edges $m$ grows. In particular,
(a) storing $\:L_{\Gg}$ in memory requires $\bigotime(m)$ space, and it is not feasible when $m$ is large,
(b) even if $\:L_{\Gg}$ is sparse and $m$ is small, the pseudo-inverse $\:L_{\Gg}^{+}$
might be dense, and thus computing and storing $\:L_{\Gg}^{+}$ exactly requires up to $\bigotime(n^3)$ time
and $\bigotime(n^2)$ space.

\vspace{-0.05in}
\section{Distributed Spectral Sparsification}\label{sec:dist-kl}
\vspace{-0.05in}

In this section, we describe a new, sequential, distributed, and efficient algorithm for graph sparsification that can be used as a preprocessing step to solve a large variety of downstream learning tasks, without significantly affecting their performance.
We point out that while distributing data-agnostic sparsifiers (e.g.\@ uniform sampling) is straightforward, distributing the computation of sparsifiers based on effective resistances
requires a careful merging procedure to guarantee satisfactory memory vs.\,accuracy tradeoff, which is what we provide in this section.

\subsection{$(\vareps, \gamma)$-spectral sparsifiers}

We start with the introduction of the notion of $(\vareps, \gamma)$-sparsifier that is adapted for the learning tasks
that use sparsified graph Laplacian.

\begin{definition}\label{def:eps-sparsifier}
A $(\vareps, \gamma)$-spectral sparsifier of $\Gg$ is a re-weighted sub-graph $\Hg \subseteq \Gg$ whose Laplacian $\:L_{\Hg}$ satisfies
\begin{align}\label{eq:sparsif-quad-form}
(1-\vareps) \:L_\Gg - \vareps\gamma \:I \preceq \:L_\Hg \preceq (1 + \vareps) \:L_\Gg + \vareps\gamma \:I.
\end{align}
\end{definition}

For $\gamma=0$, this definition reduces to the standard notion of $\vareps$-spectral sparsifier~\citep{spielman_spectral_2011}. The main difference is that an $(\vareps, \gamma)$-spectral sparsifier allows for an extra \emph{additive} error of order $\vareps\gamma$. This change is directly motivated by the fact that the sparsifier $\Hg$ may be used in learning tasks whose solution may not be sensitive to small (additive) errors. As a result, $(\vareps, \gamma)$-spectral sparsifiers are able to further reduce the size of $\Hg$ w.r.t.\,$(\vareps, 0)$-sparsifiers, without significantly affecting the final learning performance. Formally, an $\vareps$-sparsifier preserves all the quadratic forms up to a small multiplicative (constant) error, and thus can be used to
provide an accurate approximation to many important quantities such as graph cuts
or eigenvalues. In fact, for all $i \in [n]$, an $\vareps$-sparsifier guarantees that $(1-\vareps)\lambda_i(\:L_{\Gg}) \leq \lambda_i(\:L_{\Hg}) \leq (1+\vareps)\lambda_i(\:L_{\Gg})$.
%
Nonetheless, in many learning tasks (e.g., \ltr) the noise level in the signal~$\:f$ requires regularizing the solution so that the Laplacian $\:L_{\Gg}$ itself is eventually replaced by $\:L_{\Gg}+\lambda\:I$ (e.g., Eq.\,\ref{eq:laprls}). This corresponds to \textit{soft-thresholding} the eigenvalues of the Laplacian, so that eigenvalues below $\lambda$ are partially ignored. If $\lambda$ is properly tuned w.r.t.\,the noise, the regularization increases stability and improves the learning performance.
 Therefore, constructing a sparsifier that accurately reconstructs \textit{all} eigenvalues of $\:L_{\Gg}$ may be wasteful, as it may require keeping most of the edges. 
As a result, in tasks where $\:L_{\Gg}$ is regularized, it is better to use $(\vareps,\gamma)$-sparsifiers, as their additive error $\gamma\:I$ is homogeneous with the regularization $\lambda\:I$ and their smaller size allows scaling to up.\footnote{Whenever no regularization is required in the learning task (i.e., \hfs, \scluster), we set $\gamma=0$ and consider ``standard'' $\vareps$-sparsifiers.}
We now extend the results of~\citet{spielman2011graph} for the construction of $\vareps$-spectral sparsifiers to the general case of $(\vareps,\gamma)$-sparsifiers. We redefine the edge effective resistance to account for the regularization.

\begin{definition}\label{def:eff-res}
The $\gamma$-effective resistance of an edge $e$ in graph $\Gg$ is defined as 
\begin{align}\label{eq:eff.resist}
r_e(\gamma) \triangleq \:b_e^\transp \big(\:L_\Gg+\gamma\:I\big)^{-1} \:b_e.
\end{align}
The ``effective dimension'' of the graph is the total sum of the $\gamma$-effective resistances,
%
$\deff(\gamma) \triangleq \sum_e r_e(\gamma)$.
\end{definition}
We can now construct a sparsifier $\Hg$ by sampling $\qbar$ times each
edge 
 with a probability proportional to its 
$\gamma$-effective resistance. More formally, the resulting (random) graph
contains $q_e \sim \B(r_e(\lambda); \qbar)$ copies of each edge, where $\B$ is
the Binomial distribution, and its associated Laplacian is $\:L_{\Hg} = \sum_{e
\in \Hg} q_e/(\qbar r_e(\gamma))\:b_e\:b_e^\transp$, which is an unbiased
estimator of $\:L_{\Gg}$.
We can then apply existing results from sketching of PSD matrices
\cite{alaoui2014fast}
to prove that $\Hg$ is
a valid $(\vareps,\gamma)$-sparsifier.
\begin{proposition}[\citealt{cohen2017input}]\label{lemma:esp.gamma.sparsifier} 
Let $\vareps>0$ and $\gamma\geq 0$ be the accuracy parameters and $0 \leq \delta \leq 1$ the probability of error. Let~$\Hg$ be the graph obtained by sampling edges in $\Gg$ with a probability proportional to their $\gamma$-effective resistances. If $\qbar \geq 4\log(4n/\delta)/\vareps^2,$ then w.p.\,$1-\delta$, $\Hg$ is an $(\vareps,\gamma)$-sparsifier with $\bigotime(\deff(\gamma)\qbar)$ edges.
\end{proposition}

We first notice that this result reduces to the one of~\citet{spielman2011graph} for $\gamma=0$. In fact, $\deff(0)=n-1$ for all graphs, thus matching the space requirement $\qbar$ for $\vareps$-sparsifiers.
Nonetheless, as $\gamma$ increases, the size of $\Hg$ reduces significantly. Using $\:L_{\Gg} = \:B_{\Gg}^\transp\:B_{\Gg}$, the effective dimension $\deff(\gamma)$ can be conveniently rewritten as
%
\[
\deff(\gamma) = \Tr\big(\:B_{\Gg}^\transp\:B_{\Gg} (\:B_{\Gg}^\transp\:B_{\Gg}+\gamma\:I)^{-1} \big) = \sum_{i=2}^n \frac{\lambda_i(\:L_{\Gg})}{\lambda_i(\:L_{\Gg})+\gamma}\CommaBin
\]
%
thus showing that $\deff(\gamma)$ is the ``soft'' rank of the Laplacian, where $\gamma$ significantly reduces the contribution of small eigenvalues to the total sum. While in the worst case $\deff(\gamma)$ can be as large as $n-1$, for a variety of graphs with rapidly decaying spectrum~\citep{jamakovic2006laplacian,samukhin2008laplacian,zhan2010distributions,akoglu2015graph}, $\deff(\gamma)$ may be significantly smaller than $n-1$, thus reducing the number of edges $\qbar$ required to obtain an $(\vareps,\gamma)$-sparsifier. 


\subsection{The algorithm}

\begin{algorithm}[t]
\begin{algorithmic}[1]
\renewcommand{\algorithmicrequire}{\textbf{Input:}}
\renewcommand{\algorithmicensure}{\textbf{Output:}}
\Require $\Gg$
\Ensure $\Hg_{\Gg}$
\State Partition $\Gg$ into $k$ sub-graphs: \\ \quad $\Hg_{1,\ell} \gets \Gg_\ell \gets \{(e_{i,j}, q_e = 1, \wt{p}_{1,e} = 1)\}$
\State Initialize set $\poolH_{1} = \{\Hg_{1,\ell}\}_{\ell=1}^k$
\For{$h = 1,\dots,k-1$}
\State Pick two sparsifiers $\Hg_{h,i'}, \Hg_{h,i'}$ from $\poolH_h$
\State $\wb{\Hg} \gets \mergeresparsH(\Hg_{h,i}, \Hg_{h,i'})$
\State Place $\wb{\Hg}$ back into $\poolH_{h+1}$
\EndFor
\State Return $\Hg_{\Gg}$, the last sparsifier in $\poolH_{k}$
\end{algorithmic}
\caption{The \disre algorithm.}\label{alg:kl_dist}
\end{algorithm}

\begin{algorithm}[t]
\begin{algorithmic}[1]
\Require $(\vareps,\gamma)$-sparsifiers $\Hg_{h,i},\Hg_{h,i'}$ of graphs $\Gg_{h,i}, \Gg_{h,i'}$
\Ensure $\wb{\Hg}$, an $(\vareps,\gamma)$ sparsifier of $\Gg_{h,i} + \Gg_{h,i'}$
\State Initialize $\wb{\Hg} = \Hg_{h,i} + \Hg_{h,i'}$
\State For all $e \in \wb{\Hg}$, use a fast SDD solver to compute \label{code:compute-eff-res}
$$\wt{r}_{h+1,e}(\gamma) \gets (1-\vareps)\:b_e^\transp(\laplacian_{\wb\Hg} + (1+\vareps)\gamma\:I)^{-1}\:b_e$$
\State Set probabilities $\wt{p}_{h+1,e} \gets \min\{\wt r_{h+1,e}(\gamma), \wt{p}_{h,e}\}$
\State Sample $q_{h+1,e}$ from $\mathcal{B}(\wt{p}_{h+1,e}/\wt{p}_{h,e}, q_{h,e})$
\State Return $\wb{\Hg} \gets \{(e_{i,j}, q_{h+1,e}, \wt{p}_{h+1,e})\}$ for all $q_{h+1,e} > 0$
\end{algorithmic}
\caption{\mergeresparsH}\label{alg:merge-resp}
\end{algorithm}

\begin{figure}[t]
\begin{center}
\includegraphics[width=0.99\columnwidth]{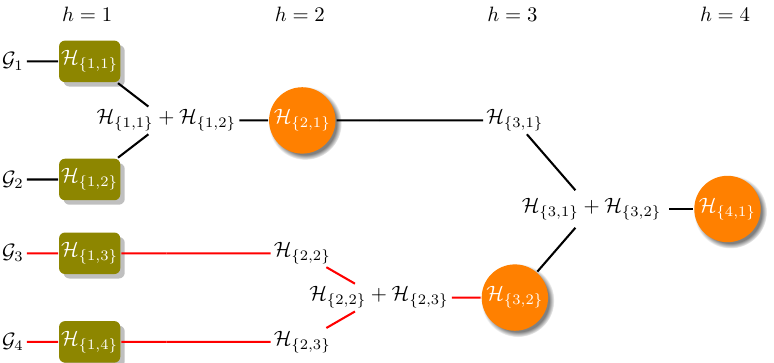}
\end{center}
\vspace{-0.1in}
\caption{Merge tree for Alg.\,\ref{alg:kl_dist}.}\label{fig:merge-trees}
\vspace{-0.2in}
\end{figure}

As pointed out in the introduction, the main limitation of
effective-resistance-based sparsification is that the computation of $r_e$
requires inverting the Laplacian matrix, thus resulting in a computational cost
that already matches the cost of the learning tasks themselves. Moreover,
large graphs cannot be stored in memory, and multiple passes over the graph would
result in a disk access overhead larger than the computational cost.
In order to
avoid these problems, we adapt our previous work \cite{calandriello_disqueak_2017} in online sparsification and
randomized linear algebra (see a thorough discussion and comparison at the end
of the section) to obtain
the distributed sequential resparsification (\disre) algorithm (Alg.\,\ref{alg:kl_dist}).\footnote{Whenever the original graph contains $m \leq  \wt{\bigotime}(\deff(\gamma))$edges, there is no need to run \disre as the $(\vareps,\gamma)$-sparsifiers would not reduce the size of the graph.}\todom{add constants to the footnote?} 

\textbf{The structure.} We represent a sparsifier $\Hg$ as a collection of weighted edges $\Hg \triangleq \{(e_{i,j}, q_e, \wt{p}_{e})\}$,
and the Laplacian can be reconstructed as $\:L_{\Hg} \triangleq \sum_{e \in \Hg} 1/{\wt{p}_e} (q_e/\qbar)\:b_e\:b_e^\transp$.
Intuitively, each edge $e$ has an associated weight based on its probability $\wt{p}_e$, and a number of included copies
$q_e$. Keeping multiple copies of each edge helps the random $\:L_{\Hg}$ to concentrate towards $\:L_{\Gg}$,
where the maximum number of copies $\qbar$ for an edge trades-off success probability and the size of $\Hg$.
We assume we have $k$ machines.
\disre begins by partitioning the graph $\Gg$ into $k$ sub-graphs $\Gg_{\ell}$
on~$n$ vertices and $m_\ell \geq n$ edges, such that $\Gg = \{\Gg_{\ell}\}_{i=\ell}^k$
In other words, it splits the matrix $\:B_{\Gg}$ into submatrices
$\:B_{\Gg_i}$ by arbitrarily selecting a subset of rows.
The sub-graphs are small enough that they can be stored in memory,\footnote{Whenever this is not possible (i.e., $m/k$ is too large to be stored on a single machine), we can simply apply the same merging scheme of \disre by loading \textit{small enough} chunks of the graph and sparsifying them sequentially.} and they
are also obviously sparsifiers of themselves, therefore we can define
an initial set of sparsifiers $\poolH_{1} \triangleq \{\Hg_{1,\ell}\}_{\ell=1}^k$,
with $\Hg_{1,\ell} \triangleq \{(e_{i,j}, q_{1,e} = \qbar, \wt{p}_{1,e} = 1)\}_{e \in \Gg_l}$. \todom{is it clear enough here that by $\wt{p}_{1,e} = 1$ we are not forgetting the weights?}
With this definition, $\Hg_{1,\ell}$ contains edges $e_{i,j}$
with unit weight $\wt{p}_{1,e} = 1$ and $\Hg_{1,\ell} = \Gg_\ell$. Starting from these initial sparsifiers, \disre proceeds through a sequence of \textit{merge} and \textit{sparsify} operations where two sparsifiers are first combined and then sparsified again to keep having manageable-size graphs at each step. 
While \disre can run on any arbitrary sequence of merges, we consider the most (computationally) effective scheme, where sparsifiers are merged two-by-two in parallel, thus inducing a \textit{balanced} full binary merge tree (see Fig.\,\ref{fig:merge-trees}). For notational convenience, we consider that at each iteration $h$, the inner loop of Alg.\,\ref{alg:kl_dist} only merges two arbitrary sparsifiers from the pool of available sub-graphs~$\poolH_h$ and merges them into a new sparsifier. In practice, multiple merge-and-sparsify operations can be executed in a parallel and asynchronous way. 
The size of $\poolH_h$, number of sparsifiers present at layer $h$, is $|\poolH_h| = k-h+1$. Therefore, a node in the tree corresponding to a sparsifier is uniquely identified by two indices $\{h,\ell\}$ where
$h$ is the height of the layer and $\ell\leq |\poolH_h|$ is the index of the node in the layer. We also define the graph
$\Gg_{\{h,\ell\}}$ as the union of all sub-graphs~$\Gg_{\ell'}$ that are
 reachable from node $\{h,\ell\}$ as leaves (descendants of $\{h,\ell\}$).
For example, in Fig.\,\ref{fig:merge-trees}, sparsifier $\Hg_{3,1}$
in node $\{3,1\}$ approximates the graph $\Gg_{\{3,2\}} = \Gg_3 + \Gg_4$,
where we highlight in red the descendant tree. 

\textbf{The resparsification.} In Alg.\,\ref{alg:merge-resp} we detail how two arbitrary sparsifiers are combined to obtain a temporary graph~$\wb\Hg$. While the merge operation simply combines $\Hg_{h,i}$ and $\Hg_{h,i'}$ by summing their weights, the resparsification aims at generating a valid sparsifier from the ``original'' sub-graph $(\Gg_{h,i} +\Gg_{h,i'})$, as if it was directly sparsified at the beginning. We first compute estimates $\wt r(\gamma)$ of the $\gamma$-effective resistance by using fast solvers to invert the strongly diagonal dominant $L_{\wb\Hg}+\gamma\:I$ matrix. Instead of sampling edges in $\wb\Hg$ directly proportionally to $\wt r(\gamma)$ (more precisely $\wt p_{h+1,e}$), we perform a ``resampling'' scheme where an edge $e$ is preserved with a ``reweighted'' probability $\wt p_{h+1,e}/\wt p_{h,e}$. Intuitively, the overall sequence of resampling guarantees that at each step $h+1$, an edge $e\in(\Gg_{h,i} +\Gg_{h,i'})$ has the ``correct'' probability $\wt p_{h+1,e}$ of being included in the sparsifier.

\textbf{Performance.}
We now study the performance of \disre and its complexity. \textit{Time} complexity refers to the amount of time necessary to compute the final solution and \textit{work} complexity refers to the total amount of operations carried out by \emph{all} machines  to compute the final solution.

\begin{restatable}{theorem}{restathmparallelalgmain}
\label{thm:parallel-alg-main} 
Let $\varepsilon>0$ be the accuracy, $0 \leq \delta \leq 1$ the probability of error, and $\rho \triangleq (1+3\epsilon)/(1-\epsilon)$. Given an arbitrary graph $\Gg$ and an arbitrary merge tree structure, if \disre is run with parameter $\qbar \triangleq 26\rho\log(3n/\delta)/\vareps^2$, then each sub-graphs $\Hg_{\{h,\ell\}}$ is an $(\vareps,\gamma)$-sparsifier of $\Gg_{\{h,\ell\}}$ with at most $3\qbar \deff(\gamma)$ edges with probability $1-\delta$. Whenever the merge tree is balanced and~$k$ is big enough such that $m/k \leq 3\qbar \deff(\gamma)$,\footnote{This implies that there are enough machines so that the leaves in the merge tree already have relatively sparse sub-graphs.} then merge operations can be run in parallel across the machines with an overall time complexity of $\bigotime(\deff(\gamma)\log^3(n))$, a total work $\bigotime(m\log^3(n))$, and $\bigotime(\log(n))$ rounds of communication.
\end{restatable}


\textbf{Discussion.} \citet{kelner_spectral_2013} proposed a sequential algorithm for graph sparsification that closely emulates the batch sampling of~\citet{spielman2011graph} in a semi-streaming setting and incrementally constructs an $\vareps$-sparsifier. However, their proof had a flaw since they treated dependent variables as independent
\cite{calandriello2016analysis}. \citet{kyng_framework_2016} resolved the issues in the proof of~\citet{kelner_spectral_2013} and showed that a slightly modified algorithm
can construct a sparsifier with $\bigotime(n \log(n)/\varepsilon^2)$ edges
in $\bigotime(m \log^2(n)/\varepsilon^2)$ time, matching the space complexity
of batch sampling. 
The method proposed by~\citet{kyng_framework_2016} can be further improved by
parallelizing its computation over multiple machines. Using the parallel
sparsification algorithm of~\citet{koutis2016simple}, the time complexity can
be reduced up to $\wt \bigotime(\log^6(n))$. Nonetheless, since these methods
require \emph{random access to the edges}, they cannot be easily distributed (it
would have $\bigotime(m\polylog(n))$ communication cost) and scaled to graphs
that cannot be stored on a single machine. Furthermore, the algorithm
of~\citet{kyng_framework_2016} accurately reconstructs the whole
spectrum of the Laplacian, which leads to sparsifiers whose number of edges
scales linearly with $n$. On the other hand, in regularized learning tasks, the
presence of multiplicative and additive spectral error allows creating smaller
sparsifiers whose size scales with $\deff(\gamma)$. Notice that, for $\gamma$
large enough, this possibly means sparsifiers with less than $n-1$ edges,
necessarily leading to disconnected graphs. Finally, note that merging two
traditional $\varepsilon$-sparsifiers gives an $\varepsilon$-sparsifier,
merging two $(\gamma,\varepsilon)$-sparsifiers produces a less accurate
$(2\gamma,\varepsilon)$-sparsifier. Therefore simple merge-and-reduce strategies
\cite{Feldman2013mergereduce}, which address every resparsification as independent,
would either cumulate errors or require multiple passes over the data.
Similarly to \citet{kyng_framework_2016}, \disre's sequential \mergeresparsH
solves this problem (Appendix~\ref{seckl}).

Mixed additive-multiplicative reconstruction is studied more extensively in randomized matrix
algebra~\citep{drineas2017lectures}. \citet{cohen_online_2016} developed an
efficient method to spectrally sparsify generic matrices up to $(1\pm\vareps)$
multiplicative and $\gamma$-additive errors using an incremental sampling
method based on \textit{ridge leverage scores} (i.e., the analog of
$\gamma$-effective resistances for matrices). If applied to graph Laplacians,
their method adds edges incrementally and returns an
$(\vareps,\gamma)$-sparsified graph with $\bigotime(\deff(\gamma)\log^2(n))$
edges in $\bigotime(m\log(n))$ time. Nonetheless,
\citet{cohen_online_2016} provided only $\varepsilon$-sparsifiers, suggesting
to set $\gamma$ as small as possible, and did not explore the advantages
possible in machine learning.
Moreover, no existing $(\varepsilon,\gamma)$-sparsifier construction
method can leverage both distribution and fast solvers.
\citet{cohen_online_2016} can only add edges 
(but not remove them as \disre), preventing repeated merge-and-resparsify.
Other streaming RLS sampling methods,
such as by~\citet{cohen2017input}, use dense intermediate sketches,
such as \emph{frequent directions} \cite{ghashami2016frequent},
that are not Laplacians of a sub-graph and cannot be easily paired
with near-linear solvers for Laplacians.


%
\vspace{-0.05in}
\section{Downstream Guarantees}\label{sec:downstream}
\vspace{-0.05in}

\todoa{We should have a better ad-hoc discussion of the actual contributions and comparison to ~\citep{sadhanala_graph_2016}.}

\todoa{Anything changes with $(\vareps,\gamma)$?}

We now show how the spectral reconstruction guarantees provided by
$(\varepsilon,\gamma)$-sparsifiers translate into guarantees on the quality
of the approximate solutions computed using $\Hg$ instead of $\Gg$.
We first introduce a result for $\varepsilon$-sparsifiers in \ssl
and then show how for regularized problems, $(\varepsilon,\gamma)$-sparsification
can further improve computational performance without loss in accuracy in \rls.

\subsection{Generalization bounds for SSL}
Given the closed form solutions of \hfs (Eq.\,\ref{eq:hfs.original}) and \ltr (Eq.\,\ref{eq:hfs.cortes}),
we simply replace $\:L_{\Gg}$
with $\:L_{\Hg}$ and then run a nearly-linear time solver
to obtain approximate solutions
$\wt{\:f}_{\text{HFS}}$ and $\wt{\:f}_{\text{LTR}}$.
We compare approximate solutions 
to their exact counterparts in the context of algorithmic stability.

\begin{definition}\label{def:beta-stability}
Let $\mathcal{L}$ be a transductive learning algorithm. We denote by $\:f$ and $\:f'$ the solutions obtained by running $\mathcal{L}$ on datasets $\X \triangleq (\trainingset, \testset)$  and $\X \triangleq (\trainingset', \testset')$ respectively. $\mathcal{L}$~is uniformly $\beta$-stable w.r.t.\,the squared loss if there exists $\beta \geq 0$ such that for any two partitions $(\trainingset, \testset)$ and $(\trainingset', \testset')$ that differ by exactly one training (and test) point and for all $i \in [n]$, we have $|([\:f]_i - [\:y]_i)^2 - ([\:f']_i - [\:y]_i)^2| \leq \beta.$
\end{definition}

The stability of LTR was proven by \citet{cortes2008stability}. On the other hand,
the singularity of the Laplacian may lead to unstable behavior in HFS due to the
$(\gamma \ell\:L_{\Gg} + \:I_{\trainingset})^{+}$ pseudo-inverse, with
drastically different results for small perturbations of the dataset. For this
reason, we take the \hfsstable algorithm by~\citet{belkin2004regularization}, where an additional regularization term is
introduced to restrict the space of admissible solutions to the space
$\funcspace \triangleq \left\{\:f : \langle\:f,\:1\rangle = 0\right\}$ of solutions
orthogonal to the null space of $\:L_\Gg$ (i.e., centered functions).
As shown by~\citet{belkin2004regularization}, to satisfy the constraint,
it is sufficient to set an additional regularization parameter $\mu$ to
$\mu \triangleq ((\gamma \ell \:L_\Gg + \:I_\trainingset)^+\:y_S)^\transp \:1 / ((\gamma \ell \:L_\Gg + \:I_\trainingset)^+ \:1)^\transp \:1$,
and compute the solution $\wh{\:f}_{\text{STA}}$ as
$\wh{\:f}_{\text{STA}} \triangleq (\gamma \ell \:L_\Gg + \:I_\trainingset)^+ (\:y_\calS - \mu \:1)$.
While \hfsstable
is more stable and thus more suited for theoretical analysis, its
space and time requirement remains $\bigotime(m)$ 
and cannot be applied to graphs with a large number of edges.
Therefore, we again replace $\wh{\:f}_{\text{STA}}$ with an approximate solution $\wt{\:f}_{\text{STA}}$ computed using $\:L_{\Hg}$.
Define $\wh{R}(\:f) \triangleq \tfrac{1}{\ell} \sum_{i=1}^{\ell} (\:f(x_i) - \:y(x_i))^2$ as the empirical error and $R(\:f) \triangleq \tfrac{1}{u}\sum_{i=1}^{u} (\:f(x_i) - \:y(x_i))^2$ as the generalization.
\begin{restatable}{theorem}{restathmsparsesslgeneralization}\label{thm:sparse-ssl-generalization}
    Let $\Gg$ be a fixed (connected) graph with eigenvalues $0 = \lambda_1(\Gg) < \lambda_2(\Gg) \leq \ldots \leq \lambda_n(\Gg)$, and $\Hg$ an $\varepsilon$-sparsifier of $\Gg$. Let $\:y\in\Re^n$ be the labels of the nodes in $\Gg$ with $|\:y(x)| \leq c$ and $\funcspace$ be the set of centered functions such that $|\:f(x) - \:y(x)| \leq 2c$. Let $\calS\subset \X$ be a random subset of labeled nodes, if the labels $\:y_S$ are centered, then w.p.\,at least $1 - \delta$ (w.r.t.\,the random generation of the sparsifier~$\Hg$ and the random subset of labeled points~$\calS$) the resulting \hfsstable solution satisfies
\begin{align}\label{eq:bound}
    R(\wt{\:f}) \leq \widehat{R}&(\wh{\:f}) + \beta +\left(2\beta + \frac{4c^2(\ell+u)}{\ell u}\right)\sqrt{\frac{\pi(\ell,u)\ln\frac{1}{\delta}}{2}}\nonumber\\
    &+\frac{1}{1-\varepsilon}\left(\frac{2(1+\varepsilon)\varepsilon \ell\gamma\lambda_{2}(\Gg) c}{((1-\varepsilon)\ell \gamma \lambda_{2}(\Gg)\!-1)^{2}}\right)^{2}\!\!\CommaBin
\end{align}
where $\wt{\:f}$ and $\wh{\:f}$ are computed on $\Hg$ and $\Gg$,
\begin{align*}
    &\pi(\ell,u) \triangleq \frac{\ell u}{\ell+u-0.5} \frac{2\max\{\ell,u\}}{2\max\{\ell,u\}-1}\enspace \text{ and }\\
    &\beta \leq \frac{3 c\sqrt{\ell}}{((1-\varepsilon)\ell \gamma \lambda_{2}(\Gg)-1)^{2}} + \frac{4c}{(1 - \varepsilon)\ell \gamma \lambda_{2}(\Gg)-1}\cdot
\end{align*}
\end{restatable}
Thm.\,\ref{thm:sparse-ssl-generalization} (full proof in Appendix~\ref{proof4}) shows how approximating $\Gg$ with 
$\Hg$ impacts the generalization error as the number of labeled samples $\ell$ 
increases.
If we set $\varepsilon = 0$, we recover the bound of~\citet{cortes2008stability},
which depends only on $\widehat{R}(\wh{\:f})$ and~$\beta$.
When $\varepsilon > 0$, we see from Eq.\,\ref{eq:bound} that
the two terms already present in the exact case are either
unchanged ($\widehat{R}(\wh{\:f})$) or increase only by a constant
factor $\beta.$
Because of the approximation, a new error term (the last one in
Eq.\,\ref{eq:bound}) is added to the bound, but we can see that it is negligible
compared to $\beta$. In fact, it converges to zero as
$\bigotime(\varepsilon^2/\ell^2(1-\varepsilon)^4)$ as $\ell$ grows and it is dominated by~$\beta$
for any constant value of~$\varepsilon$.
This means that increasing~$\varepsilon$ corresponds
to a constant increase in the bound, regardless of the size of the problem.
Consequently, $\vareps$ can be freely chosen to trade off accuracy and space 
complexity (Thm.\,\ref{thm:parallel-alg-main}) depending on the problem 
constraints.
Finally, because the eigenvalues present in the bound are the ones
of the original graph, any additional knowledge on the spectral properties of the input graph can be
easily included in the analysis. Therefore, it is straightforward to provide stronger guarantees for
\hfssparse when combined with assumptions on the graph generating model. Finally, we remark
the level of generality of this result that holds for the integration between \hfs and any $\varepsilon$-accurate
spectral sparsification method.
We postpone computational considerations to the following subsection.

\subsection{Generalization bounds for \rls}
Starting from the closed form solution of \rls (Eq.\,\ref{eq:laprls}) we can
replace the $\laplacian_{\Gg}$ matrix with a sparsified Laplacian $\laplacian_{\Hg}$ 
and using
a fast linear solver, compute an approximate solution $\wt{\:f} = (\lambda\laplacian_{\Hg} + \:I)^{-1}\:y$
in $\bigotime(n\log^2(n))$ time and $\bigotime(n\log(n))$ space. Finally, we can decompose the error
as $\normsmall{\:f^\star - \wt{\:f}}_2^2 \leq 
\normsmall{\:f^\star - \wh{\:f}}_2^2 +
\normsmall{\wh{\:f} - \wt{\:f}}_2^2.$
The first term can be bounded using classical results from
empirical process theory \cite{buhlmann2011statistics}.
We bound the second term in the following theorem.
\begin{restatable}{theorem}{restathmsmoothingbound}
\label{thm:smoothing-bound}
For an arbitrary graph $\Gg$ and its $(\varepsilon,\gamma)$-sparsifier,
let $\wh{\:f}$ be the \rls solution computed using $\:L_{\Gg}$
and $\wt{\:f}$ the solution computed using $\:L_{\Hg}$.
Then,
\begin{align*}
    \normsmall{\wt{\:f} - \wh{\:f}}_{2}^2
  \leq \frac{\varepsilon^2}{1-\varepsilon}\left(0.25 + \lambda\gamma\right)\left(\lambda\wh{\:f}^\transp\laplacian_{\Gg}\wh{\:f}
 + \lambda\gamma\normsmall{\wh{\:f}}_2^2\right)\!,
\end{align*}
where $\lambda$ is the regularization of \rls.
\end{restatable}
For $\epsilon$-sparsifiers, \citet{sadhanala_graph_2016} derive a similar bound
$\normsmall{\wt{\:f} - \wh{\:f}}_{2}^2 \leq \bigotime(\lambda\wh{\:f}^\transp\laplacian_{\Gg}\wh{\:f})$.
Setting $\gamma = 0$, we recover their bound up to constants.
When $\gamma > 0$ instead, additional error terms emerge due to the
introduced bias. In particular, the term $\lambda\gamma\normsmall{\wh{\:f}}_2^2$
depends on the norm of the exact solution~$\wh{\:f}$, which in turn depends
on the value of~$\lambda$. Nonetheless, when $\normsmall{\:f^\star}_{2}^2$
is small, as is the case in our experiments,
setting $\gamma = 1/\lambda$ makes this term a constant, which is reflected
by the good empirical performance.
Computationally, for both \hfsstable and \rls, passing from computing a solution
on the full graph to computing a solution on the sparsifier reduces the number
of edges, which makes the memory and runtime plummet. Moreover,
carefully distributing the sparsification process across multiple machines allows computing a final solution in a time \emph{independent}
from the number of edges, since the preprocessing sparsification step takes
only $\bigotime(n\log^3(n))$ time, and the solution step only $\bigotime(n\log^2(n))$.
Up to logarithmic terms, this results in an overall $\wt{\bigotime}(n)$ near-linear
runtime, without any assumptions on the input graph. For graphs with a particularly
favorable spectrum and problems with enough regularization, this is only
$\wt{\bigotime}(\deff(\gamma))$, resulting in a potentially sub-linear
runtime. This result, only possible due to a particular structure of
learning problems, opens up unexplored possibilities that would not be possible
for general graph problems.

\subsection{Bounds for other problems}
Many other problems can be well approximated using $(\varepsilon,\gamma)$-sparsifiers.
For example, the cost of a SC solution evaluated on
$\laplacian_{\Hg}$ is very close to the cost evaluated on $\laplacian_{\Gg}$.
\begin{proposition}\label{thm:proj-guarant}
For any rank $k$ orthogonal projection $\:F^\transp\:F$, if $\Hg$ is an $(\varepsilon,\gamma)$-sparsifier of $\Gg$, we have
\[
\Tr(\:F^\transp\:L_{\Hg}\:F)
\leq (1+\varepsilon)\Tr(\:F^\transp\:L_{\Gg}\:F) + \varepsilon\gamma k.
\]
\end{proposition}
Therefore, a clustering that well separates the sparsifier will also separate well
the true graph. Similarly, we can obtain strong approximation guarantees
for a variety of other Laplacian-based algorithms.
Regularized problems such as \ltr \cite{cortes2008stability},
Laplacian-regularized least squares, and Laplacian SVM \cite{belkin2005manifold}
are of particular interest since the additive $\gamma$ error is absorbed
by the regularization and it is possible to provide strong generalization
guarantees.

\begin{table*}
\vspace{-0.1in}
\begin{center}
\begin{small}
\begin{tabular}{|m{0.14\columnwidth}|c|c|c|c|c|c|}
\hline
\textit{Alg.} &\!\textit{Parameters}\!& \textit{$|\mathcal{E}|$ (x$10^6$)} & \textit{Err.} \!\ssl\!($\ell\!=\!346$) & \textit{Err.}\! \ssl\!($\ell\!=\!672$) & \textit{Err.}\! $D(\wt{\:f})$($\sigma\!=\!10^{-3}$) & \textit{Err.}\! $D(\wt{\:f})$ ($\sigma\!=\!10^{-2}$)\\
\hline
\hline
\exact & & 98.5 & 0.312 $\pm$ 0.022 &  0.286 $\pm$ 0.010 & 0.067 $\pm$ 0.0004 & 0.756 $\pm$ 0.006 \\
\kn & $k = 60$ & $15.7$ & 0.329 $\pm$ 0.0143 & 0.311 $\pm$ 0.027 & 0.172 $\pm$ 0.0004 & 0.822 $\pm$ 0.002 \\
\kn & $k = 90$ & $21.2$ & 0.334 $\pm$ 0.024 & 0.311 $\pm$ 0.024 & 0.125 $\pm$ 0.0002 & 0.811 $\pm$ 0.003 \\
\disre & $\gamma\!=\!0$, $\qbar \!=\! 100$ & $15 $ & 0.314 $\pm$ 0.0165 & 0.296 $\pm$ 0.015 & 0.068 $\pm$ 0.0003 & 0.758 $\pm$0.005 \\
\disre & $\gamma\!=\!0$, $\qbar \!=\! 150$ & $22.8$ & 0.314 $\pm$ 0.0158 & 0.310 $\pm$ 0.024 & 0.068 $\pm$ 0.0004 & 0.756 $\pm$ 0.005 \\
\disre & $\gamma\!=\!10^3$, $\qbar \!=\! 100$ & $7.3$ & $-$ & $-$ & 0.072 $\pm$ 0.0003 & 0.789 $\pm$ 0.005 \\
\disre & $\gamma\!=\!10^2$, $\qbar \!=\! 100$ & $11.8$ & $-$ & $-$ & 0.068 $\pm$ 0.0002 & 0.772 $\pm$ 0.004 \\
\disre & $\gamma\!=\!10$, $\qbar \!=\! 100$ & $14.4$ & $-$ & $-$ & 0.068 $\pm$ 0.0004 & 0.760 $\pm$ 0.004 \\
\hline
\hline
\end{tabular}
\vspace{-0.1in}
\end{small}
\caption{Results for the \ssl and the smoothing problems.}
\vspace{-0.12in}
\label{tab:results}
\end{center}
\end{table*}

\vspace{-0.05in}
\section{Experiments}
\vspace{-0.05in}
We empirically validate our theoretical findings by testing how $(\varepsilon,\gamma)$-sparsifiers  improves computational complexity without sacrificing final accuracy.

\textbf{Dataset.} We run experiments on the Amazon co-purchase graph \cite{sadhanala_graph_2016}. This graph fits our
setting: It cannot be generated from vectorial data
and is only artificially sparse, since the crawler that created it had
no access to the true private co-purchase network held by Amazon.
To compensate,
\citet{Gleich2015robustifying} use a densification procedure that given the graph adjacency
matrix $\:A_{\Gg}$, computes all $k$-step neighbors $\:A_{\Gg,k} \triangleq \sum_{s=1}^k
\:A_{\Gg}^s$. We make the graph unweighted for numerical
stability. The final graph has $n = 334,863$ nodes and $m = 98,465,352$ edges, with an average
degree of $294$.
We followed an approach similar to~\citet{sadhanala_graph_2016}
and introduce a hand-designed smooth signal as a target.
We then
perform $2000$ iterations of the power method to compute an approximation
of the smallest eigenvector $\:v_{\min}$, which is used as a smooth function over the graph.

\textbf{Baselines.} For all setups, we compute an ``exact'' 
solution (up to convergence error) using a fast linear solver. Computing this \exact baseline requires
$\bigotime(m\log(n))$ time and $\bigotime(m)$ space and achieves the
best performance.
Afterwards, we compare three different sparsification procedures to evaluate if
they can accelerate computation while preserving accuracy. We run \disre with different values of $\gamma$ depending on the setting.
For empirically strong heuristics, we attempted to uniformly subsample
the edges, but at the sparsity level achieved by the other methods,
the uniformly sampled sparsifier is disconnected and highly inaccurate.
Instead, we compare to the state-of-the-art $k$-neighbors (\kn) heuristic
 by~\citet{sadhanala_graph_2016}, which is just as fast as
uniform sampling and more accurate in practice.

\textbf{Experimental procedure.} 
We repeat
each experiment 10 times with different sparsifiers and report the average
performance of $\wt{\:f}$ on the specific task and its standard deviation.
More details on experiments are given in the Appendix~\ref{appE}.

\vskip 0.1in
\subsection{Laplacian smoothing with Gaussian noise}


We set $\:f^\star = \:v_{\min}$ and test different levels of noise, $\log_{10}(\sigma) \in \{-3, -2, -1,0\}$.
After constructing the sparsifier $\Hg$, we compute an approximate solution $\wt{\:f}$
using \rls (Eq.\,\ref{eq:hfs.original}) with
$\lambda \in \{10^{-3}, 10^{-2}, 10^{-1},1, 10\}$. We measure the performance by the squared error $D(\wt{\:f}) = \normsmall{\:f^\star - \wt{\:f}}_2^2$. As $\normsmall{\:f^\star}_2^2 = \normsmall{\:v_{\min}}_2^2 = 1$, good values of $D(\wt{\:f})$ should be below 1.

\textbf{Accuracy.}
In the interest of space, in Tab.\,\ref{tab:results}, we report results for $\sigma=\{0.001, 0.01\}$ and the best regularization $\lambda$ for each method. We first notice that all sparsifiers are considerably smaller than the original graph, keeping only a small fraction of its edges. The smallest sparsifiers are obtained by \disre when $\gamma$ is large. The comparison with \disre with $\gamma=0$ (i.e., $\vareps$-sparsifier) confirms that the additive error translates into an extra compression of the resulting sparsifier. This also impacts the accuracy which degrades as~$\gamma$ increases. Nonetheless, we notice that while $\vareps$-sparsifiers perfectly match the accuracy of the exact method, even for large $\gamma$ (and thus much smaller graph), \disre still outperforms $\kn$, which has a significantly worse accuracy. Finally, we note that for $\gamma=0$, the impact of $\qbar$ is as expected: Increasing $\qbar$ increases the size of the sparsifier and slightly improves the performance.


\textbf{Computational complexity.}
All algorithms require 90\texttt{s} to load the graph from disk.
The preprocessing phase of \kn takes slightly less than 1\texttt{min}, while \disre's takes 12\texttt{min} on 4 machines. 
For the solving step, \exact is unsurprisingly the slowest, requiring 12\texttt{min}
to compute an $\wh{\:f}$ solution. Both $\kn$ and $(\varepsilon,\gamma)$-sparsifiers
require 1--2\texttt{min}, depending on the number of edges preserved.
Overall, preprocessing the graph with \disre before computing a solution
does not introduce any overhead compared to \exact (both take roughly 12\texttt{min}). We notice that while $\kn$ is overall faster, the time for \disre could be easily reduced by increasing the
number of parallel processes when computing effective resistances or with a better network
topology allowing point-to-point communication. 
Moreover, once we have
access to an accurate $\varepsilon$-sparsifier, it is easier to solve
problem repeatedly, e.g.,\@ to cross-validate regularization.
For example, computing a solution for 4 different values of $\lambda$ (see the appendix)
is crucial for good performance and requires 48\texttt{min} for \exact and only 20\texttt{min} for \disre.
Finally, memory usage is reduced by a factor of~3 as \exact requires over 30GB of memory to
execute while \disre never exceeds 10GB. We expect these advantages to only grow
larger as we scale to larger graphs.

\subsection{\ssl with harmonic function solution}

We also test \disre on a \ssl problem. The labels are generated taking the sign of $\:f^\star = \:v_{\min}$ and $\ell\in\{20, 346, 672, 1000\}$ labels are revealed. The labeled nodes are chosen at random so that $0$ and $1$ labels are balanced in the dataset. We run \hfsstable with $\lambda \in \{10^{-6}, 10^{-4}, 10^{-2},1\}$. In Tab.\,\ref{tab:results}, we report results for $\ell=\{346,672\}$ and the best $\lambda$ for each method. We run \disre with $\gamma=0$ as \hfsstable does not have any regularization and $\vareps$-sparsifiers are preferable. \todoaout{I don't understand myself what this means... In which sense \hfs is not regularized and smoothing is?} The average size of the sparsifiers is the same as before as they are agnostic to the learning task. Similar to the smoothing case, \disre achieves a performance that closely approximates the exact solution, despite the significant compression of the original graph. Furthermore, the effectiveness of the $\vareps$-sparsifier returned by \disre is confirmed by its comparison with \kn, whose error is significantly worse. Finally, we notice that the computational analysis in the previous section holds for \ssl as well. In fact, although the learning task is different, we use the same SSD solver to compute the \hfs  and thus the running time are comparable in the two tasks.

\newpage
\section*{Acknowledgements}
Experiments presented in this paper were carried out using the Grid'5000 testbed, supported by a scientific interest group hosted by Inria and including CNRS, RENATER, and several universities as well as other organizations (see \url{https://www.grid5000.fr}).
The research presented was also supported by European CHIST-ERA project DELTA, French Ministry of
Higher Education and Research, Nord-Pas-de-Calais Regional Council,
Inria and Otto-von-Guericke-Universit\"at Magdeburg associated-team north-european project Allocate, and French National Research Agency projects ExTra-Learn (n.ANR-14-CE24-0010-01) and BoB (n.ANR-16-CE23-0003).
I.\,Koutis is supported by NSF CAREER award CCF-1149048.

\bibliography{aistat_sparse_ssl}
\bibliographystyle{newicml2018}

\newpage
\appendix
\onecolumn

\section{Proofs for the statements in Sec.\,\ref{sec:downstream}}
\label{proof4}

In the proofs, we make several uses the following reformulation of Def.\,\ref{def:eps-sparsifier}.
\begin{proposition}\label{prop:ordering-to-norm-bound}
 A sub-graph $\Hg$ is a $(\varepsilon,\gamma)$-sparsifier of $\Gg$ if and only if
 \begin{align*}
(1-\vareps) \:L_\Gg - \vareps\gamma \:I \preceq \:L_\Hg \preceq (1 + \vareps) \:L_\Gg + \vareps\gamma \:I
&&\iff &&\normsmall{(\laplacian_{\Gg} + \gamma\:I)^{-1/2}(\laplacian_{\Hg} - \laplacian_{\Gg})(\laplacian_{\Gg} + \gamma\:I)^{-1/2}}^2_2 \leq \varepsilon.
\end{align*}
\end{proposition}
\restathmsparsesslgeneralization*

\begin{proof}[Proof of Thm.~\ref{thm:sparse-ssl-generalization}]
\textbf{Step 1 (generalization of stable algorithms).}
Let $\beta$ be the stability of \hfsstable when using the sparsified Laplacian $\:\:L_{\Hg}$ in place of $\:\:L_{\Gg}$. Then using the result of~\citet{cortes2008stability}, we have that with probability at least $1 - \delta$ (w.r.t.\,the randomness of the labeled set $\trainingset$) the solution $\wt{\:f}$ satisfies
\begin{align*}
R(\wt{\:f}) \leq \widehat{R}(\wt{\:f}) + \beta + \left(2\beta + \frac{c^2(\ell+u)}{\ell u}\right)\sqrt{\frac{\pi(\ell,u)\log(1/\delta)}{2}}\cdot
\end{align*}
In order to obtain the final result, 
   we first derive an upper bound on the stability $\beta$ and relate the empirical error of $\wt{\:f}_{\text{STA}}$ to the one of $\wh{\:f}_{\text{STA}}$.
Furthermore, it can be shown that if we center the vector of labels
$\wt{\:y}_\calS \triangleq \:y_\calS - \ol{\:y}_\calS$, with
$\ol{\:y} \triangleq \frac{1}{\ell}\:y_\calS^\transp \:1$, then the solution of \hfsstable
can be rewritten in closed form as
$\wh{\:f}_{\text{STA}} = (\gamma \triangleq \:L_\Gg + \:I_\trainingset)^+ (\wt{\:y}_\calS - \mu \:1) = \left(\:P (\gamma l \:L_\Gg + \:I_\trainingset)\right)^+ \wt{\:y}_\calS$.

   \textbf{Step 2 (stability).}
   The bound on the stability follows similar steps as in the analysis of \hfsstable by~\citet{belkin2004regularization} integrated with the properties of spectral sparsifiers reported in Def.\,\ref{def:eps-sparsifier}.
   Let $\trainingset$ and $\trainingset'$ be two labeled sets only that only differ by one element and $\wt{\:f}$ and $\wt{\:f}'$ be the solutions obtained by running \hfsstable using $\:\:L_{\Hg}$ and $\calS$ and $\calS'$ respectively. 
   
   Without loss of generality, we assume that $\:I_{\trainingset}(\ell,\ell) = 1$ and $\:I_{\trainingset}(\ell+1,\ell+1) = 0$, and the opposite for $\:I_{\trainingset'}$. The original proof of~\citet{cortes2008stability} showed that the stability $\beta$ can be bounded as $\beta\leq\normsmall{\wt{\:f} - \wt{\:f}'}$. In the following, we show that the difference between the solutions $\wt{\:f}$ and $\wt{\:f}'$ and thus the stability of the algorithm, is strictly related to eigenvalues of the sparse graph $\Hg$. Let $\:A\triangleq\:P(\ell \gamma \:L_{\Hg}+\:I_{\trainingset})$ and $\:B\triangleq \:P(\ell \gamma \:L_{\Hg}+\:I_{\trainingset'})$. We remind that if the labels are centered, the solutions of \hfsstable can be conveniently written as $\wt{\:f}=\:A^{-1}\wt{\:y}_{\trainingset}$ and $\wt{\:f}'=\:B^{-1}\wt{\:y}_{\trainingset'}$. As a result, the difference between the solutions can be written as
   \begin{align}\label{eq:stability.step1}
   \Vert\wt{\:f}-\wt{\:f}'\Vert&=\Vert \:A^{-1}\wt{\:y}_{\trainingset}-\:B^{-1}\wt{\:y}_{\trainingset'}\Vert
   \leq \Vert \:A^{-1}(\wt{\:y}_{\trainingset}-\wt{\:y}_{\trainingset'})\Vert+\Vert \:A^{-1}\wt{\:y}_{\trainingset'}-\:B^{-1}\wt{\:y}_{\trainingset'}\Vert.
   \end{align}
   %
   %
   Let us consider any vector $\:f \in \funcspace$, since the null space of a Laplacian $\:L_{\Hg}$ is the one vector $\:1$ and $\:P = \:L_{\Hg}\:L_{\Hg}^+$, then $\:P \:f = \:f$. Thus, we have
   \begin{align}
   &\Vert \:P(\ell\gamma \:L_{\Hg}+\:I_{\calS})\:f\Vert\stackrel{(1)}{\geq}\Vert \:P\ell\gamma \:L_{\Hg}\:f\Vert \!-\! \Vert \:P\:I_{\calS}\:f\Vert
   \stackrel{(2)}{\geq} \Vert \:Pl\gamma \:L_{\Hg}\:f\Vert \!-\! \Vert\:f\Vert \stackrel{(3)}{\geq}(\ell \gamma \lambda_{2}(\Hg)\!-\!1)\Vert \:f\Vert, \label{eq:lower.bound}
   \end{align}
   where $(1)$ follows from the triangle inequality 
   and $(2)$ follows from the fact that $\Vert \:P\:I_{\calS}\:f\Vert
   \leq \Vert\:f\Vert$, since the largest eigenvalue of the projection matrix
   $\:P$ is one and the norm of $\:f$ restricted on $\calS$ is smaller
   than the norm of $\:f$. Finally, (3) follows from the fact that $\Vert
   \:P \:L_\Hg \:f \Vert = \Vert \:L_\Hg \:L_\Hg^+ \:L_\Hg \:f \Vert = \Vert
   \:L_\Hg \:f\Vert$ and since $\:f$ is orthogonal to the null space of $\:L_\Hg$ then
   $\Vert \:L_\Hg \:f\Vert \geq \lambda_2(\Hg) \Vert\:f\Vert$, where
   $\lambda_2(\Hg)$ is the smallest non-zero eigenvalue of $\:L_\Hg$. At this point
   we can exploit the spectral guarantees of the sparsified Laplacian $\:\:L_\Hg$
   and we have that $\lambda_2(\Hg) \geq
   (1-\varepsilon)\lambda_2(\Gg)$.
   As a result, we have an upper bound on the spectral radius of the inverse operator $(\:P(\ell\gamma \:L_{\Hg}+\:I_{\trainingset}))^{-1}$ and thus
   \begin{align*}
   \Vert \:A^{-1}(\:y_{\trainingset}-\:y_{\trainingset'})\Vert&\leq \frac{4M}{\ell\gamma (1 - \varepsilon)\lambda_{2}(\Gg)-1}\CommaBin
   \end{align*}
   where the first step follows from Eq.\,\ref{eq:lower.bound} since both
   $\wt{\:y}_{\trainingset}$ and $\wt{\:y}_{\trainingset'}$ are centered and thus
   $(\:y_{\trainingset}-\:y_{\trainingset'})\in\funcspace$ and the second step is
   obtained by bounding
   $\Vert\wt{\:y}_{\trainingset}-\wt{\:y}_{\trainingset'}\Vert\leq
   \Vert\:y_{\trainingset}-\:y_{\trainingset'}\Vert +
   \Vert\ol{\:y}_{\trainingset}-\ol{\:y}_{\trainingset'}\Vert \leq 4M$. The second
   term in Eq.\,\ref{eq:stability.step1} can be bounded as
   \[
   \Vert \:A^{-1}\wt{\:y}_{\trainingset'}-\:B^{-1}\wt{\:y}_{\trainingset'}\Vert=\Vert \:B^{-1}(B-A)\:A^{-1}\wt{\:y}_{\trainingset'}\Vert
       =\Vert \:B^{-1}\:P(\:I_{\trainingset}-\:I_{\trainingset'})\:A^{-1}\wt{\:y}_{\trainingset'}\Vert
       \leq\frac{1.5 M\sqrt{\ell}}{(\ell \gamma (1-\varepsilon)\lambda_{2}(\Gg)-1)^{2}}\CommaBin
   \]
       where we used $\Vert\wt{\:y}_{\trainingset'}\Vert\leq \Vert\:y_{\trainingset'}\Vert + \Vert\ol{\:y}_{\trainingset'}\Vert \leq 2M\sqrt{\ell}$, $\Vert \:P(\:I_{\trainingset}-\:I_{\trainingset'})\Vert \leq \sqrt{2} < 1.5$ and we applied Eq.\,\ref{eq:lower.bound} twice.
       Putting it all together we obtain the stated bound.

       \textbf{Step 3 (empirical error).} The other element effected by the sparsification is the empirical error $\wh{R}(\wt{\:f})$.
       We first recall that
       $\:P = \:L_\Gg^+ \:L_\Gg = \:L_{\Gg}^{-1/2}\:L_{\Gg}\:L_{\Gg}^{-1/2}$ (and
               equivalently with $\Gg$ replaced by $\Hg$) and we introduce
       $\wt{\:P} = \:L_{\Gg}^{-1/2}\:L_{\Hg}\:L_{\Gg}^{-1/2}$
       Let $\wt{\:A} \triangleq \:P(\ell \gamma \:L_{\Hg}+\:I_{\trainingset})$,
       $\wh{\:A} \triangleq \:P(\ell \gamma \:L_{\Gg}+\:I_{\trainingset})$, then we can
       rewrite the empirical error as
       \begin{align*}
       \wh{R}(\wt{\:f}) &= \tfrac{1}{\ell} \normsmall{\:I_\trainingset \wt{\:f} -\:I_\trainingset \wh{\:f} +\:I_\trainingset \wh{\:f}  - \wt{\:y}_{\trainingset}}^2\nonumber\\
           &\leq \tfrac{1}{\ell} \normsmall{\:I_\trainingset \wh{\:f}  - \wt{\:y}_{\trainingset}}^2 + \tfrac{1}{\ell} \normsmall{\:I_\trainingset \wt{\:f} -\:I_\trainingset \wh{\:f} }^2 \\
           &\leq \wh{R}(\wh{\:f}) + \tfrac{1}{\ell} \normsmall{\:I_\trainingset(\wt{\:A}^{-1} -\wh{\:A}^{-1}) \wt{\:y}_{\trainingset} }^2\nonumber\\
           &\leq \wh{R}(\wh{\:f}) + \tfrac{1}{\ell} \normsmall{\wh{\:A}^{-1}(\wh{\:A} -\wt{\:A})\wt{\:A}^{-1} \wt{\:y}_{\trainingset} }^2\\
           &= \wh{R}(\wh{\:f}) + \tfrac{\ell^{2}\gamma^{2}}{\ell} \normsmall{\wh{\:A}^{-1}(\:P(\:L_\Gg - \:L_\Hg))\wt{\:A}^{-1} \wt{\:y}_{\trainingset} }^2\\
           &= \wh{R}(\wh{\:f}) + \ell\gamma^{2} \normsmall{\wh{\:A}^{-1}(\:P(\:L_\Gg - \:L_\Hg)\:P)\wt{\:A}^{-1} \wt{\:y}_{\trainingset} }^2,
           \end{align*}
           where in the last passagen we use $\:L_\Gg \:P = \:L_\Gg$ and $\:L_\Hg \:P = \:L_\Hg$.
           To bound the second term, we derive
           \begin{align*}
           \normsmall{\wh{\:A}^{-1}\:P(\:L_\Gg - \:L_\Hg)\:P\wt{\:A}^{-1} \wt{\:y}_{\trainingset} }^2
               &\stackrel{(1)}{=}  \normsmall{\wh{\:A}^{-1}\:L_\Gg^{1/2} \:L_\Gg^{-1/2}(\:L_\Gg - \:L_\Hg)\:L_\Gg^{-1/2} \:L_\Gg^{1/2}\wt{\:A}^{-1} \wt{\:y}_{\trainingset} }^2\\
               &\stackrel{(2)}{=}  \normsmall{\wh{\:A}^{-1}\:L_\Gg^{1/2}\:P \:L_\Gg^{-1/2}(\:L_\Gg - \:L_\Hg)\:L_\Gg^{-1/2}\:P \:L_\Gg^{1/2}\wt{\:A}^{-1} \wt{\:y}_{\trainingset} }^2\\
               &\stackrel{(3)}{=}  \normsmall{\wh{\:A}^{-1}\:L_\Gg^{1/2}\:P(\:P - \wt{\:P})\:P \:L_\Gg^{1/2}\wt{\:A}^{-1} \wt{\:y}_{\trainingset} }^2\\
               &\leq   \normsmall{\wh{\:A}^{-1}\:L_\Gg^{1/2}\:P}^2 \normsmall{\:P - \wt{\:P}}^{2} \normsmall{\:P\:L_\Gg^{1/2}\wt{\:A}^{-1}}^{2} \normsmall{\wt{\:y}_{\trainingset} }^2\\
               &\stackrel{(4)}{\leq} \ell M^2\varepsilon^{2}    \normsmall{\wh{\:A}^{-1}\:L_\Gg^{1/2}\:P}^2 \normsmall{\wt{\:A}^{-1}\:L_\Gg^{1/2}\:P}^{2},
               \end{align*}
               where in $(1)$ and $(2)$, we use the definition of $\:P$, in $(3)$ we use the
               definition of $\wt{\:P}$, while in $(4)$
               we use the fact that Def.\,\ref{def:eps-sparsifier} implies that
               $(1-\vareps) \:P \preceq \wt{\:P} \preceq (1 + \vareps) \:P$
               and thus the largest eigenvalue of $\:P - \wt{\:P}$ is
               $\varepsilon$ and $\Vert \:P- \wt{\:P}\Vert^2 \leq \vareps^2$.
               We need now to bound $\normsmall{\wt{\:A}^{-1}\:L_\Gg^{1/2}\:P}^{2} = \normsmall{\wt{\:A}^{-1}\:L_\Gg^{1/2}}^{2}$\!. From the definition of spectral norm,
               \begin{align*}
               \normsmall{\wt{\:A}^{-1}\:L_\Gg^{1/2}}^{2}
               &= \max_{\normsmall{\:x} = 1} \:x^\transp\:L_\Gg^{1/2}\wt{\:A}^{-1}\wt{\:A}^{-1}\:L_\Gg^{1/2}\:x
               = \max_{\normsmall{\:x} = 1} \:x^\transp\wt{\:A}^{-1}\:L_\Gg\wt{\:A}^{-1}\:x\\
                 &\leq \frac{1}{1-\varepsilon}\max_{\normsmall{\:x} = 1} \:x^\transp\wt{\:A}^{-1}\:L_\Hg\wt{\:A}^{-1}\:x
                 = \frac{1}{1-\varepsilon}\normsmall{\wt{\:A}^{-1}\:L_\Hg^{1/2}}^{2}
                 = \frac{1}{1-\varepsilon}\normsmall{\wt{\:A}^{-1}\:L_\Hg^{1/2}\:P}^{2}\!.
                 \end{align*}

                 Similarly to Eq.\,\ref{eq:lower.bound}, finding a lower bound on
                 $\Vert \wt{\:A}\:L_{\Hg}^{-1/2}\:P\:x\Vert$
                 for all $\:x$ is equivalent to find a lower bound for all $\:f \in \funcspace$
                 to
                 \begin{align*}
                 \Vert \:P(\ell\gamma \:L_{\Hg}+\:I_{\calS})\:L_{\Hg}^{-1/2}\:f\Vert
                     &{\geq}\Vert \:P\ell\gamma \:L_{\Hg}\:L_{\Hg}^{-1/2}\:f\Vert \!-\! \Vert \:P\:I_{\calS}\:L_{\Hg}^{-1/2}\:f\Vert\\
                     &{\geq} \Vert \:P\ell\gamma \:L_{\Hg}\:L_{\Hg}^{-1/2}\:f\Vert \!-\! \Vert \:L_{\Hg}^{-1/2} \:f\Vert\\
                     &{\geq}\left(\ell \gamma \sqrt{\lambda_{2}(\Hg)}\!-\!\frac{1}{\sqrt{\lambda_{2}(\Hg)}}\right)\Vert \:f\Vert\\
                     &{\geq}\frac{1}{\sqrt{\lambda_{2}(\Hg)}} \left(\ell \gamma\lambda_{2}(\Hg)\!-1\!\right)\Vert \:f\Vert\\
                     &{\geq}\frac{1}{\sqrt{(1+\varepsilon)\lambda_{2}(\Gg)}} \left(\ell \gamma(1-\varepsilon)\lambda_{2}(\Gg)\!-1\!\right)\Vert \:f\Vert.
                     \end{align*}
                     Similarly, we can show that
                     \begin{align*}
                     \Vert \:P(l\gamma \:L_{\Gg}+\:I_{\calS})\:L_{\Gg}^{-1/2}\:f\Vert
                         &{\geq}\Vert \:P\ell\gamma \:L_{\Gg}\:L_{\Gg}^{-1/2}\:f\Vert \!-\! \Vert \:P\:I_{\calS}\:L_{\Gg}^{-1/2}\:f\Vert\\
                         &{\geq} \Vert \:P\ell\gamma \:L_{\Gg}\:L_{\Gg}^{-1/2}\:f\Vert \!-\! \Vert \:L_{\Gg}^{-1/2} \:f\Vert\\
                         &{\geq}\left(\ell\gamma \frac{\lambda_{2}(\Gg)}{\sqrt{\lambda_{2}(\Gg)}}\!-\!\frac{1}{\sqrt{\lambda_{2}(\Gg)}}\right)\Vert \:f\Vert\\
                         &{\geq}\frac{1}{\sqrt{\lambda_{2}(\Gg)}} \left(\ell \gamma \lambda_{2}(\Gg)\!-1\!\right)\Vert \:f\Vert\\
                         &{\geq}\frac{1}{\sqrt{(1+\varepsilon)\lambda_{2}(\Gg)}} \left(\ell \gamma (1-\varepsilon)\lambda_{2}(\Gg)\!-1\!\right)\Vert \:f\Vert.
                         \end{align*}
                         Taking this and putting all together gives
                         \begin{align*}
                         \wh{R}(\wt{\:f}) \leq \wh{R}(\wh{\:f}) + \frac{1}{1-\varepsilon}\frac{(1+\varepsilon)^2\varepsilon^2\ell^2\gamma^{2}\lambda_{2}(\Gg)^{2} M^2}{(\ell \gamma (1-\varepsilon)\lambda_{2}(\Gg)\!-1)^{4}}\cdot
                         \end{align*}

                         Combining the three steps above concludes the proof.

                         \end{proof}

\restathmsmoothingbound*
\begin{proof}
We need to bound the distance between $\wt{\:f}$ and $\wh{\:f}$. Using the definition, the fact that $\:A^{-1} - \:B^{-1} = \:B^{-1}(\:B - \:A)\:A^{-1}$ and
collecting $(\laplacian_{\Gg} + \gamma\:I)^{1/2}$ we have
\begin{align*}
\normsmall{\wt{\:f} - \wh{\:f}}_{2}^2
 &= \normsmall{(\lambda\laplacian_{\Hg} + \:I)^{-1} - \lambda\laplacian_{\Gg} + \:I)^{-1})\:y}_{2}^2
 = \normsmall{(\lambda\laplacian_{\Hg} + \:I)^{-1}(\lambda\laplacian_{\Hg} - \lambda\laplacian_{\Gg})(\lambda\laplacian_{\Gg} + \:I)^{-1})\:y}_{2}^2\\
 &= \lambda^2\normsmall{(\lambda\laplacian_{\Hg} + \:I)^{-1}(\lambda\laplacian_{\Hg} - \laplacian_{\Gg})(\lambda\laplacian_{\Gg} + \:I)^{-1})\:y}_{2}^2\\
 &= \lambda^2\normsmall{(\lambda\laplacian_{\Hg} + \:I)^{-1}(\laplacian_{\Gg} + \gamma\:I)^{1/2}(\laplacian_{\Gg} + \gamma\:I)^{-1/2}(\laplacian_{\Hg} - \laplacian_{\Gg})(\laplacian_{\Gg} + \gamma\:I)^{-1/2}(\laplacian_{\Gg} + \gamma\:I)^{1/2}\wh{\:f}}_{2}^2.
 \end{align*}
 Then, using Prop.\,\ref{prop:ordering-to-norm-bound},
\begin{align*}
 \lambda^2\normsmall{(\lambda\laplacian_{\Hg} + \:I)^{-1}(\laplacian_{\Gg} +&\gamma\:I)^{1/2}(\laplacian_{\Gg} + \gamma\:I)^{-1/2}(\laplacian_{\Hg} - \laplacian_{\Gg})(\laplacian_{\Gg} + \gamma\:I)^{-1/2}(\laplacian_{\Gg} + \gamma\:I)^{1/2}\wh{\:f}}_{2}^2\\
 &\leq \varepsilon^2\lambda^2\normsmall{(\lambda\laplacian_{\Hg} + \:I)^{-1}(\laplacian_{\Gg} + \gamma\:I)^{1/2}}_2^2\wh{\:f}^\transp(\laplacian_{\Gg} + \gamma\:I)\wh{\:f}\\
 &= \frac{\varepsilon^2}{1-\varepsilon}\lambda^2\normsmall{(\lambda\laplacian_{\Hg} + \:I)^{-1}((1-\varepsilon)\laplacian_{\Gg} -\varepsilon\gamma\:I + \gamma\:I)(\lambda\laplacian_{\Hg} + \:I)^{-1}}_2\wh{\:f}^\transp(\laplacian_{\Gg} + \gamma\:I)\wh{\:f}\\
 &\leq \frac{\varepsilon^2}{1-\varepsilon}\lambda^2\normsmall{(\lambda\laplacian_{\Hg} + \:I)^{-1}(\laplacian_{\Hg} + \gamma\:I)(\lambda\laplacian_{\Hg} + \:I)^{-1}}_2\wh{\:f}^\transp(\laplacian_{\Gg} + \gamma\:I)\wh{\:f}\\
 & =\frac{\varepsilon^2}{1-\varepsilon}\lambda^2\max_i\left\{\frac{\lambda_i(\laplacian_{\Hg}) + \gamma}{(\lambda\lambda_i(\laplacian_{\Hg}) + 1)^2}\right\}\wh{\:f}^\transp(\laplacian_{\Gg} + \gamma\:I)\wh{\:f}\\
 & =\frac{\varepsilon^2}{1-\varepsilon}\lambda\max_i\left\{\frac{\lambda\lambda_i(\laplacian_{\Hg})}{(\lambda\lambda_i(\laplacian_{\Hg}) + 1)^2} + \frac{\lambda\gamma}{(\lambda\lambda_i(\laplacian_{\Hg}) + 1)^2}\right\}\wh{\:f}^\transp(\laplacian_{\Gg} + \gamma\:I)\wh{\:f}\\
 & \leq \frac{\varepsilon^2}{1-\varepsilon}\lambda\left(\max_i\left\{\frac{\lambda\lambda_i(\laplacian_{\Hg})}{(\lambda\lambda_i(\laplacian_{\Hg}) + 1)^2}\right\} + \max_i\left\{\frac{\lambda\gamma}{(\lambda\lambda_i(\laplacian_{\Hg}) + 1)^2}\right\}\right)\wh{\:f}^\transp(\laplacian_{\Gg} + \gamma\:I)\wh{\:f}\\
 & \leq \frac{\varepsilon^2}{1-\varepsilon}\lambda\left(0.25 + \lambda\gamma\right)\wh{\:f}^\transp(\laplacian_{\Gg} + \gamma\:I)\wh{\:f}
  \leq \frac{\varepsilon^2}{1-\varepsilon}(\left(0.25 + \lambda\gamma\right)\lambda\wh{\:f}^\transp\laplacian_{\Gg}\wh{\:f} + \left(0.25 + \lambda\gamma\right)\lambda\gamma\:I),
\end{align*}
which concludes the proof.
\end{proof}

%
\section{Proof of Thm.\,\ref{thm:parallel-alg-main}}
\label{seckl}

\restathmparallelalgmain*
The proof of Thm.\,\ref{thm:parallel-alg-main} 
assembles the techniques from two sources. 
First, it is based on the analysis of the algorithm of~\cite{kelner_spectral_2013}, that shows it produces a spectral sparsifier in high probability.  We have previously published this analysis as a technical report \cite{calandriello2016analysis}
for the case of $\gamma = 0$. In this alternative proof, we rigorously take into account the dependencies across subsequent resparsifications using martingale inequalities, fixing a flaw in the original analysis.
Second, it also copies the steps
for the analysis of \squeak \citep{calandriello_disqueak_2017}
which is an algorithm for general kernel sparsificatioN.
In particular, Alg.\,\ref{alg:kl_dist} is an instantiation of \squeak
to the special case of graph sparsification. While \squeak's analysis by
\citet{calandriello_disqueak_2017} holds in general, in \disre we can
exploit the specific structure of graph Laplacians to perform a few optimizations. For completeness and ease of verification, we restate here \squeak's original proof with the necessary modifications to the notation, constants, and relevant quantities to target the graph learning setting.

\begin{figure}[t]
\begin{tabular}{m{0.33\textwidth}|m{0.33\textwidth}|m{0.33\textwidth}}
\subfigure[arbitrary tree]{\includegraphics[height=0.45\textwidth]{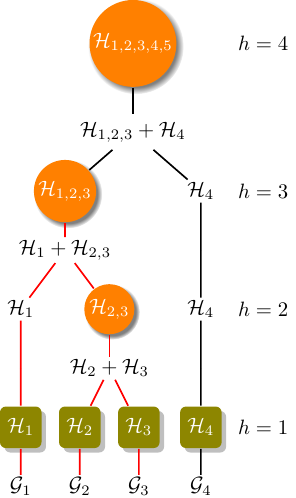} \label{fig:merge-trees-exp}}
& \subfigure[sequential tree]{\includegraphics[height=0.45\textwidth]{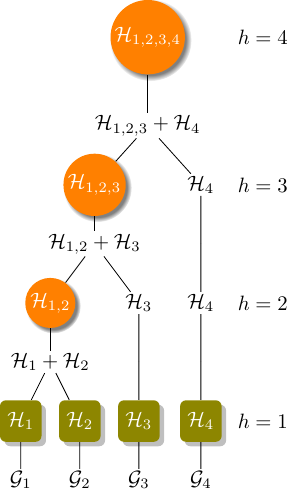}}
& \subfigure[minimum depth tree]{\includegraphics[height=0.45\textwidth]{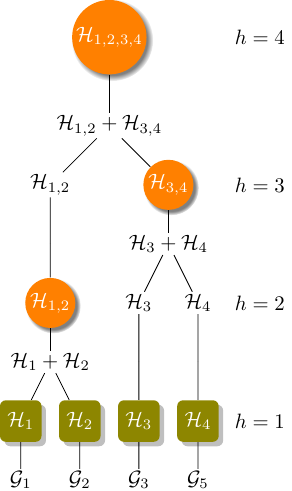}}
\end{tabular}
\caption{Merge trees for Alg.\,\ref{alg:kl_dist}.}\label{fig:merge-trees}
\end{figure}

\textbf{Merge trees} We first formalize the random process induced by Alg.\,\ref{alg:kl_dist}.

We partition $\Gg$ into $k$ disjoint sub-graphs $\Gg_i$
of size $n_i$, such that $\Gg = \cup_{e=1}^k\Gg_i$.
For each sub-graph $\Gg_i$, we construct an initial sparsifier
$\coldict_{\{1,i\}} \triangleq \{(j,\wt{p}_{0,i} = 1, q_{0,i} = \wb{q}) : j \in \Gg_i\}$
by inserting all edges from $\Gg_i$ into $\coldict_{1,i}$ with
weight $\wt{p}_{0,i} \triangleq 1$ and number of copies $q_{0,i} \triangleq \wb{q}$.
It is easy to see that $\coldict_{\{1,i\}}$ is an $(0,0)$-accurate sparsifier,
and we can split the graph into small enough sub-graphs to make sure that it can be
easily stored and manipulated in memory.
Afterwards, the initial sparsifiers $\coldict_{\{1,i\}}$ are included into the sparsifier pool $\dictpool_{1}$.

At iteration $h$, the inner loop of Alg.\,\ref{alg:kl_dist}
arbitrarily chooses  two sparsifiers from
$\dictpool_h$ and merges them into a new sparsifier.
Any arbitrary sequence of merges can be described by a full binary tree, i.e.,
a binary tree where each node is either a leaf or has exactly two children.
Fig.\,\ref{fig:merge-trees} shows several different merge trees corresponding
to different choices for the order of the merges.
Note that starting from $k$ leaves, a full binary tree will always have exactly $k-1$ internal
nodes. Therefore, regardless of the structure of the merge tree, we can always
transform it into a tree of depth $k$, with all the initial sparsifiers
$\coldict_{1,i}$ as leaves on its deepest layer.
After this transformation, we index the tree nodes using their height
(longest path from the node to a leaf, also defined as the depth of the tree minus the depth of the node),
where leaves have height 1 and the root has height $k$.
We can also see that at each layer, there is a single sparsifier merge,
and the size of $\dictpool_h$ (number of sparsifiers present at layer $h$)
is $|\dictpool_h| = k-h+1$.
Therefore, a node corresponding to a sparsifier is uniquely identified with two indices $\{h,\ell\}$, where
$h$ is the height of the layer and $\ell\leq |\dictpool_h|$ is the index of the node in the
layer. For example, in Fig.\,\ref{fig:merge-trees-exp}, the node containing
$\coldict_{1,2,3}$ is indexed as $\{3,1\}$, and the highest
node containing $\coldict_4$ is indexed as $\{3,2\}$.

We also define the graph
$\Gg_{\{h,\ell\}}$ as the union of all sub-graph $\Gg_{\ell'}$ that are
 reachable from node $\{h,\ell\}$ as leaves.
For example, in Fig.\,\ref{fig:merge-trees-exp}, sparsifier $\coldict_{1,2,3}$
in node $\{3,1\}$ is constructed starting from all edges in $\Gg_{\{3,1\}} = \Gg_1 \cup \Gg_2 \cup \Gg_3$,
where we highlight in red the descendant tree.
We now define $\laplacian^h$ as the block diagonal matrix where each diagonal block $\laplacian_{\Gg_{\{h,\ell\}}}$ is the
Laplacian constructed on
$\Gg_{\{h,\ell\}}$.
Without loss of generality, we will assume that each of the sub-graphs $\Gg_i$
is connected and spans all the $n$ nodes in the graph. This simplifies the notation for
the $\laplacian_{\Gg_{\{h,\ell\}}}$ matrices, making them all $n \times n$ matrices.
If this is not the case, the whole proof still follows through with different
number of nodes $n_{\{h,\ell\}}$ for each $\Gg_{\{h,\ell\}}$.
Again, from Fig.\,\ref{fig:merge-trees}, $\laplacian^3$ is a $2n \times 2n$ matrix with two blocks on the diagonal, a
first $n \times n$ block $\laplacian_{\Gg_{\{3,1\}}}$ constructed
on $\Gg_{\{3,1\}} = \Gg_1 \cup \Gg_2 \cup \Gg_3$, and a second $n \times n$ block
$\laplacian_{\Gg_{\{3,2\}}}$ constructed on $\Gg_{\{3,2\}} = \Gg_4$. Similarly, we  
combine Def.\,\ref{def:eps-sparsifier} and Prop.\,\ref{prop:ordering-to-norm-bound} to define
\begin{align*}
\:P_{\{h,\ell\}} &\triangleq (\laplacian_{\Gg_{\{h,\ell\}}} + \gamma\:I)^{-1/2}\laplacian_{\Gg_{\{h,\ell\}}}(\laplacian_{\Gg_{\{h,\ell\}}} + \gamma\:I)^{-1/2} \quad \text{and} \\
\wt{\:P}_{\{h,\ell\}} &\triangleq (\laplacian_{\Gg_{\{h,\ell\}}} + \gamma\:I)^{-1/2}\laplacian_{\Hg_{\{h,\ell\}}}(\laplacian_{\Gg_{\{h,\ell\}}} + \gamma\:I)^{-1/2},
\end{align*}
and have $\:P^h$ as a block diagonal projection matrix,
where each block $\:P_{\{h,\ell\}}$ is defined using $\laplacian_{\Gg_{\{h,\ell\}}}$,
and block diagonal $\wt{\:P}^h$,
where each block $\wt{\:P}_{\{h,\ell\}}$ is defined using $\laplacian_{\Hg_{\{h,\ell\}}}$
and $\coldict_{\{h,\ell\}}$.

\textbf{The statement.}
Since $\:P^h - \wt{\:P}^h$ is block diagonal, we have that a bound on its largest eigenvalue implies an equal bound on each matrix on the diagonal, i.e.,
\begin{align*}
\normsmall{\:P^h - \wt{\:P}^h} = \max_{\ell} \normsmall{\:P_{\{h,\ell\}} - \wt{\:P}_{\{h,\ell\}}} \leq \varepsilon
\implies \normsmall{\:P_{\{h,\ell\}} - \wt{\:P}_{\{h,\ell\}}} \leq \varepsilon
\end{align*}
for all blocks $\ell$ on the diagonal, and since each block corresponds to
a sparsifier $\coldict_{\{h,\ell\}}$, this means that if $\normsmall{\:P^h - \wt{\:P}^h} \leq \varepsilon$,
all sparsifiers at layer $\ell$ are $(\varepsilon,\gamma)$-sparsifiers
of their respective graphs.
Let $\deff^{\{h,\ell\}}(\gamma)$ be the effective dimension of $\laplacian_{\Gg_{\{h,\ell\}}}$.
Our goal is to show
\begin{align}\label{eq:distri-theorem-goal}
    &\probability\bigg(\exists  h\in\{1,\ldots,k\}: \normsmall{\:P^h - \wt{\:P}^h}_2 \geq \varepsilon \;\cup\; \max_{\ell=1,\dots,|\dictpool_h|}|\coldict_{\{h,\ell\}}| \geq 3\wb{q}\deff^{\{h,\ell\}}(\gamma)\bigg)\nonumber\\
     &= \probability\bigg(\exists  h\in\{1,\ldots,k\} : \underbrace{\left(\max_{\ell=1,\dots,|\dictpool_h|}\normsmall{\:P_{\{h,\ell\}} - \wt{\:P}_{\{h,\ell\}}}_2\right) \geq \varepsilon}_{A_h} \;\cup\; \underbrace{\left(\max_{\ell=1,\dots,|\dictpool_h|}|\coldict_{\{h,\ell\}}| \geq 3\wb{q}\deff^{\{h,\ell\}}(\gamma)\right)}_{B_h}\bigg) \leq \delta,
\end{align}
where event $A_h$ refers to the case when some sparsifier $\coldict_{\{h,\ell\}}$ at an intermediate layer $h$ fails to accurately approximate $\laplacian_{\{h,\ell\}}$ and event $B_h$ considers the case when the memory requirement is not met (i.e., too many edges are kept in one of the sparsifiers $\coldict_{\{h,\ell\}}$ at a certain layer $h$). After reformulating and a union bound we obtain
\begin{align}\label{eq:distri-theorem-goal}
    &\probability\bigg(\exists  h\in\{1,\ldots,k\}: \normsmall{\:P^h - \wt{\:P}^h}_2 \geq \varepsilon \;\cup\; \max_{\ell=1,\dots,|\dictpool_h|}|\coldict_{\{h,\ell\}}| \geq 3\wb{q}\deff^{\{h,\ell\}}(\gamma)\bigg)\nonumber\\
     &\leq \sum_{h = 1}^{k}\sum_{\ell=1}^{|\dictpool_h|} \probability\left(\normsmall{\:P_{\{h,\ell\}} - \wt{\:P}_{\{h,\ell\}}}_2 \geq \varepsilon \right)\nonumber\\
     &\quad\quad+\sum_{h = 1}^{k}\sum_{\ell=1}^{|\dictpool_h|}\probability\left(|\coldict_{\{h,\ell\}}| \geq 3\wb{q}\deff^{\{h,\ell\}}(\gamma) \cap \left\{\forall  h' \in \{1, \dots, h\} :  \normsmall{\:P^{h'} - \wt{\:P}^{h'}}_2 \leq \varepsilon\right\}\right) \leq \delta.
\end{align}
The accuracy of the sparsifier (first term in the previous bound) is guaranteed by the fact that given an $(\varepsilon,\gamma)$-accurate sparsifier we obtain $\gamma$-effective resistance estimates (i.e.\@ RLS estimates) which are at least a fraction of the true ones, thus forcing the algorithm to sample each column \textit{enough}. On the other hand, the space complexity bound is achieved by exploiting the fact that estimates are always upper-bounded by the true $\gamma$-effective resistance, thus ensuring that Alg.\,\ref{alg:kl_dist} does not oversample columns w.r.t.\,the sampling process following the exact $\gamma$-effective resistance.

In the reminder of the proof, we will show that both events happen with probability
smaller than $\delta/(2k^2)$. Since $|\dictpool_h| = k - h + 1$,
we have
\begin{align*}
\sum_{h = 1}^{k}\sum_{\ell=1}^{|\dictpool_h|}\frac{\delta}{2k^2} = \sum_{h = 1}^{k}(k - h + 1)\frac{\delta}{2k^2} = k(k+1)\frac{\delta}{4k^2} \leq k^2\frac{\delta}{2k^2} = \frac{\delta}{2},
\end{align*}
and the union bound over all events is smaller than $\delta$.
The main advantage of splitting the failure probability as we did in Eq.\,\ref{eq:distri-theorem-goal}
is that we can now analyze the processes that generated each $\:P_{\{h,\ell\}} - \wt{\:P}_{\{h,\ell\}}$
(and each sparsifier $\coldict_{\{h,\ell\}}$) separately.
Focusing on a single node $\{h,\ell\}$ restricts our
problem on a well defined graph $\Gg_{\{h,\ell\}}$,
where we can analyze the evolution of $\coldict_{\{h,\ell\}}$ sequentially.

\subsection{Bounding the projection error $\normsmall{\mathbf{P}_{\{h,\ell\}} - \wt{\mathbf{P}}_{\{h,\ell\}}}$}\label{ss:setting.stage}


\textbf{The sequential process.} Thanks to the union bound in Eq.\,\ref{eq:distri-theorem-goal},
instead of having to consider the whole merge tree followed by
Alg.\,\ref{alg:kl_dist}, we can focus on each individual node
$\{h,\ell\}$ and study the sequential process that generated its sparsifier $\coldict_{\{h,\ell\}}$.
We will now map more clearly the actions
taken by Alg.\,\ref{alg:kl_dist} to the process that
generated $\:P_{\{h,\ell\}} - \wt{\:P}_{\{h,\ell\}}$.
We begin by focusing on $\wt{\:P}_{\{h,\ell\}}$, which is a random matrix defined starting from
the fixed graph Laplacian $\laplacian_{\Gg_{\{h,\ell\}}}$ and the random sparsifier Laplacian $\laplacian_{\Hg_{\{h,\ell\}}}$,
where the randomness influences both which edges are included in $\coldict_{{\{h,\ell\}}}$,
and the weight with which they are added.
\todod{fixed vs random tree}

Note that since the merge tree is decided in advance, the graph $\Gg_{\{h,\ell\}}$ is
not a random object and is fixed for the whole process.
Consider now an edge
$e \in \Gg_{\{h,\ell\}}$. Again for simplicity and without loss of generality\footnote{Alternatively, we can assign an index to each of the edges in the leaf graphs, requiring at most $km \leq kn^2$ indices.},
we will assume that the starting graphs in the leaves are edge-disjoint.
Therefore, there is a single path in the tree, with length $h$, from the leaves to $\{h,\ell\}$.
This means that for all $s<h$, we can properly define a unique $\wt{p}_{s,e}$
and $q_{s,e}$ associated with that point. More in detail, if at layer $s$ point $i$ is present
in $\Gg_{\{s,\ell'\}}$, it means that either (1) Alg.\,\ref{alg:kl_dist}
used $\coldict_{\{s,\ell'\}}$ to compute $\wt{p}_{s,e}$, and $\wt{p}_{s,e}$ to compute $q_{s,e}$,
or (2) at layer $h$, Alg.\,\ref{alg:kl_dist} did not have any merge scheduled
for point $i$, and we simply propagate $\wt{p}_{s,e} = \wt{p}_{s-1,i}$ and $q_{s,e} = q_{s-1,i}$.
Consistently with the algorithm, we initialize
$\wt{p}_{0,i} = 1$ and $q_{0,i}=\wb{q}$.

Denote $\nhl_{\{h,\ell\}} = |\Gg_{\{h,\ell\}}|$ so that we can use index $i \in [\nhl_{\{h,\ell\}}]$ to index all edges in
$\Gg_{\{h,\ell\}}$. Given the $n \times \nhl_{\{h,\ell\}}$ matrix $\concmat = (\laplacian_{\Gg_{\{h,\ell\}}} + \gamma\:I)^{-1/2}\:B_{\Gg_{\{h,\ell\}}}$ with its $e$-th column $\concvec_{i} = (\laplacian_{\Gg_{\{h,\ell\}}} + \gamma\:I)^{-1/2}\:B_{\Gg_{\{h,\ell\}}} \:e_{\nhl_{\{h,\ell\}},e}$, we can rewrite the projection matrix as that
$\:P_{\{h,\ell\}} = \concmat\concmat^\transp = \sum_{e=1}^{\nhl_{\{h,\ell\}}} \concvec_e\concvec_e^\transp$.
Note that
\begin{align*}
\normsmall{\concvec_e\concvec_e^\transp}
= \concvec_e^\transp\concvec_e
= \:e_{\nhl_{\{h,\ell\}},e}^\transp\concmat^\transp\concmat\:e_{\nhl_{\{h,\ell\}},e}
= \:e_{\nhl_{\{h,\ell\}},e}^\transp\concmat\concmat^\transp\:e_{\nhl_{\{h,\ell\}},e}
= \:e_{\nhl_{\{h,\ell\}},e}^\transp\:P_{\{h,\ell\}}\:e_{\nhl_{\{h,\ell\}},e} = r_e(\gamma),
    \end{align*}
or, in other words, the norm $\normsmall{\concvec_e\concvec_e^\transp}$ is equal to the $\gamma$-effective resistance of
the $e$-th edge w.r.t.\@ to graph $\Gg_{\{h,\ell\}}$. Note that since $e$ is present only in node $l$ on layer $h$, its
$\gamma$-effective resistance is uniquely defined w.r.t.\@ $\Gg_{\{h,\ell\}}$ and can be shortened as
$r_{h,e}$.
Using $\concvec_e$, we can also introduce the random
matrix $\wt{\:P}^{\{h,\ell\}}_{s}$ as
\begin{align*}
\wt{\:P}^{\{h,\ell\}}_{s} = \sum_{e=1}^{\nhl_{\{h,\ell\}}} \frac{q_{s,e}}{\wb{q}\wt{p}_{s,e}}\concvec_e\concvec_e^\transp
= \sum_{e=1}^{\nhl_{\{h,\ell\}}} \sum_{j=1}^{\wb{q}}\frac{z_{s,e,j}}{\wb{q}\wt{p}_{s,e}}\concvec_e\concvec_e^\transp.
\end{align*}
where $z_{s,e,j}$ are $\{0,1\}$ r.v.\,such that $q_{s,e} = \sum_{j=1}^{\wb{q}}z_{s,e,j}$,
or in other words $z_{s,e,j}$ are the Bernoulli random variables
that compose the Binomial $q_{s,e}$ associated with edge $e$, with $j$ indexing each individual copy
of the edge.
Note that when $s = h$, we have that $\wt{\:P}^{\{h,\ell\}}_h = \wt{\:P}_{\{h,\ell\}}$ and we recover
the definition of the approximate projection matrix from Alg.\,\ref{alg:kl_dist}.
But, for
a general $s \neq h$ $\wt{\:P}^{\{h,\ell\}}_s$ does not have a direct interpretation in the context
of Alg.\,\ref{alg:kl_dist}. It combines the vectors $\concvec_e$,
which are defined using $\laplacian_{\Gg_{\{h,\ell\}}}$ at layer $h$, with the weights
$\wt{p}_{s,e}$ computed by Alg.\,\ref{alg:kl_dist} across multiple nodes at layer $s$, which are potentially stored in
different machines that cannot communicate. Nonetheless, $\wt{\:P}^{\{h,\ell\}}_s$ is a useful tool
to analyze Alg.\,\ref{alg:kl_dist}.

\todod{fixed vs random merge tree}
Taking into account that we are now considering a specific node $\{h,\ell\}$,
we can drop the index from the
graphs $\Gg_{\{h,\ell\}} = \Gg$, $\gamma$-effective resistances $\tau_{h,e}$,
and size $\nhl_{\{h,\ell\}} = \nhl$.
Using this shorter notation, we can reformulate our objective as bounding $\normsmall{\:P_{\{h,\ell\}} - \wt{\:P}_{\{h,\ell\}}}_2 =\normsmall{\:P_{\{h,\ell\}} - \wt{\:P}^{\{h,\ell\}}_{h}}_2$, and
reformulate the process as a sequence of matrices $\{\:Y_s\}_{s=1}^h$ defined as
\begin{align*}
\:Y_{s} = \:P_{\{h,\ell\}} - \wt{\:P}^{\{h,\ell\}}_s = \frac{1}{\wb{q}}\sum_{e=1}^{\nhl} \sum_{j=1}^{\wb{q}}\left(1 - \frac{z_{s,e,j}}{\wt{p}_{s,e}}\right)\concvec_e\concvec_e^\transp,
\end{align*}
where $\:Y_{h} = \:P_{\{h,\ell\}} - \wt{\:P}^{\{h,\ell\}}_{h} = \:P_{\{h,\ell\}} - \wt{\:P}_{\{h,\ell\}}$,
and $\:Y_{1} = \:P_{\{h,\ell\}} - \wt{\:P}^{\{h,\ell\}}_{0} = \:0$ since $\wt{p}_{0,i} = 1$ and $q_{0,i} = \wb{q}$.

\subsection{Bounding $\mathbf{Y}_h$}\label{ssec:bounding_y}
We transformed the problem of bounding $\normsmall{\:P_{\{h,\ell\}} - \wt{\:P}_{\{h,\ell\}}}$
into the problem of bounding $\:Y_h$, which we modeled as a random matrix process,
connected to Alg.\,\ref{alg:kl_dist} by the fact that both algorithm
and random process $\:Y_h$ make use of the same weight $\wt{p}_{s,e}$
and multiplicities $q_{s,e}$.

\textbf{The frozen process.}
Inspired by \citet{cohen_online_2016}, we will now replace the process $\:Y_s$ with an alternative process $\wb{\:Y}_{s}$
defined as
\begin{align*}
    \wb{\:Y}_s = 
    \:Y_{s-1} \indfunc\left\{\normsmall{\wb{\:Y}_{s-1}} \leq \varepsilon\right\}
    +\wb{\:Y}_{s-1}\indfunc\left\{\normsmall{\wb{\:Y}_{s-1}} \geq \varepsilon\right\}.
\end{align*}
This process starts from $\wb{\:Y}_0 = \:Y_0 = \:0$, and
is identical to $\:Y_s$ until a step $\wb{s}$ where for the first
time $\normsmall{\:Y_{\wb{s}}} \leq \varepsilon$ and
$\normsmall{\:Y_{\wb{s}+1}} \geq \varepsilon$. After this failure happen
the process $\wb{\:Y}_s$ is ``frozen'' at $\wb{s}$ and $\wb{\:Y}_{s} = \:Y_{\wb{s}+1}$ for all $\wb{s} +1\leq s \leq h$.
Consequently, if any of the intermediate elements
of the sequence violates the condition $\normsmall{\:Y_s}~\leq~\varepsilon$,
the last element will violate it too. For the rest, $\wb{\:Y}_s$ behaves
exactly like $\:Y_s$.
Therefore,
\begin{align*}
&\probability\left( \normsmall{\:Y_h} \geq \varepsilon\right)
\leq \probability\Big( \normsmall{\wb{\:Y}_h} \geq \varepsilon\Big),
\end{align*}
and if we can bound $\probability\Big( \normsmall{\wb{\:Y}_h} \geq \varepsilon\Big)$
we will have a bound for the failure probability of Alg.\,\ref{alg:kl_dist},
even though after ``freezing'' the process $\wb{\:Y}_h$ does not make the same choices as
the algorithm.

We will see now how to construct the process $\wb{\:Y}_s$ starting from
$z_{s,e,j}$ and $\wt{p}_{s,e,j}$.
We recursively define the indicator ($\{0,1\}$) random variable~$\wb{z}_{s,e,j}$ as
\begin{align*}
    \wb{z}_{s,e,j} = \indfunc\left\{ u_{s,e,j} \leq \frac{\wb{p}_{s,e,j}}{\wb{p}_{s-1,e,j}}\right\} \wb{z}_{s-1,e,j},
\end{align*}
where $u_{s,e,j} \sim \mathcal{U}(0,1)$ is a $[0,1]$ uniform random variable and $\wb{p}_{s,e,j}$ is defined as
\begin{align*}
\wb{p}_{s,e,j} =\wt{p}_{s,e}
    \indfunc\left\{\normsmall{\wb{\:Y}_{s-1}} \leq \varepsilon \cap z_{s-1,e,j} = 1\right\}
    +\wb{p}_{s-1,e,j}\indfunc\left\{\normsmall{\wb{\:Y}_{s-1}} \geq \varepsilon \cup z_{s-1,e,j} = 0\right\}\!.
\end{align*}

This definition of the process satisfies the freezing condition, since if
$\normsmall{\:Y_{\wb{s}+1}} \geq \varepsilon$ (we have a failure at step
$\wb{s}$), for all $s' \geq \wb{s}+1$ we have
$\wb{z}_{s',i,j} = \wb{z}_{\wb{s}+1,i,j}$ with probability 1
($\wb{p}_{\wb{s}+1,i,j}/\wb{p}_{\wb{s},i,j} = \wb{p}_{\wb{s},i,j}/\wb{p}_{\wb{s},i,j} = 1$),
and the weights $1/(\wb{q}\wb{p}_{\wb{s}+1,i,j}) = 1/(\wb{q}\wb{p}_{\wb{s},i,j})$
never change.

Introducing a per-copy weight $\wb{p}_{s,e,j}$ and enforcing that $\wb{p}_{s+1,i,j} = \wb{p}_{s,e,j}$
when $z_{s,e,j} = 0$ avoids subtle inconsistencies in the formulation.
In particular, not doing so would semantically correspond to reweighting
dropped copies. Although this does not directly affect $\:Y_s$ (since the ratio $z_{s,e,j}/\wt{p}_{s,e}$
is zero for dropped copies), and therefore
the relationship $\probability\left( \normsmall{\:Y_h} \geq \varepsilon\right)
\leq \probability\Big( \normsmall{\wb{\:Y}_h} \geq \varepsilon\Big)$
still holds, we will see later how maintaining consistency helps us bound the second moment of our process.

We can now arrange the indices $s$, $e$, and $j$ into a linear index $\lidx=s$ in the
range $[1,\dots,\nhl^2\wb{q}]$, obtained as $\lidx~=~\{s,e,j\}~=~(s-1)\nhl\wb{q}~+~(e-1)\wb{q}~+~j$.
We also define the difference matrix as 
\begin{align*}
\wb{\:X}_{\{s,e,j\}} = \frac{1}{\wb{q}}\left(\frac{z_{s-1,e,j}}{\wb{p}_{s-1,e,j}} - \frac{z_{s,e,j}}{\wb{p}_{s,e,j}}\right)\concvec_e\concvec_e^\transp,
    \end{align*}
which allows  writing the cumulative matrix as
$\wb{\:Y}_{\{s,e,j\}}~=~\sum_{r=1}^{\{s,e,j\}}~\wb{\:X}_{\{s,e,j\}}$
where the checkedges $\{s,\nhl,\wb{q}\}$ correspond to $\wb{\:Y}_s$,
\begin{align*}
\wb{\:Y}_{\{s,\nhl,\wb{q}\}} = \wb{\:Y}_s = 
    \frac{1}{\wb{q}}\sum_{e=1}^{\nhl}\sum_{j=1}^{\wb{q}}\left(1-\frac{z_{s,e,j}}{\wb{p}_{s,e,j}}\right)\concvec_e\concvec_e^\transp.
\end{align*}
Let $\F_s$ be the filtration containing all the
realizations of the uniform random variables $u_{s,e,j}$ up to the step~$s$, that is $\F_s = \{ u_{s',e',j'}, \forall\{s',e',j'\}
\leq s\}$. Again, we notice that $\F_s$ defines the state of
the algorithm after completing iteration $s$ because, unless a ``freezing'' happened,
Alg.\,\ref{alg:kl_dist} and $\wb{\:Y}_s$ flip coins with the same probability,
and generate the same sparsifiers.
Since $\atau_{s,e}$ and
$\wb{p}_{s,e,j}$ are computed at the beginning of iteration $s$ using the
sparsifier $\coldict_{\{s,\ell'\}}$ (for some $\ell'$ unique at layer $s$), they are fully determined by $\F_{s-1}$.
Furthermore, since $\F_{s-1}$ also defines the values of all indicator
variables~$\wb{z}_{s',e,j}$ up to $\wb{z}_{s-1,e,j}$ for any $i$ and $j$, we have that
all the Bernoulli variables $\wb{z}_{s,e,j}$ at iteration $s$ are conditionally
independent given $\F_{s-1}$. In other words, we have that for any
$e'$, and $j'$ such that $\{s,1,1\} \leq \{s,e',j'\} <s$ the following
random variables are equal in distribution,
\begin{align}\label{eq:distro.z}
\wb{z}_{s,e,j} \big| \F_{\{s,e',j'\}} = \wb{z}_{s,e,j} \big| \F_{\{s-1,\nhl,\wb{q}\}} \sim \mathcal{B}\Big( \frac{\wb{p}_{s,e,j}}{\wb{p}_{s-1,e,j}} \Big)\CommaBin
\end{align}
and for any $e'$, and $j'$ such that $\{s,1,1\} \leq \{s,e',j'\} \leq \{s,\nhl,\wb{q}\} $ and $s \neq \{s,e',j'\}$ we have the independence
\begin{align}\label{eq:distro.z.indep}
\wb{z}_{s,e,j} \big| \F_{\{s-1,\nhl,\wb{q}\}} \perp \wb{z}_{s,e',j'} \big| \F_{\{s-1,\nhl,\wb{q}\}}.
\end{align}
While knowing that $\normsmall{\:Y_s} \leq \varepsilon$ is not sufficient to provide guarantees for the approximate probabilities
$\wt{p}_{s,e}$, we can show that it is enough to prove that the frozen probabilities $\wb{p}_{s,e,j}$ are
never too small.
\begin{lemma}\label{lem:wbp-always-good}
    Let $\alpha \eqdef (1+3\varepsilon)/(1-\epsilon)$ and $\wb{p}_{s,e,j}$ be the sequence of probabilities generated
    by the freezing process. Then for any $s,e,$ and $j$, we
    have $\wb{p}_{s,e,j} \geq p_{h,e}/\alpha = r_{h,e}/\alpha$.
\end{lemma}
\begin{proof}
Let $\wb{s}$ be the step where the process freezes ($\wb{s} = h$ if it does not
freeze), or, in other words, $\normsmall{\:Y_{\wb{s}}} < \varepsilon$
and $\normsmall{\:Y_{\wb{s}+1}} \geq \varepsilon$.
From the definition of $\wb{p}_{s,e,j}$, we have
that 
\begin{align*}
\wb{p}_{s,e,j} \geq \wb{p}_{\wb{s},i} = \wt{p}_{\wb{s},e}
&= \max\left\{\min\left\{\atau_{\wb{s},e},\; \wt{p}_{\wb{s}-1,e} \right\},\; \wt{p}_{\wb{s}-1,e}/2\right\}\\
&\geq \min\left\{\atau_{\wb{s},e},\; \wt{p}_{\wb{s}-1,e} \right\}
=\min\left\{\atau_{\wb{s},e},\; \wt{p}_{\wb{s}-2,e}\right\}
=\min\left\{\atau_{\wb{s},e},\; \wt{p}_{\wb{s}-3,e}\right\} \ldots
=\min\left\{\atau_{\wb{s},e},\; \wt{p}_{0,e}\right\}
=\atau_{\wb{s},e},
\end{align*}
and therefore $\wb{p}_{s,e,j} \geq \atau_{\wb{s},e}$.
Now let $\{\wb{s},\ell'\}$ be the node where $\wt{r}_{\wb{s},e}$ was computed.
From Alg.\,\ref{alg:kl_dist} we know it is computed using the sparsifier $\wb{\Hg}$
generated by the union of the two children of node $\{\wb{s},\ell'\}$,
which is an $(\varepsilon,2\gamma)$-sparsifier of $\Gg_{\{\wb{s},\ell'\}}$.

From its definition,
we know that $\wt{r}_{\wb{s},e}$ is computed by Alg.\,\ref{alg:kl_dist} as
\begin{align*}
\wt{r}_{s,e} = (1-\varepsilon)\:\phi_e^\transp\left(\laplacian_{\wb{\Hg}} + (1+\varepsilon)\gamma \:I\right)^{-1}\:\phi_{i},
\end{align*}
using only the edges in $\wb{\Hg}$ that are available at node $\{\wb{s},\ell'\}$.
Let once again $\Gg = \Gg_{\{h,\ell\}}$ by dropping the subscript,
and define $\wb{\Hg}^c$ as the complement set of all edges in other sparsifiers at
level $\wb{s}$ not in $\wb{\Hg}$.
From the definition of $\wt{\:P}^{\{h,\ell\}}_{\wb{s}}$ and Prop.\,\ref{prop:ordering-to-norm-bound}
we know that 
\begin{align*}
\normsmall{\:Y_{\wb{s}}}
=\norm{\:P_{\{h,\ell\}} - \wt{\:P}^{\{h,\ell\}}_{\wb{s}}}{2}
&= \norm{(\laplacian_{\Gg} + 2\gamma\:I)^{-1/2}(\laplacian_{\Gg} - (\laplacian_{\wb{\Hg}} + \laplacian_{\wb{\Hg}^c}))(\laplacian_{\Gg} + 2\gamma\:I)^{-1/2}}{2} \leq \varepsilon
\end{align*}
and we know that this implies
\begin{align*}
    \laplacian_{\wb{\Hg}} \preceq \laplacian_{\Gg} + \varepsilon(\laplacian_{\Gg} + 2\gamma\:I) - \laplacian_{\wb{\Hg}^c} \preceq \laplacian_{\Gg} + \varepsilon(\laplacian_{\Gg} + 2\gamma\:I). 
\end{align*}
Plugging it in the initial definition,
\begin{align*}
\atau_{s,e} &= (1-\varepsilon)\:b_e^\transp\left(\laplacian_{\wb{\Hg}} + (1+\varepsilon)\gamma \:I\right)^{-1}\:b_{i}\\
&\geq (1 - \vareps)\:b_e^\transp(\laplacian_{\Gg} + \varepsilon(\laplacian_{\Gg} + 2\gamma\:I) + (1+\varepsilon)\gamma \:I)^{-1}\:b_e\\
&= (1 - \vareps)\frac{1}{1+3\varepsilon}\:b_e^\transp(\featkermatrix_{\Gg}\featkermatrix_{\Gg}^\transp + \gamma \:I)^{-1}\:b_e
\geq \frac{1-\varepsilon}{1+3\varepsilon}r_{h,e}
\geq \frac{r_{h,e}}{\alpha}\cdot
\end{align*}
\end{proof}

We now proceed by studying the process
$\{\wb{\:Y}_s\}_{s=1}^h$ and showing that it is a bounded martingale.
In order to show that $\wb{\:Y}_s$ is a martingale, it is sufficient to verify the following (equivalent) conditions,
\begin{align*}
    \expectedvalue\left[\wb{\:Y}_s \condbar \F_{s-1}\right]
    =\wb{\:Y}_{s-1} \enspace \Leftrightarrow \enspace
    \expectedvalue\left[\wb{\:X}_{\{s,e,j\}} \condbar \F_{s-1}\right] = \:0.
\end{align*}
We begin by inspecting the conditional random variable $\wb{\:X}_{\{s,e,j\}} | \F_{s-1}$. Given the definition of $\wb{\:X}_{\{s,e,j\}}$, the conditioning on $\F_{s-1}$ determines the values of $\wb{z}_{s-1,e,j}$ and the approximate probabilities $\wb{p}_{s-1,e,j}$ and $\wb{p}_{s,e,j}$. In fact, remember that these quantities are fully determined by the realizations in $\F_{s-1}$ which are contained in $\F_{s-1}$. As a result, the only stochastic quantity in $\wb{\:X}_{\{s,e,j\}}$ is the variable $\wb{z}_{s,e,j}$. 
Specifically, if $\normsmall{\wb{\:Y}_{s-1}} \geq \varepsilon$,
then we have $\wb{p}_{s,e,j} = \wb{p}_{s-1,e,j}$ and $\wb{z}_{s,e,j} = \wb{z}_{s-1,e,j}$
(the process is stopped), and the martingale requirement
$ \expectedvalue\left[\wb{\:X}_{\{s,e,j\}} \condbar \F_{s-1}\right] = \:0$
is trivially satisfied.
On the other hand, if $\normsmall{\wb{\:Y}_{s-1}} \leq \varepsilon$ we have
\begin{align*}
\expectedvalue_{u_{s,e,j}}\left[\frac{1}{\wb{q}} \left(\frac{\wb{z}_{s-1,e,j}}{\wb{p}_{s-1,e,j}} - \frac{\wb{z}_{s,e,j}}{\wb{p}_{s,e,j}}\right)\concvec_e\concvec_e^\transp \condbar \F_{s-1}\right]
 &= \frac{1}{\wb{q}} \left( \frac{\wb{z}_{s-1,e,j}}{\wb{p}_{s-1,e,j}} -\frac{\wb{z}_{s-1,e,j}}{\wb{p}_{s,e,j}}\expectedvalue\left[\indfunc\left\{ u_{s,e,j} \leq \frac{\wb{p}_{s,e,j}}{\wb{p}_{s-1,e,j}}\right\}\condbar\F_{s-1}\right]\right)\concvec_e\concvec_e^\transp\\
&= \frac{1}{\wb{q}}
    \left(\frac{\wb{z}_{s-1,e,j}}{\wb{p}_{s-1,e,j}} - \frac{\wb{z}_{s-1,e,j}}{\wb{p}_{s,e,j}}\frac{\wb{p}_{s,e,j}}{\wb{p}_{s-1,e,j}}
\right)\concvec_e\concvec_e^\transp
= \:0,
\end{align*}
where we use the recursive definition of $\wb{z}_{s,e,j}$ and the fact that $u_{s,e,j}$ is a uniform random variable in $[0,1]$. This proves that $\wb{\:Y}_s$ is indeed a martingale.
We now compute an upper bound $R$ on the norm of the values of the difference process as
\begin{align*}
&\normsmall{\wb{\:X}_{\{s,e,j\}}}
= \frac{1}{\wb{q}} \left|\left(\frac{\wb{z}_{s-1,e,j}}{\wb{p}_{s-1,e,j}} - \frac{\wb{z}_{s,e,j}}{\wb{p}_{s,e,j}}\right)\right|\normsmall{\concvec_e\concvec_e^\transp}
\leq \frac{1}{\wb{q}} \frac{1}{\wb{p}_{s,e,j}} \normsmall{\concvec_e\concvec_e^\transp}
= \frac{1}{\wb{q}} \frac{1}{\wb{p}_{s,e,j}} \tau_{h,i}
\leq \frac{1}{\wb{q}} \frac{\alpha}{\tau_{h,i}}\tau_{h,i}
=\frac{\alpha}{\wb{q}} \triangleq R,
\end{align*}
where we used Lem.\,\ref{lem:wbp-always-good} to bound
$\wb{p}_{s,e,j} \leq r_{h,e}/\alpha$.
If instead, $\normsmall{\wb{\:Y}_{s-1}} \geq \varepsilon$,
 the process is stopped and 
$\normsmall{\wb{\:X}_s} = \normsmall{\:0} = 0 \leq R$.

We are now ready to use a Freedman matrix inequality from \citet{tropp2011freedman} to bound the norm of $\wb{\:Y}$.

\begin{proposition}[\citealp{tropp2011freedman},~Theorem~1.2]\label{prop:matrix-freedman}
Consider a matrix martingale $\{ \:Y_k : k = 0, 1, 2, \dots \}$ whose values are self-adjoint matrices with dimension $d$, and let $\{ \:X_k : k = 1, 2, 3, \dots \}$ be the difference sequence.  Assume that the difference sequence is uniformly bounded in the sense that
\begin{align*}
 \normsmall{\:X_k}_2  \leq R
\quad\text{almost surely}
\quad\text{for $k = 1, 2, 3, \dots$}.
\end{align*}
Define the predictable quadratic variation process of the martingale as
\begin{align*}
\:{W}_k \eqdef \sum_{j=1}^k \expectedvalue \left[ \:X_j^2 \condbar \{\:X_{s}\}_{s=0}^{j-1} \right]\!,
\quad\text{for $k = 1, 2, 3, \dots$}.
\end{align*}
Then, for all $\varepsilon \geq 0$ and $\sigma^2 > 0$,
\begin{align*}
\probability\left( \exists k \geq 0 : \normsmall{\:Y_k}_2 \geq \varepsilon \ \cap\ 
        \normsmall{ \:W_{k} } \leq \sigma^2 \right)
	\leq 2d \cdot \exp \left\{ - \frac{ \varepsilon^2/2 }{\sigma^2 + R\varepsilon/3} \right\}\!\cdot
\end{align*}
\end{proposition}

In order to use the previous inequality, we develop the probability of error for any fixed $h$ as
\begin{align*}
\probability\left( \normsmall{\:Y_h} \geq \varepsilon\right)
\leq \probability\left( \normsmall{\wb{\:Y}_h} \geq \varepsilon\right)
&= \probability\left( \normsmall{\wb{\:Y}_h} \geq \varepsilon \cap \normsmall{\:W_h} \leq \sigma^2\right)
+ \probability\left( \normsmall{\wb{\:Y}_h} \geq \varepsilon \cap \normsmall{\:W_h} \geq \sigma^2\right)\\
&\leq \underbrace{\probability\left( \normsmall{\wb{\:Y}_h} \geq \varepsilon \cap \normsmall{\:W_h} \leq \sigma^2\right)}_{\mbox{(a)}}
    + \underbrace{\probability\left( \normsmall{\:W_h} \geq \sigma^2\right)}_{\mbox{(b)}}\!.
\end{align*}
Using the bound on  $\normsmall{\wb{\:X}_{\{s,e,j\}}}_2$, we can directly apply Prop.\,\ref{prop:matrix-freedman} to bound $\mbox{(a)}$ for any fixed $\sigma^2$.
To bound the part $\mbox{(b)}$, we use the following lemma, proved later in Sec.\,\ref{ssec:bounding_w}.

\begin{lemma}[Low probability of the large norm of the predictable quadratic variation process]\label{lem:prob-dominance} We have that
    \begin{align*}
        \probability\left( \normsmall{\:W_h} \geq \frac{6\alpha}{\wb{q}}\right)
        \leq n \cdot \exp \left\{ - \frac{2\wb{q}}{\alpha} \right\}\!\cdot
    \end{align*}
\end{lemma}

Since $\:P_{\{h,\ell\}}$ is defined at most on $n$ nodes, combining Prop.\,\ref{prop:matrix-freedman} with $\sigma^2 = 6\alpha/\wb{q}$, Lem~\ref{lem:prob-dominance}, the fact that $2\varepsilon/3 \leq 1$ and the value used by Alg.\,\ref{alg:kl_dist}, $\wb{q} = 39\alpha\log(2n/\delta)/\varepsilon^2$ we obtain
\begin{align*}
\probability\left( \normsmall{\:P_{\{h,\ell\}} - \wt{\:P}_{\{h,\ell\}}}_2 \geq \varepsilon\right)
& = \probability\left( \normsmall{\:Y_h} \geq \varepsilon\right)
\leq \probability\left( \normsmall{\wb{\:Y}_h} \geq \varepsilon \cap \normsmall{\:W_h} \leq \sigma^2\right)
+ \probability\left( \normsmall{\:W_h} \geq \sigma^2\right)\\
&\leq 2n \cdot \exp\left\{-\frac{\varepsilon^2\wb{q}}{\alpha}\left(\frac{1}{12+2\varepsilon/3}\right)\right\}
+ n \cdot \exp\left\{- \frac{2\wb{q}}{\alpha}\right\}\\
&\leq 3n \cdot \exp \left\{ - \frac{\varepsilon^2}{13\alpha}\wb{q}\right\}
= 3n \cdot \exp \left\{ - 3\log\left(\frac{2n}{\delta}\right)\right\}\\
&= 3n \cdot \exp \left\{ - \log\left(\left(\frac{2n}{\delta}\right)^3\right)\right\}
=  \frac{3n\delta^3}{8n^3} \leq \frac{\delta}{2n^2}\cdot
\end{align*}
This, combined with the fact that $k \leq n^2/n \leq n$ since at most we can split our
graph into $n$ parts each containing $n$ edges, concludes this part of the proof.

\subsection{Proof of Lem.\,\ref{lem:prob-dominance} (bound on predictable quadratic variation)}\label{ssec:bounding_w}
\newcounter{cnt-lem-quad-variation}
\setcounter{cnt-lem-quad-variation}{1}


\textbf{Step \arabic{cnt-lem-quad-variation}\stepcounter{cnt-lem-quad-variation} (a preliminary bound).}
We start by writing out $\:W_{\lidx}$ for the process $\wb{\:Y}_s$,
\begin{align*}
\:W_{\lidx} =\frac{1}{\wb{q}^2}\sum_{\{s,e,j\}\leq \lidx} \expectedvalue \left[\left(\frac{\wb{z}_{s-1,e,j}}{\wb{p}_{s-1,e,j}} - \frac{\wb{z}_{s,e,j}}{\wb{p}_{s,e,j}}\right)^{2} \condbar \F_{\{s,e,j\}-1} \right]\concvec_e\concvec_e^\transp\concvec_e\concvec_e^\transp.
\end{align*}
We rewrite the expectation terms in the equation above as
\begin{align*}
\expectedvalue &\left[ 
\left(\frac{\wb{z}_{s-1,e,j}}{\wb{p}_{s-1,e,j}} - \frac{\wb{z}_{s,e,j}}{\wb{p}_{s,e,j}}\right)^{2} \condbar \F_{\{s,e,j\}-1} \right]\\
&= \expectedvalue \left[\frac{\wb{z}_{s-1,e,j}^2}{\wb{p}_{s-1,e,j}^2} -2 \frac{\wb{z}_{s-1,e,j}}{\wb{p}_{s-1,e,j}}\frac{\wb{z}_{s,e,j}}{\wb{p}_{s,e,j}} +\frac{\wb{z}_{s,e,j}^2}{\wb{p}_{s,e,j}^2} \condbar \F_{\{s,e,j\}-1} \right]\\
&\stackrel{(a)}{=} \expectedvalue \left[\frac{\wb{z}_{s-1,e,j}^2}{\wb{p}_{s-1,e,j}^2} -2 \frac{\wb{z}_{s-1,e,j}}{\wb{p}_{s-1,e,j}}\frac{\wb{z}_{s,e,j}}{\wb{p}_{s,e,j}} +\frac{\wb{z}_{s,e,j}^2}{\wb{p}_{s,e,j}^2} \condbar \F_{s-1} \right]\\
&= \frac{\wb{z}_{s-1,e,j}^2}{\wb{p}_{s-1,e,j}^2} -2 \frac{\wb{z}_{s-1,e,j}}{\wb{p}_{s-1,e,j}}\frac{1}{\wb{p}_{s,e,j}}\expectedvalue \left[\wb{z}_{s,e,j}\condbar \F_{s-1} \right] +\frac{1}{\wb{p}_{s,e,j}^2}\expectedvalue \left[\wb{z}_{s,e,j}^2 \condbar \F_{s-1} \right]\\
&\stackrel{(b)}{=} \frac{\wb{z}_{s-1,e,j}}{\wb{p}_{s-1,e,j}^2}
-2 \frac{\wb{z}_{s-1,e,j}}{\wb{p}_{s-1,e,j}}\frac{\wb{z}_{s-1,e,j}}{\wb{p}_{s-1,e,j}}
+\frac{1}{\wb{p}_{s,e,j}^2}\expectedvalue \left[\wb{z}_{s,e,j}\condbar \F_{s-1} \right]\\
&=\frac{1}{\wb{p}_{s,e,j}^2}\expectedvalue \left[\wb{z}_{s,e,j} \condbar \F_{s-1} \right] - \frac{\wb{z}_{s-1,e,j}}{\wb{p}_{s-1,e,j}^2}\\
&\stackrel{(c)}{=}\frac{1}{\wb{p}_{s,e,j}}\frac{\wb{z}_{s-1,e,j}}{\wb{p}_{s-1,e,j}} - \frac{\wb{z}_{s-1,e,j}}{\wb{p}_{s-1,e,j}^2}
=\frac{\wb{z}_{s-1,e,j}}{\wb{p}_{s-1,e,j}}\left(\frac{1}{\wb{p}_{s,e,j}} - \frac{1}{\wb{p}_{s-1,e,j}}\right)\CommaBin
\end{align*}
where in $(a)$ we use the fact that the approximate probabilities $\wb{p}_{s-1,e,j}$ and $\wb{p}_{s,e,j}$ and $\wb{z}_{s-1,e,j}$ are fixed at the end of the previous iteration, while in $(b)$ and $(c)$ we use the fact that $\wb{z}_{s,e,j}$ is a Bernoulli of parameter $\wb{p}_{s,e,j}/\wb{p}_{s-1,e,j}$ (whenever $\wb{z}_{s-1,e,j}$ is equal to 1).
Therefore, we can write $\:W_{\lidx}$ at the end of the process as
\begin{align*}
\:W_{h} = \:W_{\{h,m,\wb{q}\}} &= \frac{1}{\wb{q}^2}\sum_{j=1}^{\wb{q}} \sum_{e=1}^{\nhl} \sum_{s=1}^h \frac{\wb{z}_{s-1,e,j}}{\wb{p}_{s-1,e,j}}\left(\frac{1}{\wb{p}_{s,e,j}} - \frac{1}{\wb{p}_{s-1,e,j}}\right)\concvec_e\concvec_e^\transp\concvec_e\concvec_e^\transp.
\end{align*}

We can now upper-bound $\:W_h$ as
\begin{align*}
\:W_{h} &\preceq \frac{1}{\wb{q}^2}\sum_{j=1}^{\wb{q}} \sum_{e=1}^{\nhl} \sum_{s=1}^h \frac{\wb{z}_{s-1,e,j}}{\wb{p}_{s-1,e,j}}\left(\frac{1}{\wb{p}_{s,e,j}} - \frac{1}{\wb{p}_{s-1,e,j}}\right)\concvec_e\concvec_e^\transp\concvec_e\concvec_e^\transp\\
&= \frac{1}{\wb{q}^2}\sum_{j=1}^{\wb{q}} \sum_{e=1}^{\nhl} \left(\frac{\wb{z}_{h,e,j}}{\wb{p}_{h,e,j}^2} - \frac{\wb{z}_{h,e,j}}{\wb{p}_{h,e,j}^2} + \sum_{s=1}^h \frac{\wb{z}_{s-1,e,j}}{\wb{p}_{s-1,e,j}}\left(\frac{1}{\wb{p}_{s,e,j}} - \frac{1}{\wb{p}_{s-1,e,j}}\right)\right)\concvec_e\concvec_e^\transp\concvec_e\concvec_e^\transp\\
&= \frac{1}{\wb{q}^2}\sum_{j=1}^{\wb{q}} \sum_{e=1}^{\nhl} \left(\frac{\wb{z}_{h,e,j}}{\wb{p}_{h,e,j}^2} + \left(\sum_{s=1}^{h} -\frac{\wb{z}_{s,e,j}}{\wb{p}_{s,e,j}^2} + \frac{\wb{z}_{s-1,e,j}}{\wb{p}_{s,e,j}\wb{p}_{s-1,e,j}}\right) - \frac{\wb{z}_{0,e,j}}{\wb{p}_{0,e,j}^2}\right)\concvec_e\concvec_e^\transp\concvec_e\concvec_e^\transp\\
&\preceq \frac{1}{\wb{q}^2}\sum_{j=1}^{\wb{q}} \sum_{e=1}^{\nhl} \left(\frac{\wb{z}_{h,e,j}}{\wb{p}_{h,e,j}^2} + \left(\sum_{s=1}^{h}  \frac{\wb{z}_{s-1,e,j}}{\wb{p}_{s,e,j}\wb{p}_{s-1,e,j}} - \frac{\wb{z}_{s,e,j}}{\wb{p}_{s,e,j}\wb{p}_{s-1,e,j}}\right)\right)\concvec_e\concvec_e^\transp\concvec_e\concvec_e^\transp\\
&= \frac{1}{\wb{q}^2}\sum_{j=1}^{\wb{q}} \sum_{e=1}^{\nhl} \left(\frac{\wb{z}_{h,e,j}}{\wb{p}_{h,e,j}^2} + \sum_{s=1}^{h}  \frac{\wb{z}_{s-1,e,j}(1 - \wb{z}_{s,e,j})}{\wb{p}_{s,e,j}\wb{p}_{s-1,e,j}} \right)\concvec_e\concvec_e^\transp\concvec_e\concvec_e^\transp,
\end{align*}
where in the inequality we use the fact $\wb{p}_{s,e,j} \leq \wb{p}_{s-1,e,j}$. From the definition
of $\wb{p}_{s,e,j}$, we know that when $\wb{z}_{s,e,j} = 0$, $\wb{p}_{s,e,j} = \wb{p}_{s-1,e,j}$.
Therefore $\frac{\wb{z}_{s-1,e,j}(1 - \wb{z}_{s,e,j})}{\wb{p}_{s,e,j}\wb{p}_{s-1,e,j}} = \frac{\wb{z}_{s-1,e,j}(1 - \wb{z}_{s,e,j})}{\wb{p}_{s-1,e,j}^2}\CommaBin$ since the term is non-zero only when $\wb{z}_{s,e,j} = 0$.
Finally, we see that only one of the $\wb{z}_{s-1,e,j}(1-\wb{z}_{s,e,j})$ terms can be active for $s \in [h]$ and thus
\begin{align}
\:W_{h}
& \preceq \frac{1}{\wb{q}^2}\sum_{j=1}^{\wb{q}} \sum_{e=1}^{\nhl} \left(\frac{\wb{z}_{h,e,j}}{\wb{p}_{h,e,j}^2} + \sum_{s=1}^{h}  \frac{\wb{z}_{s-1,e,j}(1 - \wb{z}_{s,e,j})}{\wb{p}_{s-1,e,j}^2} \right)\concvec_e\concvec_e^\transp\concvec_e\concvec_e^\transp\nonumber\\
&=\frac{1}{\wb{q}^2}\sum_{j=1}^{\wb{q}} \sum_{e=1}^{\nhl} \left(\max\left\{ \max_{s=1,\dots,h}  \left\{\frac{\wb{z}_{s-1,e,j}(1 - \wb{z}_{s,e,j})}{\wb{p}_{s-1,e,j}^2}\right\} ; \frac{\wb{z}_{h,e,j}}{\wb{p}_{h,e,j}^2}\right\} \right)\concvec_e\concvec_e^\transp\concvec_e\concvec_e^\transp\nonumber\\
&= \frac{1}{\wb{q}^2}\sum_{j=1}^{\wb{q}}\sum_{e=1}^{\nhl}\concvec_e\concvec_e^\transp\concvec_e\concvec_e^\transp\left(\max_{s=0,\dots,h}\left\{\frac{\wb{z}_{s,e,j}}{\wb{p}_{s,e,j}^2}\right\}\right)\!\cdot\label{eq:dominance-W}
\end{align}

\textbf{Step \arabic{cnt-lem-quad-variation}\stepcounter{cnt-lem-quad-variation} (introduction of a stochastically dominant process).}
We want to study $\max_{s=0,\dots,h}\left\{\frac{\wb{z}_{s,e,j}}{\wb{p}_{s,e,j}^2}\right\}$.
To simplify notation, we will consider $\max_{s=0,\dots,h}\left\{\frac{\wb{z}_{s,e,j}}{\wb{p}_{s,e,j}}\right\}$,
where we removed the square, which will be re-added in the end.
We know trivially that this quantity is larger or equal than, 1 because $\wb{z}_{0,e,j}/\wb{p}_{0,e,j} = 1$,
but upper-bounding this quantity is not trivial as the evolution
of the various $\wb{p}_{s,e,j}$ depends in a complex way on the interaction
between the random variables $\wb{z}_{s,e,j}$.
Nonetheless, whenever $\wb{p}_{s,e,j}$ is significantly
smaller than $\wb{p}_{s-1,e,j}$, the probability of keeping a copy of edge $e$ at
iteration $s$ (i.e., $\wb{z}_{s,e,j}=1$) is also very small. As a result, we expect
the ratio $\frac{\wb{z}_{s,e,j}}{\wb{p}_{s,e,j}}$ to be still small with
high probability.

Unfortunately, due to the dependence between different copies
of the edge at different iterations, it seems difficult to exploit this intuition directly
to provide an overall high-probability bound on $\:W_{h}$. For this
reason, we simplify the analysis by replacing each of the (potentially
dependent) chains $\{\wb{z}_{s,e,j}/\wb{p}_{s,e,j}\}_{s=0}^{h}$ with a set of
(independent) random variables $w_{0,e,j}$ that will stochastically dominate
them.

We define the random variable $w_{s,e,j}$ using the following conditional
distribution,\footnote{
Notice that unlike $\wb{z}_{s,e,j}$, $w_{s,e,j}$ is no longer $\F_{s}$-measurable but it is $\F'_{s}$-measurable, where 
\begin{align*}
\F'_{\{s,e,j\}} = \left\{ u_{s',e',j'},\; \forall\{s',e',j'\} \leq \{s,e,j\} \right\} \cup \left\{  w_{s,e,j} \right\}
= \F_{\{s,e,j\}} \cup \left\{  w_{s,e,j} \right\}\!.
\end{align*}
} 
\begin{align*}
\probability\left(\frac{1}{w_{s,e,j}} \leq a \condbar \F_{s}\right)
 = \begin{cases}
0 &\text{ for }\quad a < 1/\wb{p}_{s,e,j}\\
1-\frac{1}{\wb{p}_{s,e,j}a} &\text{ for }\quad 1/\wb{p}_{s,e,j} \leq a < \alpha/p_{h,e}\\
1 &\text{ for }\quad \alpha/p_{h,e} \leq a
\end{cases}.
\end{align*}
To show that this distribution is well defined, we use Lem.\,\ref{lem:wbp-always-good}
to guarantee that $1/\wb{p}_{s,e,j} \leq a < \alpha/p_{h,e}$.
Note that the distribution of $\frac{1}{w_{s,e,j}}$ conditioned on $\F_{s}$
is determined by only $\wb{p}_{s,e,j}$, $p_{h,e}$, and~$\alpha$, where~$p_{h,e}$ and~$\alpha$
are fixed. Remembering that $\wb{p}_{s,e,j}$ is a function of
$\F_{s-1}$ (computed using the previous iteration),
we have that  
\begin{align*}
\probability\left(\frac{1}{w_{s,e,j}} \leq a \condbar \F_{s}\right)
= \probability\left(\frac{1}{w_{s,e,j}} \leq a \condbar \F_{s-1}\right).
    \end{align*}
Notice that  in the definition of $w_{s,e,j}$, none of the other $w_{s',e',j'}$
(for any different $s'$, $e'$, or $j'$) appears
and $\wb{p}_{s,e,j}$ is a function of
$\F_{s-1}$. It follows that given  $\F_{s-1}$, $w_{s,e,j}$ is independent from all other $w_{s',e',j'}$
(for any different $s'$, $e'$, or $j'$).
\begin{figure}[t]\label{fig:rand-var-dep-graph}
\begin{center}
\includegraphics[width=0.8\textwidth]{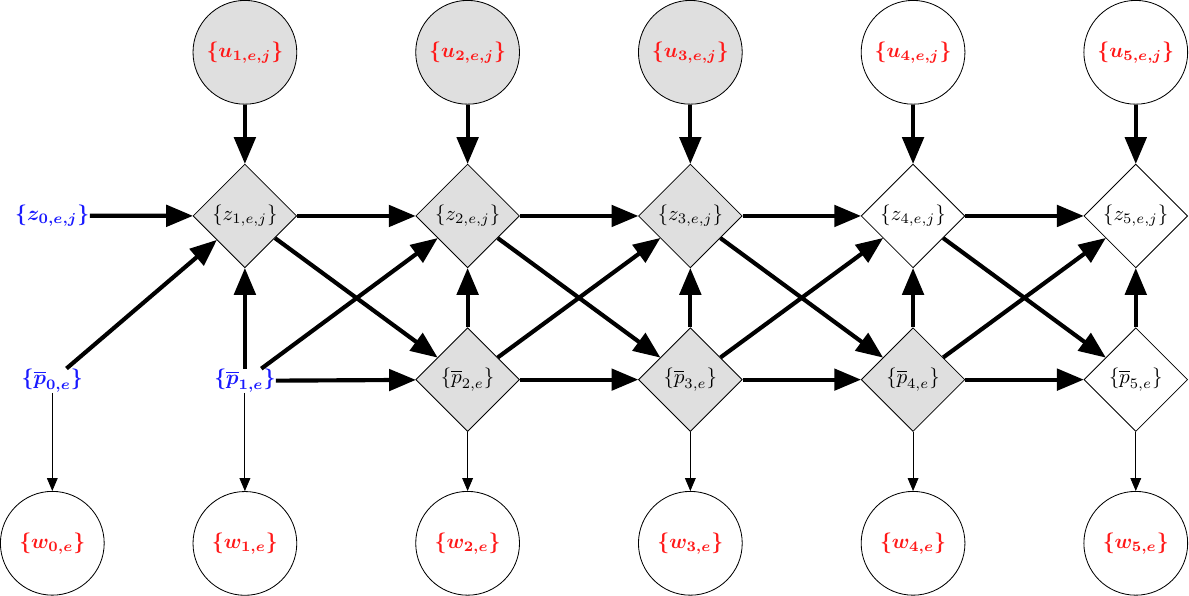}
\end{center}
\caption{The dependence graph of the considered variables. \rcolb{Red} variables are \rcolb{random}. Black variables are deterministically computed using their input (a function of their input), with bold lines indicating the deterministic (functional) relation. \bcolb{Blue} variables are \bcolb{constants}. A \colorbox{gray!30}{grey filling} indicates that a random variable is \colorbox{gray!30}{observed} or a function of observed variables.}
\label{fig:rand-var-dep-graph}
\end{figure}
This is easier to see in the probabilistic graphical model reported in
Fig.\,\ref{fig:rand-var-dep-graph}, which illustrates the dependence between
the various variables.

Finally for  the special case $w_{0,e,j}$ the definition above reduces to
\begin{align}\label{eq:distro.w}
\probability\left(\frac{1}{w_{0,e,j}} \leq a\right)
 = \begin{cases}
0 &\text{ for }\quad a < 1\\
1-\frac{1}{a} &\text{ for }\quad 1 \leq a < \alpha/p_{h,e}\\
1 &\text{ for }\quad \alpha/p_{h,e} \leq a
\end{cases},
\end{align}
since $\wb{p}_{0,e,j}=1$ by definition.
From this definition, $w_{0,e,j}$ and $w_{0,e',j'}$ are all
independent, and this will allow us to use stronger concentration inequalities
for independent random variables.

\textbf{Step \arabic{cnt-lem-quad-variation}\stepcounter{cnt-lem-quad-variation} {Proving the dominance}.}
We remind the reader that a random variable $A$ stochastically dominates random
variable $B$, if for all values $a$ the two equivalent conditions are verified,
\begin{align*}
\probability(A \geq a) \geq \probability(B \geq a) \Leftrightarrow \probability(A \leq a) \leq \probability(B \leq a).
\end{align*}
As a consequence, if $A$ dominates $B$, the following implication holds,
\begin{align*}
\probability(A \geq a) \geq \probability(B \geq a) \implies \expectedvalue[A] \geq \expectedvalue[B],
\end{align*}
while the reverse ($A$ dominates $B$, if $\expectedvalue[A] \geq \expectedvalue[B]$) is not
true in general.
Following this definition of  stochastic dominance, our goal is to prove
\begin{align*}
\probability\left(\max_{s=0}^h\frac{\wb{z}_{s,e,j}}{\wb{p}_{s,e,j}} \leq a\right)
\geq 
\probability\left(\frac{1}{w_{0,e,j}} \leq a \right).
\end{align*}
We prove this inequality by proceeding backwards with a sequence of conditional probabilities.
We first study the distribution of the maximum conditional to the state of the algorithm at the end of iteration $h$, i.e., $\Fp_{h}$. From the definition of $w_{h,e,j}$, we know that,
w.p.\,1, $1/\wb{p}_{h,e} \leq 1/w_{h,e,j}$.
Therefore,
\begin{align*}
&\probability\left(\max_{s=0,\dots,h}\frac{\wb{z}_{s,e,j}}{\wb{p}_{s,e,j}} \leq a\right)
\geq \probability\left(\max\left\{ \max_{s=0,\dots,h-1}\frac{\wb{z}_{s,e,j}}{\wb{p}_{s,e,j}} ; \frac{\wb{z}_{h,e,j}}{w_{h,e,j}}\right\} \leq a \right)\!.
\end{align*}%
Now focus on an arbitrary intermediate step $1\leq k\leq h$, where we fix $\Fp_{k-1}$. Since~$u_{k,e,j}$ and
$w_{k,e,j}$ are independent given $\Fp_{k-1}$, we have
\begin{align}
\probability\left( \frac{\wb{z}_{k,e,j}}{w_{k,e,j}} \leq a \condbar \Fp_{k-1}\right)
&=\probability\left( \indfunc\left\{u_{k,e,j} \leq \frac{\wb{p}_{k,e,j}}{\wb{p}_{k-1,e,j}}\right\}\frac{1}{w_{k,e,j}} \leq a \condbar \Fp_{k-1}\right)\nonumber\\
&=
\begin{cases}
0 &\text{ for }\quad a \leq 0 \nonumber\\
1 - \frac{\wb{p}_{k,e,j}}{\wb{p}_{k-1,e,j}} &\text{ for }\quad  0 \leq a < 1/\wb{p}_{k,e,j}\\
1-\frac{\wb{p}_{k,e,j}}{\wb{p}_{k-1,e,j}} +\frac{\wb{p}_{k,e,j}}{\wb{p}_{k-1,e,j}}\left(1-\frac{1}{\wb{p}_{k,e,j}a}\right) = 1-\frac{1}{\wb{p}_{k-1,e,j}a}  &\text{ for }\quad 1/\wb{p}_{k,e,j} \leq a < \alpha/p_{h,e} \\ 
1 &\text{ for }\quad \alpha/p_{h,e} \leq a
\end{cases}\nonumber\\
&\geq 
\begin{cases}
0 &\text{ for }\quad a < 1/\wb{p}_{k-1,e,j} \\
1-\frac{1}{\wb{p}_{k-1,e,j}a} &\text{ for }\quad  1/\wb{p}_{k-1,e,j} \leq a < 1/\wb{p}_{k,e,j} \\
1-\frac{1}{\wb{p}_{k-1,e,j}a}  &\text{ for }\quad 1/\wb{p}_{k,e,j} \leq a < \alpha/p_{h,e}  \\
1 &\text{ for }\quad \alpha/p_{h,e}  \leq a 
\end{cases}\label{w.stoch123}\\
&=\probability\left(\frac{1}{w_{k-1,e,j}} \leq a \condbar \Fp_{k-2}\right)
=\probability\left(\frac{1}{w_{k-1,e,j}} \leq a \condbar \Fp_{k-1}\right) \nonumber\!,
\end{align}
where the inequality is also represented in Fig.\,\ref{fig:dominance}.
\begin{figure}[t]
\centering
\includegraphics[width=0.6\textwidth]{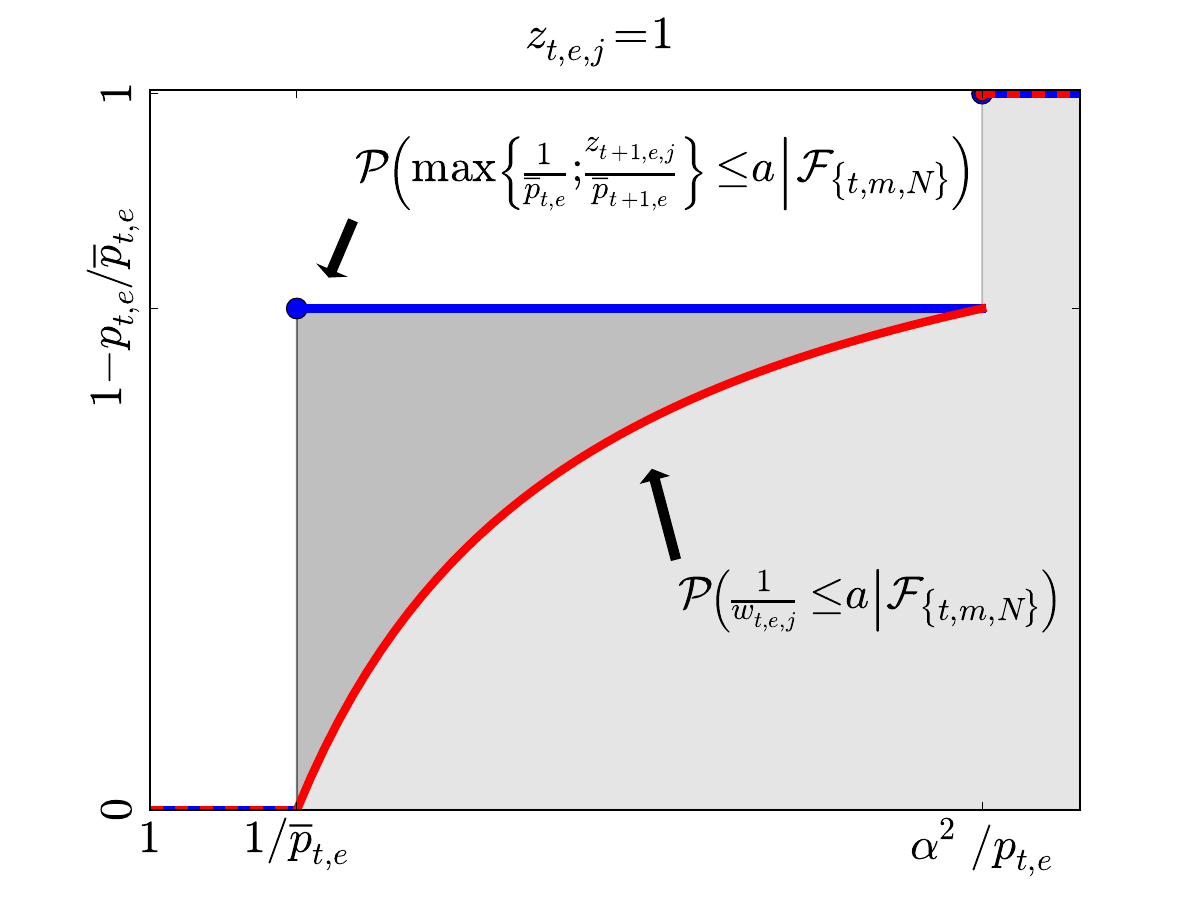}
\caption{C.d.f.\,of $\max\left\{ \wb{z}_{k-1,e,j}/\wb{p}_{t,e,j} ; \wb{z}_{k,e,j}/\wb{p}_{k,e,j}\right\}$ and
$\wb{z}_{k-1,e,j}/w_{k-1,e,j}$ conditioned on $\mathcal{F}_{\{k-1\}}$. For conciseness, we omit the $e,j$ indices.}\label{fig:dominance}
\end{figure}
We now proceed by peeling off layers from the end of the chain one by one, taking advantage of the dominance we just proved. Fig.\,\ref{fig:dominance} visualizes one step of the peeling when $\wb{z}_{k-1,e,j} = 1$ (note that the peeling is trivially true when $\wb{z}_{k-1,e,j} = 0$ since the whole chain terminated at step $\wb{z}_{k-1,e,j}$). We show how to move from an iteration $k\leq h$ to $k-1$.
%
\begin{align*}\label{eq:generic.case}
& \probability\left(\max\left\{ \max_{s=0 \dots k-1}\frac{\wb{z}_{s,e,j}}{\wb{p}_{s,e,j}} ; \frac{\wb{z}_{k,e,j}}{w_{k,e,j}}\right\} \leq a \right)
=\expectedvalue_{\Fp_{k-1}}\left[\probability\left(\max\left\{ \max_{s=0 \dots k-1}\frac{\wb{z}_{s,e,j}}{\wb{p}_{s,e,j}} ; \frac{\wb{z}_{k,e,j}}{w_{k,e,j}}\right\} \leq a \condbar \Fp_{k-1} \right)\right]\nonumber\\
&\stackrel{(a)}{\geq}\expectedvalue_{\Fp_{k-1}}\left[\probability\left(\max\left\{ \max_{s=0 \dots k-1}\frac{\wb{z}_{s,e,j}}{\wb{p}_{s,e,j}} ; \frac{\wb{z}_{k-1,e,j}}{w_{k-1,e,j}}\right\} \leq a \condbar \Fp_{k-1} \right)\right]\nonumber\\
&=\expectedvalue_{\Fp_{k-1}}\left[\probability\left(\max\left\{ \max_{s=0 \dots k-2}\frac{\wb{z}_{s,e,j}}{\wb{p}_{s,e,j}} ;\frac{\wb{z}_{k-1,e,j}}{\wb{p}_{k-1,e,j}}; \frac{\wb{z}_{k-1,e,j}}{w_{k-1,e,j}}\right\} \leq a \condbar \Fp_{k-1} \right)\right]\nonumber\\
&=\expectedvalue_{\Fp_{k-1}}\left[\probability\left(\max\left\{ \max_{s=0 \dots k-2}\frac{\wb{z}_{s,e,j}}{\wb{p}_{s,e,j}} ;\wb{z}_{k-1,e,j}\max\left\{\frac{1}{\wb{p}_{k-1,e,j}}; \frac{1}{w_{k-1,e,j}}\right\}\right\} \leq a \condbar \Fp_{k-1} \right)\right]\nonumber\\
&\stackrel{(b)}{=}\expectedvalue_{\Fp_{k-1}}\left[\probability\left(\max\left\{ \max_{s=0 \dots k-2}\frac{\wb{z}_{s,e,j}}{\wb{p}_{s,e,j}} ;\frac{\wb{z}_{k-1,e,j}}{w_{k-1,e,j}}\right\} \leq a \condbar \Fp_{k-1} \right)\right]
= \probability\left(\max\left\{ \max_{s=0 \dots k-2}\frac{\wb{z}_{s,e,j}}{\wb{p}_{s,e,j}} ; \frac{\wb{z}_{k-1,e,j}}{w_{k-1,e,j}}\right\} \leq a \right)\!,\nonumber
\end{align*}%
where in $(a)$, given $\F_{k-1}$, everything is fixed except $u_{k,e,j}$ and
$w_{k,e,j}$ and we can use the stochastic dominance in~\eqref{w.stoch123}, and
in $(b)$ we use the fact that the inner maximum is always attained by
$1/w_{k,e,j}$ since by definition $1/w_{k-1,e,j}$ is lower-bounded by
$1/\wb{p}_{k-1,e,j}$.
Applying the inequality recursively from $k=h$ to $k=1$ 
removes all $\wb{z}_{s,e,j}$ from the maximum and we
are finally left with only $w_{0,e,j}$ as we wanted,
\begin{align*}
\probability\left(\max_{s=0,\dots,h}\frac{\wb{z}_{s,e,j}}{\wb{p}_{s,e,j}} \leq a\right)
\geq 
\probability\left(\max\left\{\frac{\wb{z}_{0,e,j}}{\wb{p}_{0,e,j}} ; \frac{\wb{z}_{0,e,j}}{w_{0,e,j}}\right\} \leq a \right)
\geq 
\probability\left(\frac{1}{w_{0,e,j}} \leq a \right),
\end{align*}
where in the last inequality we used that $\wb{z}_{0,e,j} = 1$ from the definition of the algorithm 
and $\wb{p}_{0,e,j} = 1$ while $w_{0,e,j} \leq 1$ by~\eqref{eq:distro.w}.

\textbf{Step \arabic{cnt-lem-quad-variation}\stepcounter{cnt-lem-quad-variation} (stochastic dominance on $\:W_{h}$).}
Now that we proved the stochastic dominance of $1/w_{0,e,j}$, we plug this
result in the definition of $\:W_{h}$. For the sake of notation, we introduce
the term $\wb{p}_{h',e,j}^{\max}$
to indicate the maximum over the first $h'$ step of copy $e,j$ such that
\begin{align*}
\max_{s=0,\dots,h'}\frac{\wb{z}_{s,e,j}}{\wb{p}_{s,e,j}} = \frac{1}{\wb{p}_{h',e,j}^{\max}}\cdot
\end{align*}
We first notice that while $\wb{\:Y}_{h}$ is not
necessarily PSD, $\:W_{h}$ is a sum of PSD matrices.
Introducing the function $\Lambda(\{1/\wb{p}_{h,e,j}^{\max}\}_{e,j})$ we can restate Eq.\,\ref{eq:dominance-W} as
\begin{align*}
\normsmall{\:W_{h}}
=\lambda_{\max}(\:W_{h})
\leq \Lambda(\{1/\wb{p}_{h,e,j}^{\max}\}_{e,j}) \eqdef
\lambda_{\max}\left(\frac{1}{\wb{q}^2}\sum_{j=1}^{\wb{q}}\sum_{e=1}^{\nhl}\left(\frac{1}{\wb{p}_{h,e,j}^{\max}}\right)^2\concvec_e\concvec_e^\transp\concvec_e\concvec_e^\transp\right)\!.
\end{align*}
In Step 4, we showed that $1/\wb{p}_{h,e,j}^{\max}$ is stochastically
dominated by $1/w_{0,e,j}$ for every $e$ and~$j$. In order to bound
$\Lambda(\{1/\wb{p}_{h,e,j}^{\max}\}_{e,j})$,
we need to show that this dominance also applies to the summation
over all columns inside the matrix norm. We
can reformulate $\Lambda(\{1/\wb{p}_{h,e,j}^{\max}\}_{e,j})$ as
\begin{align*}
&\lambda_{\max}\left(\frac{1}{\wb{q}^2}\sum_{j=1}^{\wb{q}}\sum_{e=1}^{\nhl}\left(\frac{1}{\wb{p}_{h,e,j}^{\max}}\right)^2\concvec_e\concvec_e^\transp\concvec_e\concvec_e^\transp\right)
= \max_{\:x : \normsmall{\:x} = 1} \:x^\transp\left(\frac{1}{\wb{q}^2}\sum_{j=1}^{\wb{q}}\sum_{e=1}^{\nhl}\left(\frac{1}{\wb{p}_{h,e,j}^{\max}}\right)^2\concvec_e\concvec_e^\transp\concvec_e\concvec_e^\transp\right)\:x\\
&= \max_{\:x : \normsmall{\:x} = 1}\frac{1}{\wb{q}^2}\sum_{j=1}^{\wb{q}}\sum_{e=1}^{\nhl}\left(\frac{1}{\wb{p}_{h,e,j}^{\max}}\right)^2\normsmall{\concvec_e}_{2}^2\:x^\transp\concvec_e\concvec_e^\transp\:x
= \max_{\:x : \normsmall{\:x} = 1}\frac{1}{\wb{q}^2}\sum_{j=1}^{\wb{q}}\sum_{e=1}^{\nhl}\left(\frac{1}{\wb{p}_{h,e,j}^{\max}}\right)^2\left(\normsmall{\concvec_e}_{2}\concvec_e^\transp\:x\right)^2\!\!.
\end{align*}
From this reformulation, it is easy to see that, because $1/\wb{p}_{h,e,j}^{\max}$
is strictly positive,
the function $\Lambda(\{1/\wb{p}_{h,e,j}^{\max}\}_{e,j})$ is
monotonically increasing w.r.t.\,the individual $1/\wb{p}_{h,e,j}^{\max}$, or
in other words that increasing an $1/\wb{p}_{h,e,j}^{\max}$ without decreasing
the others can only increase the maximum.
Introducing $\Lambda(\{1/w_{0,e,j}\}_{e,j})$ as
\begin{align*}
\Lambda(\{1/w_{0,e,j}\}_{e,j}) \eqdef \max_{\:x : \normsmall{\:x} = 1}\frac{1}{\wb{q}^2}\sum_{j=1}^{\wb{q}}\sum_{e=1}^{\nhl}\left(\frac{1}{w_{0,e,j}}\right)^2\left(\normsmall{\concvec_e}_{2}\concvec_e^\transp\:x\right)^2\!\!\!,
\end{align*}
we now need to prove the stochastic dominance of $\Lambda(\{1/w_{0,e,j}\}_{e,j})$ over
$\Lambda(\{1/\wb{p}_{h,e,j}^{\max}\}_{e,j})$.
Using the definition of $1/\wb{p}_{h,e,j}^{\max}$, $w_{h,e,j}$, and the
monotonicity of $\Lambda$ we have
\begin{align*}
\probability\left(\Lambda\left(\left\{\frac{1}{\wb{p}_{h,e,j}^{\max}}\right\}_{e,j}\right) \leq a\right)
&=\probability\left(\Lambda\left(\left\{\max\left\{ \max_{s=0,\dots,h-1}\frac{\wb{z}_{s,e,j}}{\wb{p}_{s,e,j}} ; \frac{\wb{z}_{h,e,j}}{\wb{p}_{h,e,j}}\right\}\right\}_{e,j}\right) \leq a \right)\\
&\geq \probability\left(\Lambda\left(\left\{\max\left\{ \max_{s=0,\dots,h-1}\frac{\wb{z}_{s,e,j}}{\wb{p}_{s,e,j}} ; \frac{\wb{z}_{h,e,j}}{w_{h,e,j}}\right\}\right\}_{e,j}\right) \leq a \right)\!.
\end{align*}
Now pick $1\leq k \leq h$, for a fixed $\Fp_{k-1}$, $\frac{1}{\wb{p}_{k-1,e,j}^{\max}}$ is a constant
and $\max\left\{\frac{1}{\wb{p}_{k,e,j}^{\max}}; x\right\}$ is a monotonically
increasing function in~$x$, making $\Lambda\left(\max\left\{\frac{1}{\wb{p}_{k,e,j}^{\max}}; x\right\}\right)$ also an increasing function. Therefore, we have
\begin{align*}
&\probability\left(\Lambda\left(\left\{\max\left\{ \frac{1}{\wb{p}_{k-1,e,j}^{\max}} ; \frac{\wb{z}_{k,e,j}}{w_{k,e,j}}\right\} \right\}_{e,j}\right) \leq a \right)
=\expectedvalue_{\Fp_{k-1}}\left[\probability\left(\Lambda\left(\left\{\max\left\{ \frac{1}{\wb{p}_{k-1,e,j}^{\max}} ; \frac{\wb{z}_{k,e,j}}{w_{k,e,j}}\right\}\right\}_{e,j}\right) \leq a \condbar \Fp_{k-1}\right)\right]\\
&\stackrel{(a)}{\geq}\expectedvalue_{\Fp_{k-1}}\left[\probability\left(\Lambda\left(\left\{\max\left\{ \frac{1}{\wb{p}_{k-1,e,j}^{\max}} ; \frac{\wb{z}_{k-1,e,j}}{w_{k-1,e,j}}\right\}\right\}_{e,j}\right) \leq a \condbar \Fp_{k-1}\right)\right]\\
&\stackrel{(b)}{=}\expectedvalue_{\Fp_{k-1}}\left[\probability\left(\Lambda\left(\left\{\max\left\{ \frac{1}{\wb{p}_{k-2,e,j}^{\max}} ; \frac{\wb{z}_{k-1,e,j}}{w_{k-1,e,j}}\right\}\right\}_{e,j}\right) \leq a \condbar \Fp_{k-1}\right)\right]\!,
\end{align*}
where inequality (a) follows from the fact that stochastic dominance is preserved
by monotonically increasing functions \cite{levy2015stochastic}, such as
$\Lambda$, combined with the fact that for
a fixed $\Fp_{k-1}$ the variables $\wb{z}_{k,e,j}$ and $w_{k,e,j}$ are all
independent and (b) from the definition of $1/\wb{p}_{k-1,e,j}^{\max}$ and the
fact that by definition $1/w_{k-1,e,j}$ is lower-bounded by $1/\wb{p}_{k-1,e,j}$.
We can iterate this inequality to obtain the desired result
\begin{align*}
\probability(\normsmall{\:W_{h}} \geq \sigma^2)
&\leq \probability\left(\Lambda\left(\left\{\frac{1}{\wb{p}_{h,e,j}^{\max}}\right\}_{e,j}\right) \geq \sigma^2\right)
\leq \probability\left(\lambda_{\max}\left(\frac{1}{\wb{q}^2}\sum_{j=1}^{\wb{q}}\sum_{e=1}^{\nhl}\left(\frac{1}{w_{0,e,j}}\right)^2\concvec_e\concvec_e^\transp\concvec_e\concvec_e^\transp\right) \geq \sigma^2\right)\!.
\end{align*}

\textbf{Step \arabic{cnt-lem-quad-variation}\stepcounter{cnt-lem-quad-variation} (concentration inequality).}
Since all $w_{0,e,j}$ are (unconditionally) independent from each other, we can apply the following theorem.
\begin{proposition}[\citealp{tropp2015an-introduction}, Theorem~5.1.1]\label{prop:matrix-chernoff}
Consider a finite sequence $\{ \:X_k : k = 1, 2, 3, \dots \}$ whose values are independent, random, PSD Hermitian matrices with dimension $d$.  Assume that each term in the sequence is uniformly bounded in the sense that
\begin{align*}
\lambda_{\max}( \:X_k ) \leq L
\quad\text{almost surely}
\quad\text{for $k = 1, 2, 3, \dots$}.
\end{align*}
Introduce the random matrix
$\:V \eqdef \sum_{k}  \:X_k$, and
the maximum eigenvalue of its expectation
\begin{align*}
\mu_{\max} \eqdef \lambda_{\max}(\expectedvalue\left[\:V\right]) = \lambda_{\max}\left(\sum_{k} \expectedvalue\left[\:X_k\right]\right)\!.
\end{align*}
Then, for all $h \geq 0$,
\begin{align*}
\probability\left(  \lambda_{\max}(\:V) \geq (1+h)\mu_{\max}  \right)
	&\leq d \cdot \left[\frac{e^{h}}{(1+h)^{1+h}}\right]^{\frac{\mu_{\max}}{L}}\\
	&\leq d \cdot \exp \left\{ - \frac{\mu_{\max}}{L}((h+1)\log(h+1) - h) \right\}\!\cdot
\end{align*}
\end{proposition}

In our case, we have
\begin{align*}
\:X_{\{e,j\}}
= \frac{1}{\wb{q}^2}\frac{1}{w_{0,e,j}}\concvec_e\concvec_e^\transp\concvec_e\concvec_e^\transp
\preceq \frac{1}{\wb{q}^2}\frac{\alpha^2}{p_{h,e}^2}\concvec_e\concvec_e^\transp\concvec_e\concvec_e^\transp
\preceq \frac{1}{\wb{q}^2}\frac{\alpha^2}{p_{h,e}^2}\normsmall{\concvec_e\concvec_e^\transp}^2\:I
\preceq \frac{\alpha^2}{\wb{q}^2}\:I,
\end{align*}
where the first inequality follows from the definition of $w_{0,e,j}$ in Eq.\,\ref{eq:distro.w}, the second
from the PSD ordering, and the third from the definition of $\normsmall{\concvec_e\concvec_e^\transp}$.

Therefore, we can use $L \eqdef \alpha^2/\wb{q}^2$ for the purpose of Prop.\,\ref{prop:matrix-chernoff}. We need now to compute $\expectedvalue\left[\:X_k\right]$,
that we can use in turn to compute $\mu_{\max}$. We begin by computing the expected value of $1/w_{0,e,j}$.
Let us denote the c.d.f.\@ of $1/w_{0,e,j}^2$ as
\begin{align*}
F_{1/w_{0,e,j}^2}(a) = \probability\left(\frac{1}{w_{0,e,j}^2} \leq a\right) = \probability\left(\frac{1}{w_{0,e,j}} \leq \sqrt{a}\right)
 = \begin{cases}
0 &\text{ for }\quad a < 1\\
1-\frac{1}{\sqrt{a}} &\text{ for }\quad 1 \leq a < \alpha^2/p_{h,e}^2\\
1 &\text{ for }\quad \alpha^2/p_{h,e}^2 \leq a
\end{cases}.
\end{align*}
Since $\probability\left(1/w_{0,e,j}^2 \ge 0\right) = 1$, we have that 
\newcommand*\diff{\mathop{}\!\mathrm{d}}
\begin{align*}
&\expectedvalue\left[\frac{1}{w_{0,e,j}}\right]
= \int_{a=0}^\infty \left[1-F_{1/w_{0,e,j}}(a)\right]\diff a\\
&= \int_{a=0}^1 \left(1 - F_{1/w_{0,e,j}}(a)\right)\diff a + \int_{a=1}^{\alpha^2/p_{h,e}^2} \left(1 - F_{1/w_{0,e,j}}(a)\right)\diff a + \int_{a=\alpha^2/p_{h,e}^2}^{\infty} \left(1 - F_{1/w_{0,e,j}}(a)\right)\diff a\\
&= \int_{a=0}^1 \left(1 - 0\right)\diff a + \int_{a=1}^{\alpha^2/p_{h,e}^2} \left(1 - \left(1 - \frac{1}{\sqrt{a}}\right) \right)\diff a +\int_{a=\alpha^2/p_{h,e}^2}^{\infty}(1-1)\diff a \\
&= \int_{a=0}^1 \diff a + \int_{a=1}^{\alpha^2/p_{h,e}^2} \frac{1}{\sqrt{a}}\diff a
= 1+[2\sqrt{a}]_{1}^{\alpha^2/p_{h,e}^2}
= 2\alpha/p_{h,e} - 1.
\end{align*}
%
Therefore,
\begin{align*}
\mu_{\max} &= \lambda_{\max}(\expectedvalue\left[\:V\right]) = \lambda_{\max}\Big(\sum_{\{e,j\}} \expectedvalue\left[\:X_{\{e,j\}}\right]\Big)
= \lambda_{\max}\left(\frac{1}{\wb{q}^2}\sum_{j=1}^{\wb{q}}\sum_{e=1}^{\nhl}\expectedvalue\left[\frac{1}{w_{0,e,j}^2}\right]\concvec_e\concvec_e^\transp\concvec_e\concvec_e^\transp\right)\\
&= \lambda_{\max}\left(\frac{1}{\wb{q}}\sum_{e=1}^{\nhl}\left(\frac{2\alpha}{p_{h,e}} - 1\right)p_{h,e}\concvec_e\concvec_e^\transp\right)
\leq \lambda_{\max}\left(\frac{2\alpha}{\wb{q}}\sum_{e=1}^{\nhl}\concvec_e\concvec_e^\transp\right)
= \frac{2\alpha}{\wb{q}}\lambda_{\max}\left(\:P\right)
\leq \frac{2\alpha}{\wb{q}} \eqdef L.
\end{align*}
Therefore, selecting $h = 2$, $\sigma^2=6\alpha/\wb{q}$ and applying Prop.\,\ref{prop:matrix-chernoff} we have
\begin{align*}
\probability\left(\normsmall{\:W_h} \geq \sigma^2\right)
&\leq \probability\left(\lambda_{\max}\left(\frac{1}{\wb{q}^2}\sum_{j=1}^{\wb{q}}\sum_{e=1}^{\nhl}\frac{1}{w_{0,e,j}^2}\concvec_e\concvec_e^\transp\concvec_e\concvec_e^\transp\right) \geq (1+2)\frac{2\alpha}{\wb{q}}\right)\\
&\leq \nhl \cdot \exp \left\{ - \frac{2\alpha}{\wb{q}}\frac{\wb{q}^2}{\alpha^2}(3\log(3) - 2) \right\}\\
&\leq n \cdot \exp \left\{ - \frac{2\wb{q}}{\alpha} \right\}\!\cdot
\end{align*}

%
\subsection{Space complexity bound}\label{ssec:space-complexity}

Denote with $A$ the event $A = \left\{\forall  h' \in \{1, \dots, h\} :  \normsmall{\:P^{h'} - \wt{\:P}^{h'}}_2 \leq \varepsilon\right\}$, and again  $\nhl = |\Gg_{\{h,\ell\}}|$.
To begin with, we can show that under event $A$, the approximate $\gamma$-effective resistances
(i.e.\,RLS) are accurate (Lem.\,2, \citealt{calandriello_disqueak_2017}).
Let $\Hg = \Hg_{\{h,\ell\}}$ be a $(\varepsilon,2\gamma)$-sparsifier of $\Gg = \Gg_{\{h,\ell\}}$.
Using Prop.\,\ref{prop:ordering-to-norm-bound} it is straightforward to see that
under $A$ we have
\begin{align*}
\wt{p}_{h+1,e} \leq &\wt{r}_{h+1,e}(\gamma) = (1-\vareps)\:b_e^\transp(\laplacian_{\Hg} + (1+\vareps)\gamma\:I)^{-1}\:b_e\\
&\leq (1-\vareps)\:b_e^\transp((1-\varepsilon)\laplacian_{\Gg} - 2\varepsilon\gamma\:I + (1+\vareps)\gamma\:I)^{-1}\:b_e
= \frac{(1-\vareps)}{(1-\vareps)}\:b_e^\transp(\laplacian_{\Gg} + \gamma\:I)^{-1}\:b_e
= r_{h+1,e} = p_{h+1,e}.
\end{align*}

Letting $q~=~|\Hg_{\{h,\ell\}}| = \sum_{e=1}^{\nhl} q_{h,e} = \sum_{j=1}^{\wb{q}}\sum_{e=1}^{\nhl}  z_{h,e,j}$ be the random number
of edges in $\coldict_{\{h,\ell\}}$, we reformulate
\begin{align*}
    &\probability\left(|\coldict_{\{h,\ell\}}| \geq 3\wb{q}\deff^{\{h,\ell\}}(\gamma) \cap \left\{\forall  h' \in \{1, \dots, h\} :  \left(\normsmall{\:P^{h'} - \wt{\:P}^{h'}}_2 \leq \varepsilon\right) \leq \varepsilon\right\}\right)\\
    &\hspace{5cm}= \probability\left(|\coldict_{\{h,\ell\}}| \geq 3\wb{q}\deff^{\{h,\ell\}}(\gamma) \cap A\right)\\
&\hspace{5cm}= \probability\left(\sum_{j=1}^{\wb{q}}\sum_{e=1}^{\nhl}  z_{h,e,j} \geq 3\wb{q}\deff^{\{h,\ell\}}(\gamma) \cap A\right)\\
&\hspace{5cm}=\probability\left(  \sum_{j=1}^{\wb{q}}\sum_{e=1}^{\nhl} z_{h,e,j} \geq 3\wb{q}\deff^{\{h,\ell\}}(\gamma) \condbar A \right)
\probability\left(A \right)\!.
\end{align*}
While we do know that the $z_{h,e,j}$ are Bernoulli random variables (since
they are either 0 or 1), it is not easy to compute the success probability
of each $z_{h,e,j}$, and in addition there could be dependencies between
$z_{h,e,j}$ and $z_{h,e',j'}$. Similarly to Lem.\,\ref{lem:prob-dominance},
we are going to find a stochastic variable to dominate $z_{h,e,j}$.
Denoting with $u'_{s,e,j} \sim \mathcal{U}(0,1)$ a uniform random variable,
we will define $w'_{s,e,j}$ as
\begin{align*}
w'_{s,e,j} \vert \F_{\{s,e',j'\}} \eqdef w'_{s,e,j} \vert   \F_{s-2}  \eqdef \indfunc\left\{u'_{s,e,j} \leq \frac{p_{h,e}}{\wt{p}_{s-1,e}}\right\} \sim \mathcal{B}\left(\frac{p_{h,e}}{\wt{p}_{s-1,e}}\right)
\end{align*}
for any
$e'$ and $j'$ such that $\{s,1,1\} \leq \{s,e',j'\} <\{s,e,j\}$.
Note that $w'_{s,e,j}$, unlike $z_{s,e,j}$, does not have a recursive
definition, and its only dependence on any other variable comes from
$\wt{p}_{s-1,e}$.
First, we peel off the last step
\begin{align*}
&\probability\left( \sum_{j=1}^{\wb{q}} \sum_{e=1}^{\nhl} z_{h,e,j} \geq g \condbar A\right)
= \expectedvalue_{\Fp_{t-1}|A}\left[\probability\left( \sum_{j=1}^{\wb{q}} \sum_{e=1}^{\nhl} \indfunc\left\{u_{h,e,j} \leq \frac{\wt{p}_{h,e}}{\wt{p}_{t-1,e}}\right\}z_{h-1,e,j} \geq g \condbar \Fp_{t-1} \cap A\right)\right]\\
&\leq \expectedvalue_{\Fp_{t-1}|A}\left[\probability\left( \sum_{j=1}^{\wb{q}} \sum_{e=1}^{\nhl} \indfunc\left\{u'_{h,e,j} \leq \frac{p_{h,e}}{\wt{p}_{t-1,e}}\right\}z_{h-1,e,j} \geq g \condbar \Fp_{t-1} \cap A\right)\right]
= \probability\left( \sum_{j=1}^{\wb{q}} \sum_{e=1}^{\nhl} w'_{h,e,j}z_{h-1,e,j} \geq g \condbar A\right)\!,
\end{align*}
where we used the fact that conditioned on $A$, $\coldict_{\{h,\ell\}}$ is an $(\varepsilon,\gamma)$-sparsifier of $\Gg_{\{h,\ell\}}$, which guarantees that $\wt{p}_{h,e} \leq p_{h,e}$.
Plugging this in the previous bound,
\begin{align*}
\probability&\left( \sum_{j=1}^{\wb{q}} \sum_{e=1}^{\nhl} z_{h,e,j} \geq g \condbar A \right)
\probability\left(A \right)
\leq\probability\left( \sum_{j=1}^{\wb{q}} \sum_{e=1}^{\nhl} w'_{h,e,j}z_{h-1,e,j} \geq g \cap A\right)
\leq\probability\left( \sum_{j=1}^{\wb{q}} \sum_{e=1}^{\nhl} w'_{h,e,j}z_{h-1,e,j} \geq g \right)\!.
\end{align*}
We now proceed by peeling off layers from the end of the chain one by one. We show how to move from an iteration $s\leq h$ to $s-1$,
\begin{align*}
&\probability\left( \sum_{j=1}^{\wb{q}} \sum_{e=1}^{\nhl} w'_{s,e,j}z_{s-1,e,j} \geq g \right)
= \expectedvalue_{\Fp_{s-2}}\left[\probability\left( \sum_{j=1}^{\wb{q}} \sum_{e=1}^{\nhl} \indfunc\left\{u'_{s,e,j} \leq \frac{p_{h,e}}{\wt{p}_{s-1,e}}\right\}z_{s-1,e,j} \geq g \condbar \Fp_{s-2} \right)\right]\\
&= \expectedvalue_{\Fp_{s-2}}\left[\probability\left( \sum_{j=1}^{\wb{q}} \sum_{e=1}^{\nhl} \indfunc\left\{u'_{s,e,j} \leq \frac{p_{h,e}}{\wt{p}_{s-1,e}}\right\}\indfunc\left\{u_{s-1,e,j} \leq \frac{\wt{p}_{s-1,e}}{\wt{p}_{s-2,e}}\right\}z_{s-2,e,j} \geq g \condbar \Fp_{s-2} \right)\right]\\
&= \expectedvalue_{\Fp_{s-2}}\left[\probability\left( \sum_{j=1}^{\wb{q}} \sum_{e=1}^{\nhl} \indfunc\left\{u'_{s-1,e,j} \leq \frac{p_{h,e}}{\wt{p}_{s-2,e}}\right\}z_{s-2,e,j} \geq g \condbar \Fp_{s-2} \right)\right]
= \probability\left( \sum_{j=1}^{\wb{q}} \sum_{e=1}^{\nhl} w'_{s-1,e,j}z_{s-2,e,j} \geq g \right)\!.
\end{align*}
Applying this repeatedly from $s=h$ to $s=2$, we have
\begin{align*}
&\probability\left( \sum_{j=1}^{\wb{q}} \sum_{e=1}^{\nhl} w'_{h,e,j}z_{h-1,e,j} \geq g \right)
=\probability\left( \sum_{j=1}^{\wb{q}} \sum_{e=1}^{\nhl} w'_{1,e,j}z_{0,e,j} \geq g \right)
=\probability\left( \sum_{j=1}^{\wb{q}} \sum_{e=1}^{\nhl} w'_{1,e,j} \geq g \right)\!. 
\end{align*}
Now, all the $w'_{1,e,j}$ are independent Bernoulli random variables,
and we can bound their sum with a Hoeffding-like bound using Markov inequality,
\begin{align*}
\probability&\left( \sum_{j=1}^{\wb{q}} \sum_{e=1}^{\nhl} w'_{1,e,j} \geq g \right) = \inf_{\theta > 0}\probability\left( e^{\sum_{j=1}^{\wb{q}} \sum_{e=1}^{\nhl} \theta w'_{1,e,j}} \geq e^{\theta g} \right)\\
&\leq \inf_{\theta > 0} \frac{\expectedvalue \left[ e^{\sum_{j=1}^{\wb{q}} \sum_{e=1}^{\nhl} \theta w'_{1,e,j}}\right]}{e^{\theta g}}
= \inf_{\theta > 0} \frac{\expectedvalue \left[ \prod_{j=1}^{\wb{q}} \prod_{e=1}^{\nhl} e^{ \theta w'_{1,e,j}}\right]}{e^{\theta g}}
= \inf_{\theta > 0} \frac{\prod_{j=1}^{\wb{q}} \prod_{e=1}^{\nhl} \expectedvalue \left[ e^{ \theta w'_{1,e,j}}\right]}{e^{\theta g}}\\
&=\inf_{\theta > 0} \frac{\prod_{j=1}^{\wb{q}} \prod_{e=1}^{\nhl} (p_{h,e} e^\theta + (1-p_{h,e}))}{e^{\theta g}}
= \inf_{\theta > 0} \frac{\prod_{j=1}^{\wb{q}} \prod_{e=1}^{\nhl} (1+p_{h,e}( e^\theta -1))}{e^{\theta g}}\\
&\leq \inf_{\theta > 0} \frac{\prod_{j=1}^{\wb{q}} \prod_{e=1}^{\nhl} e^{p_{h,e}( e^\theta -1)}}{e^{\theta g}}
\leq \inf_{\theta > 0} \frac{e^{\wb{q}( e^\theta -1)\sum_{e=1}^{\nhl} p_{h,e}}}{e^{\theta g}}
= \inf_{\theta > 0} e^{(\deff^{\{h,\ell\}}(\gamma)\wb{q}(e^\theta - 1) - \theta g)}
\leq \inf_{\theta > 0} e^{(\deff^{\{h,\ell\}}(\gamma)\wb{q}(e^\theta - 1) - \theta g)},
\end{align*}
where we use the fact that $1 + x \leq e^x$, $w'_{1,e,j} \sim \mathcal{B}(p_{h,e})$ and by Def.\,\ref{def:eff-res}, $\sum_{e=1}^{\nhl} p_{h,e} = \sum_{e=1}^{\nhl} r_{h,e} = \deff^{\{h,\ell\}}(\gamma)$.
The choice of $\theta$ minimizing the previous expression is obtained as
\begin{align*}
\frac{d}{d\theta}e^{\left(\wb{q}\deff^{\{h,\ell\}}(\gamma)(e^\theta - 1) - \theta g \right)}
= e^{\left(\wb{q}\deff^{\{h,\ell\}}(\gamma)(e^\theta - 1) - \theta g \right)}\left(\wb{q}\deff^{\{h,\ell\}}(\gamma)e^\theta - g\right) = 0,
\end{align*}
and thus $\theta = \log (g/(\wb{q}\deff^{\{h,\ell\}}(\gamma)))$. Plugging this in the previous bound,
\begin{align*}
\inf_\theta \exp\left\{ \wb{q}\deff^{\{h,\ell\}}(\gamma)(e^\theta - 1) - \theta g)\right\}
&= \exp\left\{ g  - \wb{q}\deff^{\{h,\ell\}}(\gamma) -  g \log \left(\frac{g}{\wb{q}\deff^{\{h,\ell\}}(\gamma)}\right)\right\}\\
&= \exp\left\{ -g\left(\log \left(\frac{g}{\wb{q}\deff^{\{h,\ell\}}(\gamma)}\right) - 1\right)\right\} e^{- \wb{q}\deff^{\{h,\ell\}}(\gamma)},
\end{align*}
and choosing $g = 3\wb{q}\deff^{\{h,\ell\}}(\gamma)$, we conclude our proof.

\section{Further details on experiments}
\label{appE}
\textbf{Software and hardware} All our code is implemented in Julia and runs
on a cluster of 4 machines with $128$ GB of RAM and a 10-core Xeon E5-2630.
The distributed computation is achieved using a suboptimal but simple producer-consumer
queue.

\textbf{Linear solver} We use approximate Gaussian elimination scheme by~\citet{kyng2016approximate}
from \path{Laplacians.jl}\footnote{\url{http://github.com/danspielman/Laplacians.jl}} package.

\textbf{Details on \disre.} To compute all $(\varepsilon,\gamma)$ sparsifiers, we use \disre after splitting
the input graph into 8 sub-graphs,\footnote{To guarantee that each of the sub-graphs is defined on the same set of nodes, we pre-construct a spanning tree of $\Gg$ and include it in each of the sub-graphs. Note that the analysis holds even if each sub-graph is disconnected from the others. We choose this approach to avoid complicating the code with additional searches of connected components in the sparsifiers.} resulting in 3 rounds of resparsifications using 4 machines.
For each resparsification, we compute the effective resistance in parallel
on each machine using 10 processes.\footnote{
\path{Laplacian.jl} is strictly single-threaded. Therefore, we use multiple processes on a single machine.
Faster runtime and lower memory usage could be achieved by sharing the memory with threads.
}

\textbf{Additional results.} For completeness, in the following 
with provide the tables containing all the combinations of
hyper-parameters ($\varepsilon,\gamma,\qbar,k,\sigma,l$) used
in our experiments.
As the supplementary files, we also provide two OpenDocument spreadsheets
containing these results, one for the Laplacian smoothing experiment (\path{extra_results_smoothing.ods})
and one for the SSL experiment (\path{extra_results_ssl.ods}).

\begin{table}[]
\small
\centering
\begin{tabular}{@{}rrrr|@{}}
\toprule
\# labels & $\lambda $    & mean   & std    \\ \midrule
20          & 1.00E-006 & 0.4117 & 0.0539 \\
20          & 1.00E-004 & 0.3989 & 0.0548 \\
20          & 1.00E-002 & 0.4489 & 0.0662 \\
20          & 1.00E+000 & 0.4109 & 0.0464 \\
346         & 1.00E-006 & 0.3145 & 0.0165 \\
346         & 1.00E-004 & 0.3327 & 0.0477 \\
346         & 1.00E-002 & 0.3526 & 0.0813 \\
346         & 1.00E+000 & 0.4357 & 0.1084 \\
672         & 1.00E-006 & 0.2967 & 0.0154 \\
672         & 1.00E-004 & 0.3334 & 0.0534 \\
672         & 1.00E-002 & 0.3718 & 0.0971 \\
672         & 1.00E+000 & 0.3197 & 0.0481 \\
1000        & 1.00E-006 & 0.2795 & 0.0053 \\
1000        & 1.00E-004 & 0.2963 & 0.0542 \\
1000        & 1.00E-002 & 0.3418 & 0.0706 \\
1000        & 1.00E+000 & 0.3578 & 0.0937 \\ \bottomrule
\end{tabular}
%
%
\small
\begin{tabular}{@{}|rrrr|@{}}
\toprule
\# labels & $\lambda $    & mean   & std    \\ \midrule
20   & 1.00E-006 & 0.4308 & 0.0527 \\
20   & 1.00E-004 & 0.4121 & 0.0537 \\
20   & 1.00E-002 & 0.4274 & 0.0656 \\
20   & 1.00E+000 & 0.4343 & 0.0688 \\
346  & 1.00E-006 & 0.3144 & 0.0158 \\
346  & 1.00E-004 & 0.3161 & 0.0245 \\
346  & 1.00E-002 & 0.3783 & 0.0981 \\
346  & 1.00E+000 & 0.3867 & 0.1009 \\
672  & 1.00E-006 & 0.2957 & 0.0089 \\
672  & 1.00E-004 & 0.3063 & 0.0424 \\
672  & 1.00E-002 & 0.3710 & 0.1022 \\
672  & 1.00E+000 & 0.3107 & 0.0244 \\
1000 & 1.00E-006 & 0.2830 & 0.0078 \\
1000 & 1.00E-004 & 0.3200 & 0.0719 \\
1000 & 1.00E-002 & 0.3284 & 0.0769 \\
1000 & 1.00E+000 & 0.3257 & 0.0786
\\ \bottomrule
\end{tabular}
\label{my-label}
\begin{tabular}{@{}rrrr@{}}
\toprule
\# labels & $\lambda $    & mean   & std    \\ \midrule
20   & 1.00E-006 & 0.4354 & 0.0447 \\
20   & 1.00E-004 & 0.4316 & 0.0402 \\
20   & 1.00E-002 & 0.4251 & 0.0426 \\
20   & 1.00E+000 & 0.4186 & 0.0343 \\
346  & 1.00E-006 & 0.3293 & 0.0143 \\
346  & 1.00E-004 & 0.3480 & 0.0333 \\
346  & 1.00E-002 & 0.3425 & 0.0218 \\
346  & 1.00E+000 & 0.3493 & 0.0349 \\
672  & 1.00E-006 & 0.3112 & 0.0277 \\
672  & 1.00E-004 & 0.3219 & 0.0219 \\
672  & 1.00E-002 & 0.3112 & 0.0141 \\
672  & 1.00E+000 & 0.3235 & 0.0121 \\
1000 & 1.00E-006 & 0.3127 & 0.0225 \\
1000 & 1.00E-004 & 0.3028 & 0.0142 \\
1000 & 1.00E-002 & 0.2947 & 0.0104 \\
1000 & 1.00E+000 & 0.3025 & 0.0128
\\ \bottomrule
\end{tabular}
\vskip 0.5em
\ssl, \textbf{Left:} \disre, $\qbar=100$, \textbf{Middle:}  \disre, $\qbar=150$, \textbf{Right:} \kn, $k=60$
\end{table}

\begin{table}[]
\small
\centering
\label{my-label}
\begin{tabular}{@{}rrrr|@{}}
\toprule
\# labels & $\lambda $    & mean   & std    \\ \midrule
20   & 1.00E-006 & 0.4477 & 0.0299 \\
20   & 1.00E-004 & 0.4389 & 0.0352 \\
20   & 1.00E-002 & 0.4269 & 0.0364 \\
20   & 1.00E+000 & 0.4208 & 0.0441 \\
346  & 1.00E-006 & 0.3533 & 0.0378 \\
346  & 1.00E-004 & 0.3341 & 0.0248 \\
346  & 1.00E-002 & 0.3431 & 0.0388 \\
346  & 1.00E+000 & 0.3628 & 0.0261 \\
672  & 1.00E-006 & 0.3248 & 0.0263 \\
672  & 1.00E-004 & 0.3119 & 0.0237 \\
672  & 1.00E-002 & 0.3127 & 0.0138 \\
672  & 1.00E+000 & 0.3253 & 0.0333 \\
1000 & 1.00E-006 & 0.3107 & 0.0147 \\
1000 & 1.00E-004 & 0.3076 & 0.0115 \\
1000 & 1.00E-002 & 0.3073 & 0.0238 \\
1000 & 1.00E+000 & 0.2952 & 0.0146
\\ \bottomrule
\end{tabular}
%
\small
\centering
\label{my-label}
\begin{tabular}{@{}rrrr@{}}
\toprule
\# labels & $\lambda $    & mean   & std    \\ \midrule
20   & 1.00E-006 & 0.4047 & 0.0544 \\ 
20   & 1.87E-003 & 0.4381 & 0.0841 \\
20   & 1.23E-002 & 0.4499 & 0.0762 \\
20   & 5.34E-001 & 0.4255 & 0.0641 \\
346  & 1.00E-006 & 0.3120 & 0.0218 \\
346  & 1.87E-003 & 0.3553 & 0.0770 \\
346  & 1.23E-002 & 0.3654 & 0.0915 \\
346  & 5.34E-001 & 0.3318 & 0.0544 \\
672  & 1.00E-006 & 0.2863 & 0.0101 \\
672  & 1.87E-003 & 0.3293 & 0.0388 \\
672  & 1.23E-002 & 0.4105 & 0.1056 \\
672  & 5.34E-001 & 0.3470 & 0.0695 \\
1000 & 1.00E-006 & 0.2772 & 0.0049 \\
1000 & 1.87E-003 & 0.3181 & 0.0414 \\
1000 & 1.23E-002 & 0.3070 & 0.0533 \\
1000 & 5.34E-001 & 0.3063 & 0.0860 \\ \bottomrule
\end{tabular}
\vskip 0.5em
\ssl, \textbf{Left:}  \kn, $k=90$, \textbf{Right:}  \exact
\end{table}



\begin{table}[]
\small
\centering
\label{my-label}
\begin{tabular}{@{}rrrr|@{}}
\toprule
$\sigma$ & $\lambda$ & mean        & std      \\ \midrule
0.001 & 0.001  & 0.1908      & 0.0006   \\
0.001 & 0.01   & 0.0681      & 0.0003   \\
0.001 & 0.1    & 0.2845      & 0.0006   \\
0.001 & 1      & 0.7991      & 0.0002   \\
0.001 & 10     & 0.9728      & 0.0001   \\
0.01  & 0.001  & 19.0389     & 0.0464   \\
0.01  & 0.01   & 5.3282      & 0.0273   \\
0.01  & 0.1    & 0.7587      & 0.0059   \\
0.01  & 1      & 0.8129      & 0.0025   \\
0.01  & 10     & 0.9731      & 0.0003   \\
0.1   & 0.001  & 1904.4359   & 5.9497   \\
0.1   & 0.01   & 530.8727    & 2.7780   \\
0.1   & 0.1    & 48.0577     & 0.4268   \\
0.1   & 1      & 2.1887      & 0.0547   \\
0.1   & 10     & 1.0044      & 0.0136   \\
1     & 0.001  & 190532.1581 & 437.1137 \\
1     & 0.01   & 53197.4003  & 187.2996 \\
1     & 0.1    & 4770.8654   & 39.7099  \\
1     & 1      & 140.4328    & 3.7919   \\
1     & 10     & 5.2586      & 2.0544  
\\ \bottomrule
\end{tabular}
\begin{tabular}{@{}rrrr@{}}
\toprule
$\sigma$ & $\lambda$ & mean        & std      \\ \midrule
0.001 & 0.001  & 0.1902      & 0.0004   \\
0.001 & 0.01   & 0.0681      & 0.0004   \\
0.001 & 0.1    & 0.2846      & 0.0005   \\
0.001 & 1      & 0.7993      & 0.0003   \\
0.001 & 10     & 0.9729      & 0.0001   \\
0.01  & 0.001  & 18.9854     & 0.0510   \\
0.01  & 0.01   & 5.2975      & 0.0240   \\
0.01  & 0.1    & 0.7565      & 0.0053   \\
0.01  & 1      & 0.8124      & 0.0014   \\
0.01  & 10     & 0.9731      & 0.0004   \\
0.1   & 0.001  & 1899.5487   & 8.1249   \\
0.1   & 0.01   & 528.6970    & 1.6710   \\
0.1   & 0.1    & 48.0165     & 0.3385   \\
0.1   & 1      & 2.1773      & 0.0563   \\
0.1   & 10     & 1.0041      & 0.0086   \\
1     & 0.001  & 189582.9265 & 413.1919 \\
1     & 0.01   & 52905.9901  & 169.3004 \\
1     & 0.1    & 4775.0883   & 18.4023  \\
1     & 1      & 137.7268    & 2.7805   \\
1     & 10     & 4.7664      & 1.9795   
\\ \bottomrule
\end{tabular}
\begin{tabular}{@{}rrrr@{}}
\toprule
$\sigma$ & $\lambda$ & mean        & std      \\ \midrule
0.001 & 0.001  & 0.1910      & 0.0007   \\
0.001 & 0.01   & 0.0723      & 0.0004   \\
0.001 & 0.1    & 0.2622      & 0.0020   \\
0.001 & 1      & 0.7198      & 0.0034   \\
0.001 & 10     & 0.9096      & 0.0047   \\
0.01  & 0.001  & 19.0615     & 0.0551   \\
0.01  & 0.01   & 5.5206      & 0.0143   \\
0.01  & 0.1    & 0.9020      & 0.0088   \\
0.01  & 1      & 0.7890      & 0.0050   \\
0.01  & 10     & 0.9523      & 0.0056   \\
0.1   & 0.001  & 1906.4671   & 2.9946   \\
0.1   & 0.01   & 550.5067    & 1.7484   \\
0.1   & 0.1    & 64.5503     & 0.6398   \\
0.1   & 1      & 7.4914      & 0.3399   \\
0.1   & 10     & 5.3122      & 0.3776   \\
1     & 0.001  & 190769.1154 & 348.4353 \\
1     & 0.01   & 54985.9953  & 241.7914 \\
1     & 0.1    & 6484.4296   & 73.4393  \\
1     & 1      & 688.6995    & 39.1172  \\
1     & 10     & 416.7218    & 40.3694  
\\ \bottomrule
\end{tabular}
\vskip 0.5em
\rls, \textbf{Left:} \disre, $\qbar=100$, \textbf{Middle:}  \disre, $\qbar=150$, \textbf{Right:} \disre,  $\qbar=100, \gamma = 1000$
\end{table}

\begin{table}[]
\small
\centering
\label{my-label}
\begin{tabular}{@{}rrrr|@{}}
\toprule
$\sigma$ & $\lambda$ & mean        & std      \\ \midrule
0.001 & 0.001  & 0.1905      & 0.0004   \\
0.001 & 0.01   & 0.0689      & 0.0003   \\
0.001 & 0.1    & 0.2814      & 0.0010   \\
0.001 & 1      & 0.7925      & 0.0010   \\
0.001 & 10     & 0.9716      & 0.0002   \\
0.01  & 0.001  & 19.0622     & 0.0340   \\
0.01  & 0.01   & 5.3400      & 0.0122   \\
0.01  & 0.1    & 0.7727      & 0.0047   \\
0.01  & 1      & 0.8086      & 0.0029   \\
0.01  & 10     & 0.9719      & 0.0004   \\
0.1   & 0.001  & 1906.6166   & 6.2333   \\
0.1   & 0.01   & 533.4745    & 2.2259   \\
0.1   & 0.1    & 49.6360     & 0.3576   \\
0.1   & 1      & 2.3184      & 0.0440   \\
0.1   & 10     & 1.0092      & 0.0084   \\
1     & 0.001  & 190451.4911 & 600.4823 \\
1     & 0.01   & 53479.7617  & 251.9642 \\
1     & 0.1    & 4961.0662   & 44.4829  \\
1     & 1      & 152.0926    & 4.5554   \\
1     & 10     & 4.4540      & 2.1627   \\
\bottomrule
\end{tabular}
\begin{tabular}{@{}rrrr@{}}
\toprule
$\sigma$ & $\lambda$ & mean        & std      \\ \midrule
0.001 & 0.001  & 0.1907      & 0.0007   \\
0.001 & 0.01   & 0.0681      & 0.0004   \\
0.001 & 0.1    & 0.2841      & 0.0005   \\
0.001 & 1      & 0.7983      & 0.0002   \\
0.001 & 10     & 0.9727      & 0.0001   \\
0.01  & 0.001  & 19.0317     & 0.0489   \\
0.01  & 0.01   & 5.3175      & 0.0253   \\
0.01  & 0.1    & 0.7606      & 0.0046   \\
0.01  & 1      & 0.8127      & 0.0027   \\
0.01  & 10     & 0.9730      & 0.0006   \\
0.1   & 0.001  & 1905.0821   & 4.0734   \\
0.1   & 0.01   & 531.3563    & 1.5535   \\
0.1   & 0.1    & 48.2788     & 0.3639   \\
0.1   & 1      & 2.1985      & 0.0424   \\
0.1   & 10     & 1.0191      & 0.0240   \\
1     & 0.001  & 190205.0504 & 370.3744 \\
1     & 0.01   & 53118.5249  & 176.1894 \\
1     & 0.1    & 4809.2633   & 34.7541  \\
1     & 1      & 140.9587    & 3.7546   \\
1     & 10     & 4.8436      & 1.5676   \\
\bottomrule
\end{tabular}
\begin{tabular}{@{}rrrr@{}}
\toprule
$\sigma$ & $\lambda$ & mean        & std      \\ \midrule
0.001 & 0.001  & 0.3024      & 0.0010   \\
0.001 & 0.01   & 0.1724      & 0.0004   \\
0.001 & 0.1    & 0.2838      & 0.0008   \\
0.001 & 1      & 0.7856      & 0.0006   \\
0.001 & 10     & 0.9706      & 0.0001   \\
0.01  & 0.001  & 30.1645     & 0.0703   \\
0.01  & 0.01   & 15.7825     & 0.0399   \\
0.01  & 0.1    & 1.9238      & 0.0107   \\
0.01  & 1      & 0.8227      & 0.0027   \\
0.01  & 10     & 0.9712      & 0.0005   \\
0.1   & 0.001  & 3022.1232   & 7.4520   \\
0.1   & 0.01   & 1575.6447   & 3.3189   \\
0.1   & 0.1    & 166.1367    & 0.6022   \\
0.1   & 1      & 4.5508      & 0.0500   \\
0.1   & 10     & 1.0276      & 0.0111   \\
1     & 0.001  & 302265.6761 & 375.8875 \\
1     & 0.01   & 157503.6242 & 369.7657 \\
1     & 0.1    & 16585.2588  & 43.9628  \\
1     & 1      & 380.2291    & 5.5260   \\
1     & 10     & 6.8515      & 1.1969   \\
\bottomrule
\end{tabular}
\vskip 0.5em
\rls, \textbf{Left:} \disre, $\qbar=100, \gamma = 100$, \textbf{Middle:}  \disre, $\qbar=100, \gamma = 10$, \textbf{Right:}  
\kn, $k=60$
\end{table}

\begin{table}[H]
\small
\centering
\label{my-label}
\begin{tabular}{@{}rrrr|@{}}
\toprule
$\sigma$ & $\lambda$ & mean        & std      \\ \midrule
0.001 & 0.001  & 0.2823      & 0.0007   \\
0.001 & 0.01   & 0.1259      & 0.0003   \\
0.001 & 0.1    & 0.2834      & 0.0008   \\
0.001 & 1      & 0.7929      & 0.0005   \\
0.001 & 10     & 0.9719      & 0.0001   \\
0.01  & 0.001  & 28.2517     & 0.0560   \\
0.01  & 0.01   & 11.0908     & 0.0309   \\
0.01  & 0.1    & 1.1154      & 0.0087   \\
0.01  & 1      & 0.8119      & 0.0029   \\
0.01  & 10     & 0.9722      & 0.0006   \\
0.1   & 0.001  & 2823.5932   & 6.3830   \\
0.1   & 0.01   & 1108.9582   & 5.3674   \\
0.1   & 0.1    & 84.0189     & 0.4347   \\
0.1   & 1      & 2.8117      & 0.0394   \\
0.1   & 10     & 1.0124      & 0.0178   \\
1     & 0.001  & 282388.0211 & 754.7190 \\
1     & 0.01   & 110729.9273 & 250.8183 \\
1     & 0.1    & 8373.5778   & 55.8408  \\
1     & 1      & 200.9318    & 3.7598   \\
1     & 10     & 4.8285      & 1.0723   
\\ \bottomrule
\end{tabular}
\begin{tabular}{@{}rrrr@{}}
\toprule
$\sigma$ & $\lambda$ & mean        & std      \\ \midrule
0.001 & 0.001  & 0.1896      & 0.0002   \\
0.001 & 0.01   & 0.0676      & 0.0005   \\
0.001 & 0.1    & 0.2856      & 0.0002   \\
0.001 & 1      & 0.8000      & 0.0002   \\
0.001 & 10     & 0.9730      & 0.0000   \\
0.01  & 0.001  & 18.9408     & 0.0554   \\
0.01  & 0.01   & 5.2789      & 0.0147   \\
0.01  & 0.1    & 0.7566      & 0.0067   \\
0.01  & 1      & 0.8127      & 0.0028   \\
0.01  & 10     & 0.9737      & 0.0003   \\
0.1   & 0.001  & 1894.0375   & 5.7404   \\
0.1   & 0.01   & 527.5392    & 1.0052   \\
0.1   & 0.1    & 47.7877     & 0.4702   \\
0.1   & 1      & 2.1503      & 0.0282   \\
0.1   & 10     & 1.0083      & 0.0184   \\
1     & 0.001  & 189515.7627 & 293.8406 \\
1     & 0.01   & 52639.4113  & 346.7699 \\
1     & 0.1    & 4724.3469   & 44.4066  \\
1     & 1      & 138.3458    & 2.9040   \\
1     & 10     & 4.1240      & 0.9370   \\ 
\bottomrule
\end{tabular}
\vskip 0.5em
\rls, \textbf{Left:}  \kn, $k=90$, \textbf{Right:}  \exact
\end{table}
\vfil


\end{document}